\documentclass[twoside]{article}

\usepackage[preprint]{icml2026}
\usepackage{amsmath}
\usepackage{booktabs}
\usepackage{graphicx}
\usepackage{macro}
\usepackage{hyperref}

\usepackage{caption}
\usepackage{subcaption}

\begin{document}

%

%

\twocolumn[
  \icmltitle{Data Augmentation via Causal-Residual Bootstrapping}



  \icmlsetsymbol{equal}{*}

  \begin{icmlauthorlist}
    \icmlauthor{Mateusz Gajewski}{poznan,akces}
    \icmlauthor{Sophia Xiao}{dartmouth}
    \icmlauthor{Bijan Mazaheri}{dartmouth,ewsc}
  \end{icmlauthorlist}

  \icmlaffiliation{poznan}{Poznań University of Technology, Poznań, Poland}
  \icmlaffiliation{dartmouth}{Thayer School of Engineering, Dartmouth, Hanover, NH, USA}
  \icmlaffiliation{ewsc}{Eric and Wendy Schmidt Center, Broad Institute of MIT and Harvard, Cambridge, MA, USA}
  \icmlaffiliation{akces}{AKCES NCBR, Warsaw, Poland}

  \icmlcorrespondingauthor{Bijan Mazaheri}{bijan.h.mazaheri@dartmouth.edu}

\icmlkeywords{Causal Inference, Data Augmentation, Machine Learning}

  \vskip 0.3in
]

\printAffiliationsAndNotice{}







\begin{abstract}
  Data augmentation integrates domain knowledge into a dataset by making domain-informed modifications to existing data points. For example, image data can be augmented by duplicating images in different tints or orientations, thereby incorporating the knowledge that images may vary in these dimensions. Recent work by Teshima and Sugiyama has explored the integration of causal knowledge (e.g, A causes B causes C) up to conditional independence equivalence. We suggest a related approach for settings with additive noise that can incorporate information beyond a Markov equivalence class. The approach, built on the principle of independent mechanisms, permutes the residuals of models built on marginal probability distributions. Predictive models built on our augmented data demonstrate improved accuracy, for which we provide theoretical backing in linear Gaussian settings.
\end{abstract}

\section{Introduction}


Large, comprehensive datasets have driven a meteoric rise in the scale and capabilities of machine learning systems. Nevertheless, data-sparse regimes remain, particularly in human health and science, where financial and ethical barriers prohibit extensive data collection.

By adding noise or anticipated transformations, data augmentation creates new training data points, thereby enhancing the power of small datasets \citep{wang2024comprehensive}. Such augmentation strategies can be thought of as incorporating prior knowledge as new data. For example, an augmented image dataset may account for known variations by varying lighting conditions and performing rotations \citep{van2001art}.

Causal information is often represented in the form of a structural causal model (SCM), which is known to imply conditional independencies according to d-separation rules \citep{pearl2009causality}. For example, consider a causal chain with $A$ causes $B$ and $B$ causes $C$, and no other causal dependencies. Such a system is represented using a directed acyclic graph $A \rightarrow B \rightarrow C$, which is constrained by a Markov property of $A \indep C \given B$.

Causal structures or conditional independencies are often known from prior experiments or by design. For example, randomized controlled trials (which are often small) are known to assign a randomized treatment ($T$) that contains no arrows to or from other attributes of the participants ($\bvec{X}$). Such knowledge tells us that $\bvec{X} \indep T$ without the need for statistical testing. In other settings, causal structures can be learned using a process known as causal discovery \citep{squires2023causal}.

Recently, foundation models based on the prior–data fitted network (PFN) paradigm have gained traction in the ML community. This includes TabPFN~\citep{mullertransformers, hollmann2025accurate}, which was pre-trained on synthetic datasets generated using random structural causal models and Bayesian networks. Surprisingly, models pre-trained on synthetic data perform well on real-world datasets, seemingly benefiting from information that a causal structure \emph{exists}, even though the exact causal structure of the eventual real-world data was not necessarily provided. This suggests that structural priors capture complex independence relations in real-world data~\cite{jiang2025tabstruct}.

Seminal work by \citet{teshima2021incorporating} has explored incorporating causal information into data augmentation. In particular, they noted that conditional independencies allow entries of a dataset to be permuted. For example, if $A \indep B$, then any observed value of $A$ may be paired with any observed value of $B$. These permutations augment the data, e.g., a two-point set $(0,0), (1,1)$ may be augmented to $(0,0), (1,1), (0,1), (1,0)$. Their approach shows moderate improvement in prediction accuracy.

The approach by \citet{teshima2021incorporating} is limited in two ways. First, their approach requires reweighting data according to estimates of probability density, which thereby requires tuning the radius of a RBF kernel \citep{poinsot2023guide}. Reweighting based on density was criticized by \citet{NEURIPS2020_24368c74} for its dependence on tuning and general instability. Second, their approach is limited to augmenting conditional independence properties, which cannot account for causal direction. As such, the entire directed acyclic graph (DAG) cannot be utilized.

Conditional independence information alone cannot fully specify a causal structure. For example, $A \rightarrow B \rightarrow C$, $A \leftarrow B \leftarrow C$, and $A \leftarrow B \rightarrow C$ display the same set of conditional independence properties, known as a Markov equivalence class (MEC), which is represented using a ``completed partial DAG'' (CP-DAG) when there are no latent variables, and an acyclic directed mixed graph (ADMG) in the presence of latent confounding \citep{spirtes2000causation}. Literature on causal discovery has shown that systems with nonlinear dependencies and Gaussian additive noise contain additional information beyond the conditional independence properties that specify an MEC \citep{HoyerJanzingMooijPetersScholkopf2008, PetersMooijJanzingScholkopf2014}. In particular, nonlinear relationships shift Gaussian distributions to non-Gaussian distributions, yielding Gaussian residuals in the causal direction and non-Gaussian residuals in the anti-causal direction.

The goal of this paper is to address the limitation of \citet{teshima2021incorporating} by incorporating information from the \emph{full causal structure}, including causal directions. We do this using the principle of independent mechanisms, which states that if $X$ causes $Y$, then the mechanism that turns $X$ into $Y$ (i.e., $\Pr(Y | X)$) does not depend on how $X$ is distributed ($\Pr(X)$) \citep{peters2017elements}. As such, the distribution of residuals of the effect $Y$ is \emph{independent} from the value of the cause $X$.

\subsection{Summary of Contributions}

We begin by motivating the incorporation of causal structure via data augmentation, rigorously demonstrating its ability to improve downstream regression tasks. Section~\ref{sec: causal constraints} provides a theoretical justification for this claim, showing that the improvement in regression error scales proportionally to the inverse of the starting dataset size (smaller datasets are improved more).

Motivated by these theoretical results, we build on the work of \citet{teshima2021incorporating} to develop a more broadly-applicable causal data augmentation approach, called ``Causal-Residual Bootstrapping'' (CRB). CRB is given in Section~\ref{sec: CRB}, utilizing the principle of independent mechanisms. The key observation is that predictions of variables from their direct causes should exhibit residuals that are independent from their predecessors' values. As such, the residuals may be permuted, allowing us to ``bootstrap'' the noise term and generate additional data. CRB satisfies two goals: it may preserve existing structural information by augmenting data based on the output of a causal discovery algorithm, or it may help inject new structural information into a dataset by leveraging a known set of causal dependencies.

To emphasize the need for new methods like CRB, we demonstrate that genAI-based augmentation not only fails to capture causal structure but also contributes to its decline. Section~\ref{sec: data augmentation and causal structure} shows that VAEs, GANs, and diffusion models all weaken the performance of causal discovery algorithms run on their augmented datasets. Meanwhile, CRB and the approach by \citet{teshima2021incorporating} preserve information about causal structure, as evidenced by unchanging performance in causal discovery. We further test both of these causally-informed augmentation methods on linear models with non-Gaussian noise, which exhibit additional information beyond their MECs. In these experiments, the direct LinGAM algorithm \citep{shimizu2011directlingam} does not perform well on data augmented by \citet{teshima2021incorporating}'s approach, but retains performance for CRB. This demonstrates the added value of augmenting according to the full causal structure.

Since many datasets do not satisfy the linear restrictions of our theoretical results, Section~\ref{sec: real data experiments} provides an empirical study of CRB's improvements to downstream regression tasks. Experiments are done on both synthetic (Ap.~\ref{sec: synthetic data exp}) and real datasets. These tests include both settings with known causal structures and those with unknown causal structures learned from causal discovery algorithms.

\subsection{Related Works}

\paragraph{Data Augmentation.}
Various domains have utilized prior knowledge for data augmentation. Modern techniques include generative adversarial networks \citep{bowles2018gan, tanaka2019data} and diffusion models \citep{trabucco2024effective}. \citet{myronenko20183d} utilized an auto-encoder with built-in randomness to regularize a model for brain-tumor segmentation to incorporate the prior belief that brain tumors are not location-specific. SMOTE \citep{chawla2002smote} uses data augmentation to address undersampling. 

\paragraph{Incorporating Causal Constraints.}
A series of works has explored incorporating successively increasing amounts of causally-motivated constraints into regression. Early work by \citet{chaudhuri2007estimation} investigated improvements in predictive regression tasks by zeroing entries of covariance matrices based on known \emph{marginal} independencies. \citet{teshima2021incorporating} expanded this notion to \emph{conditional} independencies by incorporating information about the full MEC implied by a causal structure. This work also shifted from covariance matrix modification to data augmentation, thereby relaxing the approach to non-Gaussian and non-linear settings. Still, this approach did not incorporate all causal information, omitting causal directions that are unspecified by an MEC. Most recently, augmenting datasets using the underlying (causal) Bayesian network has improved fine-tuning of tabular foundation models in low-data regimes~\citep{buhler2026causal}. This approach was developed for a more specialized setting, but aligns with our motivation to use the entire causal structure for augmentation. Our approach eliminates the need for density estimation, used in \citet{buhler2026causal} and \citet{teshima2021incorporating}. The approach can be thought of as a data-augmentation version of projecting distributions into a causal structure, a notion introduced by \citet{mazaheri2025meta}.

\paragraph{Causal Structure and Predictive Models.}
Having access to the causal structure was shown to be beneficial in many cases. This includes improving robustness to distribution shifts~\citep{lu2021invariant, heinze2021conditional, magliacane2018domain, rojas2018invariant, mazaheri2023causal},  feature selection via Markov blankets \citep{tsamardinos2003towards, yu2020causality}, and measuring fairnes~\citep{makhlouf2024causality, maasch2025local}. All of these approaches are orthogonal to our own in that they model invariances that help build more robust machine learning models, but do not directly incorporate that information into the data. Such works motivate our goal of incorporating this information directly into the dataset.

\paragraph{Other Related Works.}
Our work should not be confused with ``causal boostrapping,'' proposed by \citet{little2019causal}, which involves re-sampling interventional distributions from observational ones.

\section{Preliminaries}

\subsection{Notation Conventions}
We will use the capital Roman alphabet (e.g., $X, Y, A, B$) to denote endogenous (observed) random variables and the lowercase Roman alphabet to denote instantiations of those random variables (e.g., $x_i$ is a data measurement of $X$). $\varepsilon_i$ will generally be used to represent random variables in the form of exogenous noise.

We will encounter two types of sets. Bold font will indicate sets of random variables, i.e. $\bvec{V} = \{V_1, V_2, \ldots, V_n\}$. Caligraphic font will generally be used to indicate sets of assignments (e.g., a dataset) to those random variables. For example, $\mathcal{V_i} = \{v_i^{(1)}, v_i^{(2)}, \ldots v_i^{(m)}\}$.

This paper will make use of directed acyclic graphs (DAGs), e.g. $\mathcal{G} = (\bvec{V}, \bvec{E})$, for which we will need to reference parent sets. The vertices of these graphs will usually be indexed as $\bvec{V} = \{V_1, V_2, \ldots, V_n\}$. We denote the parent set with respect to $\G$ with edges $\bvec{E}$ using $\PA(V_j) = \{V_i : (V_i, V_j) \in \bvec{E}\}$.

\subsection{Structural Causal Models}
A \emph{structural causal model} (SCM) is a set of random variables $\bvec{V}$ and a DAG $\mathcal{G} = (\bvec{V}, \bvec{E})$ on those random variables, representing causal dependencies \citep{pearl2009causality}. For example, $A \rightarrow B$ indicates that $A$ causes $B$.

Graphical properties of $\G$ imply conditional independence properties on the random variables. D-separation rules, given in \citet{pearl2009causality}, specify these graphical properties for conditional independence. Constraint-based causal discovery algorithms utilize these properties to ascertain graphical information using statistical tests for independence \citep{spirtes2000causation}.

Conditional independencies do not form a one-to-one mapping with causal DAGs; in fact, many causal DAGs can correspond to the same set of conditional independence properties. For this reason, many causal discovery algorithms only aim to recover these ``Markov equivalence classes'' (MECs). In general, all of the graphs in a MEC contain the same ``undirected skeleton,'' formed by replacing all directed edges $(V_i, V_j)$ with undirected edges $\{V_i, V_j\}$. In this sense, we can conclude that conditional independence contains full information about causal adjacency (the presence or lack of a causal linkage), but incomplete information about the directions of those relationships. This incomplete information is often represented in an ADMG called a completed partial DAG (CP-DAG) \citep{spirtes2000causation}. For example, an undirected edge between $A,B$ means that either $A \rightarrow B$ or $B \rightarrow A$, but we do not know which one.

\subsection{Structural Equation Models}
Parametric assumptions have been utilized further to resolve the output of causal discovery into a full causal DAG. In order to make these assumptions, we must augment SCMs with functional forms for the relationships between their random variables,
\begin{equation} \label{eq: structural equations}
    V_i = f_i(\PA(V_i), \varepsilon_i).
\end{equation}
$V_i$ is therefore generated using a function of $\PA_i$, the direct causes of $V_i$, and it's own source of exogeneous noise $\varepsilon_i$. It is common to assume that the exogenous noise sources are independent, i.e., $\varepsilon_i \indep \varepsilon_j$ for all $i,j$.

A common assumption on Eq.~\eqref{eq: structural equations} is additive noise, i.e.,
\begin{equation} \label{eq:anm}
    V_i = f_i(\PA(V_i)) + \varepsilon_i.
\end{equation}
In this framework, full identifiability of a causal model has been shown for linear $f_i(\cdot)$ and non-Gaussian $\varepsilon_i$ (LiNGAM) \citep{shimizu2014lingam} and non-linear $f_i(\cdot)$ with Gaussian $\varepsilon_i$s \citep{HoyerJanzingMooijPetersScholkopf2008, PetersMooijJanzingScholkopf2014}.

\section{Causally Constrained Regression} \label{sec: causal constraints}

In this section, we summarize our theoretical results in Ap.~\ref{sec: theory} on the value of incorporating causal constraints into predictive regressions on observational data. For now, we only focus on the incorporation of the constraints within regression. Later (and in Ap.~\ref{subsec: eq to constrained MLE}), we show that our data augmentation method CRB, asymptotically approaches this constrained regression.

The principle assumption of our theoretical results is that the data-generating process follows a linear Gaussian structural causal model (SCM) as in Eq.~\eqref{eq:anm}, where $f_i$ are linear functions and $\varepsilon_i$ are Gaussian noise variables with variance $\Var(\varepsilon_i) > 0$, and that the true causal DAG $\mathcal{G} = (\bvec{V}, \bvec{E})$ is known. The full set of assumptions is stated formally in Ap.~\ref{sec: theory}. We emphasize that violations of linearity and Gaussianity are heavily tested in our empirical sections.

We compare two approaches to fitting a multivariate Gaussian distribution to the given training data:
\begin{enumerate}
    \item \textbf{Unconstrained estimation:} We fit a full multivariate Gaussian by computing the empirical covariance matrix $\hat{\boldsymbol{\Sigma}}_{\text{full}}$ without imposing any structural constraints.
    
    \item \textbf{DAG-constrained estimation:} We fit a multivariate Gaussian distribution $\hat{\boldsymbol{\Sigma}}_{\text{DAG}}$ that respects the conditional independence structure implied by the graph $\mathcal{G}$. Specifically, for any sets of variables $A$, $B$, and $S$, if $A \indep B \mid S$ holds in $\mathcal{G}$ according to d-separation, then the fitted distribution must also satisfy $A \indep B \mid S$.
\end{enumerate}

From the estimated covariance matrix, we can derive the regression coefficients for predicting a target variable $Y$ from features $\mathbf{X}$:
\begin{equation}
    \boldsymbol{\beta} = \boldsymbol{\Sigma}_{\mathbf{XX}}^{-1}\boldsymbol{\Sigma}_{\mathbf{X}\mathbf{Y}}.
\end{equation}
In the unconstrained case, this is equivalent to standard ordinary least squares (OLS) regression.

Enforcing conditional independence constraints directly on the covariance matrix parameterization is non-trivial, as these constraints translate into complex nonlinear relationships among the entries of $\boldsymbol{\Sigma}$. However, such constraints are equivalent to setting certain parameters to zero in the decomposition of the \emph{precision matrix}. Our results show that this sparsity pattern induces zeros in the Fisher information matrix, thereby decreasing the asymptotic variance of the estimated parameters.

Specifically, we establish the following chain of results. First, the DAG-constrained estimator achieves lower variance in the Loewner ordering:
\begin{equation}
    \Cov(\hat{\boldsymbol{\Sigma}}_{\text{DAG}}) \preceq \Cov(\hat{\boldsymbol{\Sigma}}_{\text{full}}).
\end{equation}
This ordering transfers to the regression coefficients and, consequently, to the prediction error:
\begin{equation}
    \mathbb{E}[\mathrm{MSE}_{\text{pred}}^{\text{DAG}}] \leq \mathbb{E}[\mathrm{MSE}_{\text{pred}}^{\text{full}}].
\end{equation}

Moreover, since the MLE parameters are asymptotically distributed as $\hat{\boldsymbol{\theta}} \sim \mathcal{N}(\boldsymbol{\theta}, \frac{1}{N}\mathcal{I}^{-1})$, where $\mathcal{I}$ is the Fisher information matrix, we can quantify the expected improvement. The difference in prediction MSE scales as:
\begin{equation}
    \mathbb{E}[\mathrm{MSE}_{\text{pred}}^{\text{full}}] - \mathbb{E}[\mathrm{MSE}_{\text{pred}}^{\text{DAG}}] = \frac{C}{N},
\end{equation}
where $C > 0$ is a constant that depends on the DAG structure and the true distribution, but is independent of the sample size $N$. 

Finally, we show that only DAG constraints involving the Markov boundary of the chosen label $Y$ contribute to the improvement in the prediction of $Y$. Although constraints among variables outside $\mathrm{MB}(Y)$ decrease the variance of the estimated distribution parameters, this reduction does not translate to lower prediction MSE.

\section{Causal-Residual Bootstrapping (CRB)} \label{sec: CRB}

In this section, we present a procedure that allows one to incorporate prior knowledge of a DAG by adding new augmented points to the dataset.
We do this by resampling residuals, leveraging the principle of independent mechanisms.
The proposed method uses regression models instead of more complicated generative models and can be applied to non-linear and non-Gaussian settings, unlike the constrained regression discussed in our theoretical results.

\subsection{Problem Setup and Input Specification}

The augmentation procedure takes as input:

\begin{itemize}
    \itemsep0em 
    \item A directed acyclic graph (DAG) $\mathcal{G} = (\bvec{V}, \bvec{E})$ representing prior knowledge that we have about the data generation process and that we want to incorporate.
    \item Dataset of values of $n$ observed datapoints $\mathcal{X}$.
    \item A regression model class $\mathcal{F}$ such that $f\in\mathcal{F}: \mathbb{R}^d \to \mathbb{R}, d \in \mathbb{N}$ (e.g., linear regression, random forests, neural networks) for the structural equations that can be fitted to data.
    \item $M$, the number of points to generate.
\end{itemize}

The procedure outputs a new dataset $\tilde{\mathcal{X}}$ with augmented samples.

\subsection{Augmentation Procedure}

The augmentation procedure consists of two main phases:

\textbf{Learning Phase:} We learn the causal mechanisms by:
\begin{enumerate}
    \item Computing a topological ordering $\pi$ of the DAG $\mathcal{G}$. 
    \item For each non-root variable $X_j$, we:
    \begin{itemize}
        \item Train a regression model $\hat{f}_j \in \mathcal{F}$ to predict $X_j$ from $\PA(X_j)$ using the data $\mathcal{V}$.
        \item Compute residuals $\hat{\varepsilon}_j^{(i)} = x_j^{(i)} - \hat{f}_j(\PA(x_j^{(i)}))$ for all samples $i = 1, \ldots, N$.
    \end{itemize}
    \item For root variables, we store the vector of all observations $v_j$.
\end{enumerate}

\textbf{Generation Phase:} To generate $N$ synthetic samples, we:
\begin{enumerate}
    \item For each synthetic sample $m = 1, \ldots, M$:
    \begin{itemize}
        \item For each variable $X_j$ in topological order $\pi$:
        \begin{itemize}
            \item If $X_j$ is a root node: sample $\tilde{x}_j^{(n)}$ from the vector $v_j$.
            \item Otherwise: compute $\tilde{x}_j^{(m)} = \hat{f}_j(\PA(\tilde{x}_j^{(m)})) + \tilde{\varepsilon}_j$, where $\tilde{\varepsilon}_j$ is randomly sampled from $\{\hat{\varepsilon}_j^{(i)}\}_{i=1}^{N}$.
        \end{itemize}
    \end{itemize}
\end{enumerate}

This procedure ensures that the generated data respects the structural equation framework specified by the assumed SEM/SCM. As such, both the conditional independence properties (implied by the d-separation conditions of the MEC) and the law of independent mechanisms (implied by causal direction) are enforced in the augmented data. The pseudocode is given in Alg.~\ref{alg:pseudocode} in Ap.~\ref{sec: pseudocode}.

\subsection{CRB as Constrained Maximum Likelihood}
\label{sec:crb_mle}

Under the assumptions of a linear SCM with Gaussian noise, when CRB uses linear regression to model each variable as a function of its parents, the method is asymptotically equivalent to performing global maximum likelihood estimation over the class of multivariate Gaussians constrained by the DAG structure, as described in Section~\ref{sec: causal constraints}. Specifically, as the number of generated samples grows large, the CRB estimates converge to the constrained MLE solution. The full derivation is provided in Ap.~\ref{subsec: eq to constrained MLE}.

Figure~\ref{fig:mse_improvement_validation} validates these theoretical predictions empirically. We tested two synthetic configurations: a simple chain ($A \to B \to C$) and a confounded structure ($A \to B \leftarrow D$, $B \to C$). In both cases, the observed MSE gap between unconstrained and DAG-constrained estimation follows the predicted $C/N$ decay discussed in Section~\ref{sec: causal constraints}, with larger improvements at smaller sample sizes.

\begin{figure}[htbp]
    \centering
    \includegraphics[width=\columnwidth]{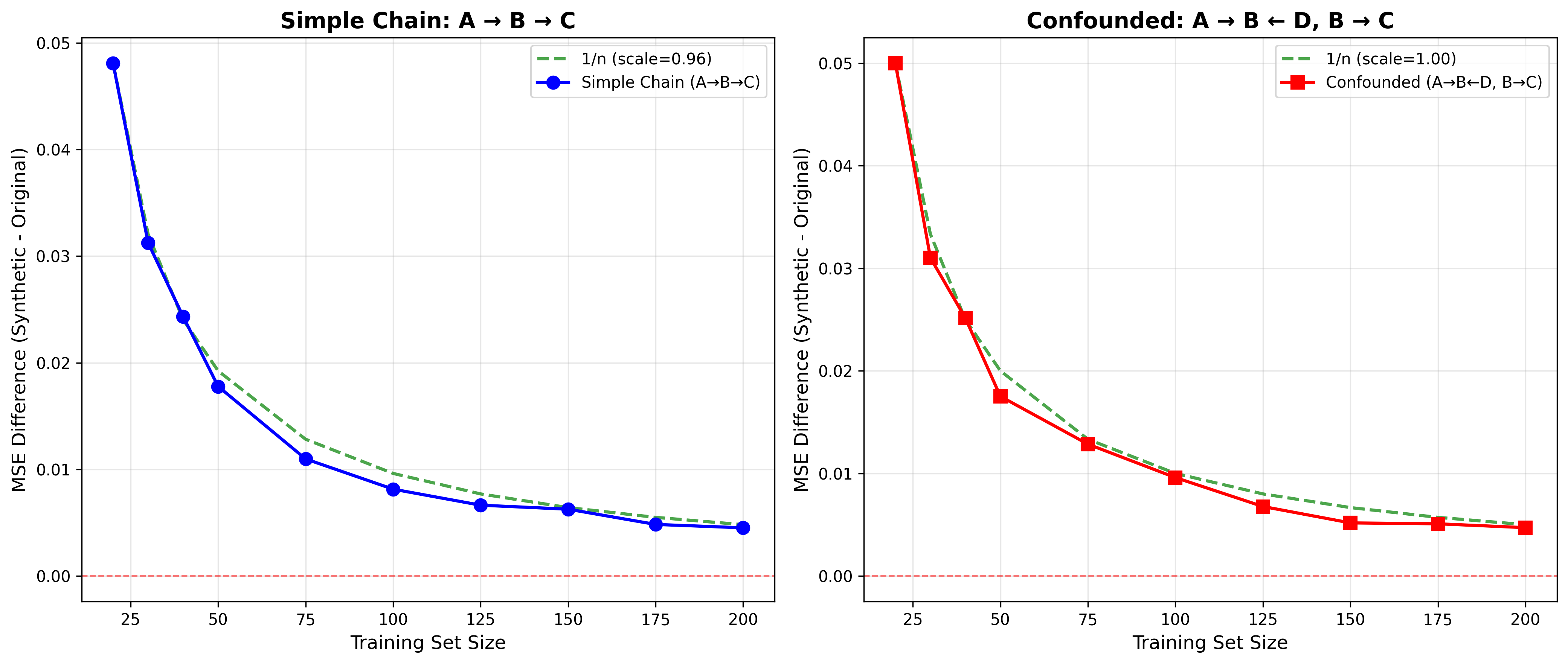}
    \caption{Empirical validation of MSE improvement rate. Left: Simple chain $A \to B \to C$. Right: Confounded structure $A \to B \leftarrow D$, $B \to C$. Both configurations show MSE improvement following the predicted $1/N$ decay. In both cases B is the predicted value}
    \label{fig:mse_improvement_validation}
\end{figure}

\section{Data Augmentation and Causal Structure}\label{sec: data augmentation and causal structure}

To demonstrate the necessity of the proposed causal-residual-bootstrapping approach, we investigate whether deep learning augmentation techniques can preserve information about causal structures. We evaluated this using two causal discovery algorithms: The PC algorithm \citep{spirtes2000causation}, which recovers causal structures up to their MECs, and direct LinGAM \citep{shimizu2011directlingam}, which recovers the \emph{full causal structure} in settings with linear structural equations and non-Gaussian noise. Good performance, as measured by the Structural Hamming Distance (SHD), is taken to indicate that the augmented data preserves information about the causal structure. Poor performance indicates that the augmentation technique fails to preserve the key properties of data generated by that structure. 

\subsection{Preservation of MECs}
We first tested the preservation of information pertaining to the \emph{Markov equivalence class}. We evaluated three generative methods: (1) Diffusion Models, (2) Variational Autoencoders (VAEs), and (3) Generative Adversarial Networks (GANs). Tests were performed on $100$ toy datasets generated from linear Gaussian SEMs corresponding to distinct and known DAGs. For each method, new data points were generated using the trained augmentation model and appended to the original dataset. The PC algorithm was then applied to the augmented datasets to recover the causal graph. Further details for the experimental setup are given in Ap.~\ref{sec: implementation details for causal structure loss}.

Figure \ref{fig:shd_results} illustrates the structural accuracy of the recovered graphs as a function of the number of augmented points.
We observe that as the number of synthetic datapoints increases, the SHD worsens across all baseline methods. This indicates that while these methods may achieve high distributional fidelity (low FID), they dilute the causal structural signal inherent in the original data. In contrast to causal-residual-bootstrapping, these causally-agnostic approaches fail to preserve the conditional independence relations required for accurate causal discovery. Meanwhile, the performance of causal discovery is unaffected by augmentation from CRB, as indicated by a non-increasing SHD.

\begin{figure*} 
\hfill \begin{subfigure}[b]{0.23\textwidth} \centering \includegraphics[width=\textwidth]{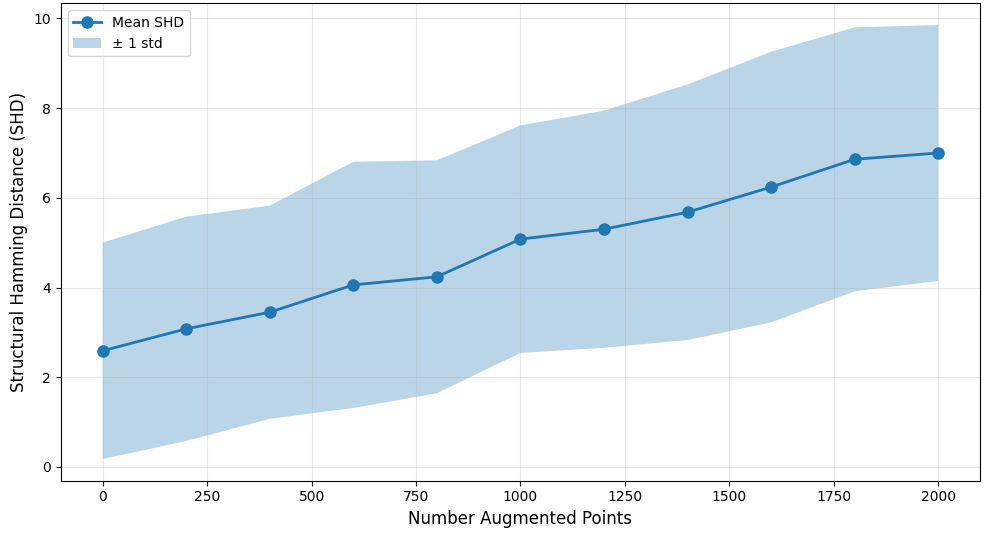}
\caption{Diffusion Model} \label{fig:diffusion} \end{subfigure} \hfill \begin{subfigure}[b]{0.23\textwidth} \centering  \includegraphics[width=\textwidth]{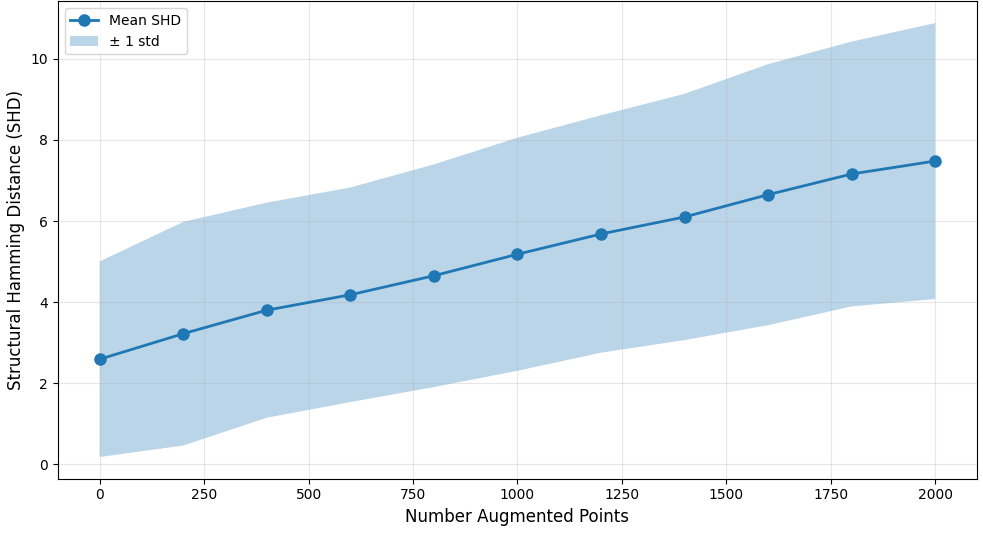}
\caption{Variational Autoencoder} \label{fig:vae} \end{subfigure} \hfill
\begin{subfigure}[b]{0.23\textwidth} \centering
\includegraphics[width=\textwidth]{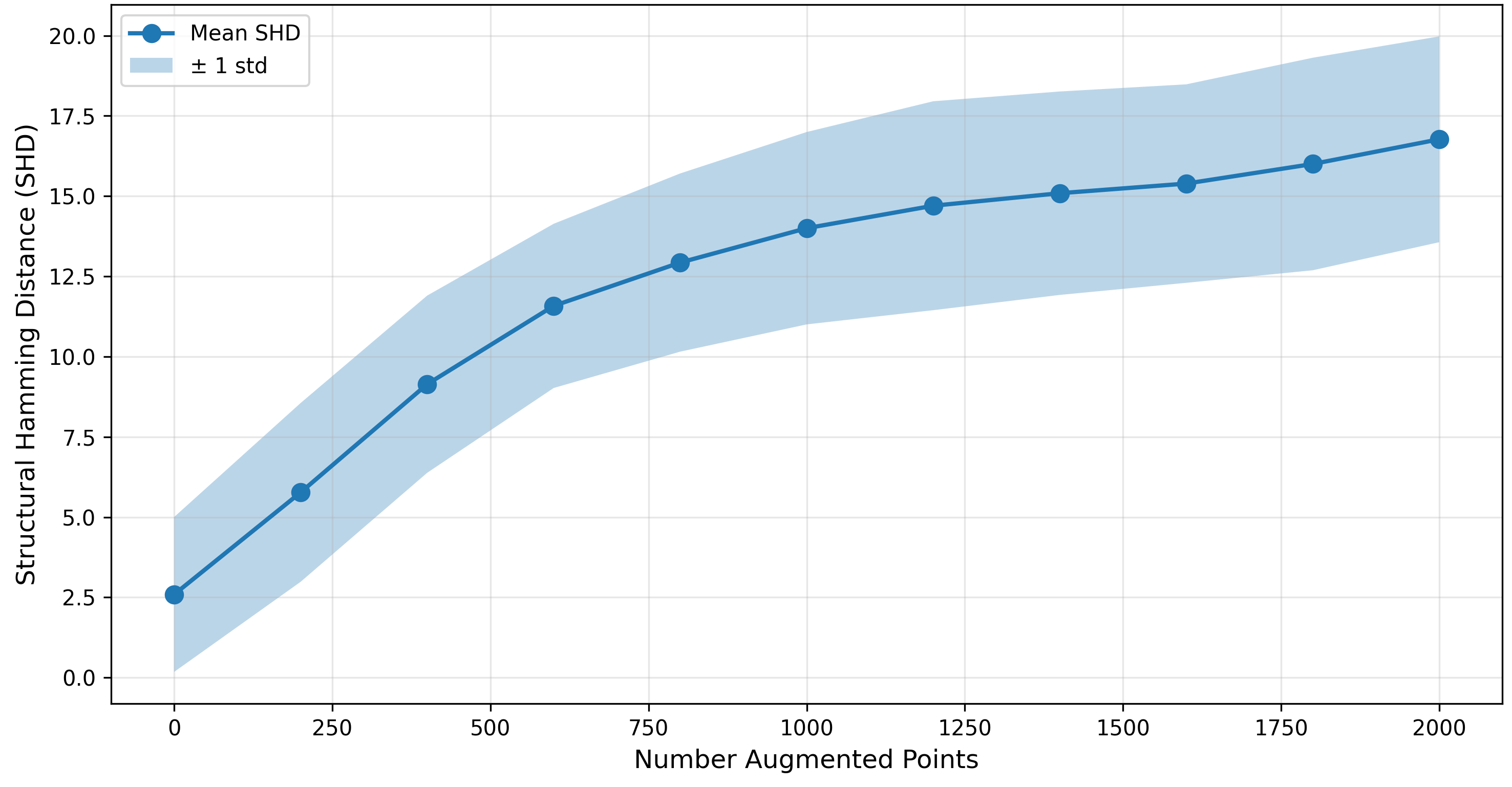}
\caption{GAN} \label{fig:gan} 
\end{subfigure}
\begin{subfigure}[b]{0.23\textwidth} \centering
\includegraphics[width=\textwidth]{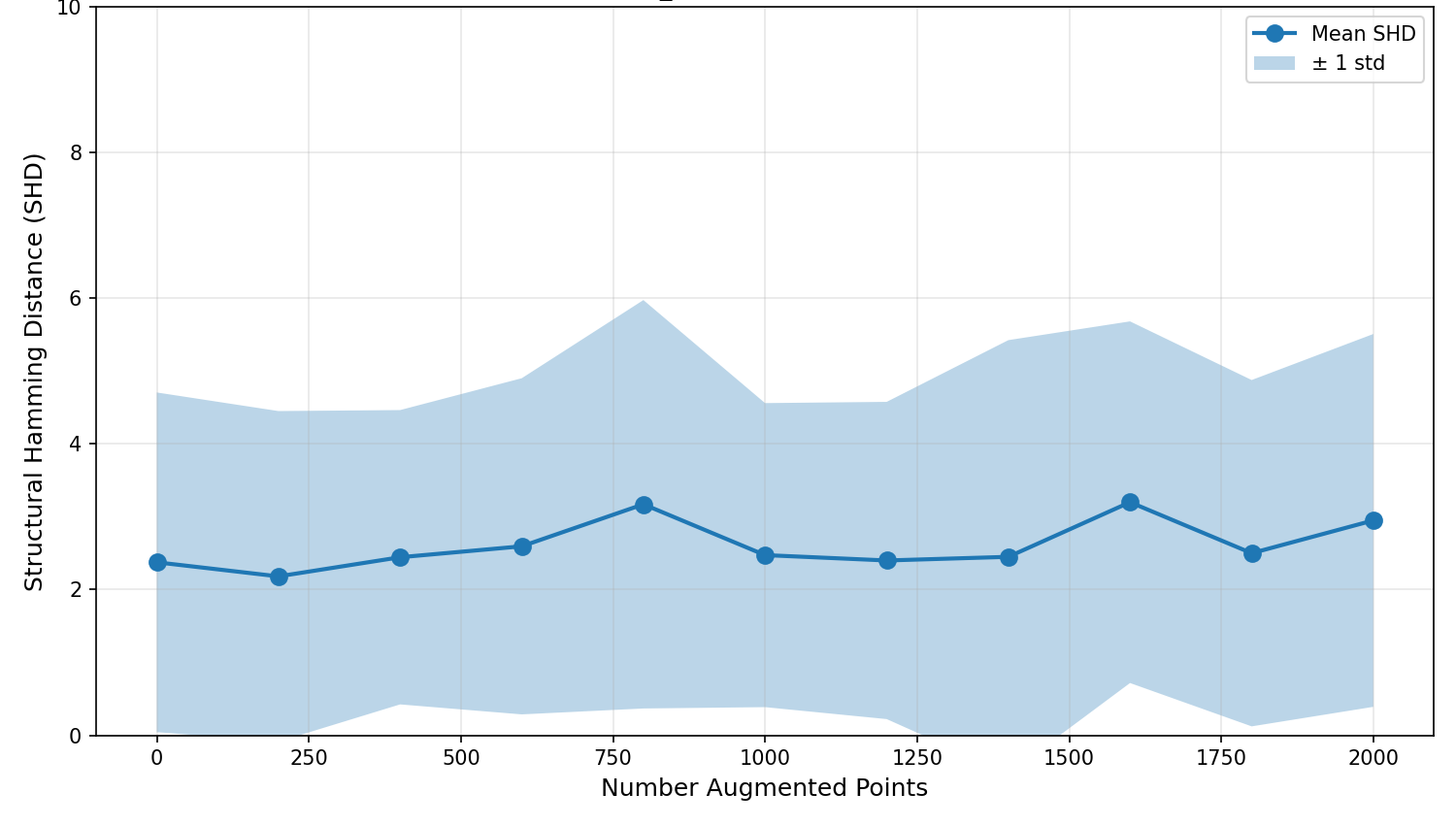}
\caption{CRB} \label{fig:crb}
\end{subfigure}

\caption{Performance of the PC algorithm on datasets augmented by alternative methods. The plots show the average Structural Hamming Distance (SHD) between the true DAG and the estimated CPDAG returned by the PC algorithm as the number of augmented points increases; the shaded region indicates one standard deviation. Higher SHD indicates worse performance.}
\label{fig:shd_results}
\end{figure*}

\subsection{Preservation of Full Causal Structure}
While linear models with Gaussian additive noise cannot be recovered beyond their MECs, additional information is contained in non-Gaussian settings. Fig.~\ref{fig:directLinGAM} shows the performance of the direct LinGAM algorithm on data augmented with CRB and \citet{teshima2021incorporating}'s ADMGTian approach. CRB continues to preserve all of the causal structure, while data added by ADMGTian weaken the algorithm's ability to orient causal direction despite retaining conditional independencies. This demonstrates the added value of shifting from a purely MEC-based augmentation technique like ADMGTian to CRB.

\begin{figure}[htbp]
\centering
\begin{subfigure}[t]{0.48\columnwidth}
    \centering
    \includegraphics[width=\textwidth]{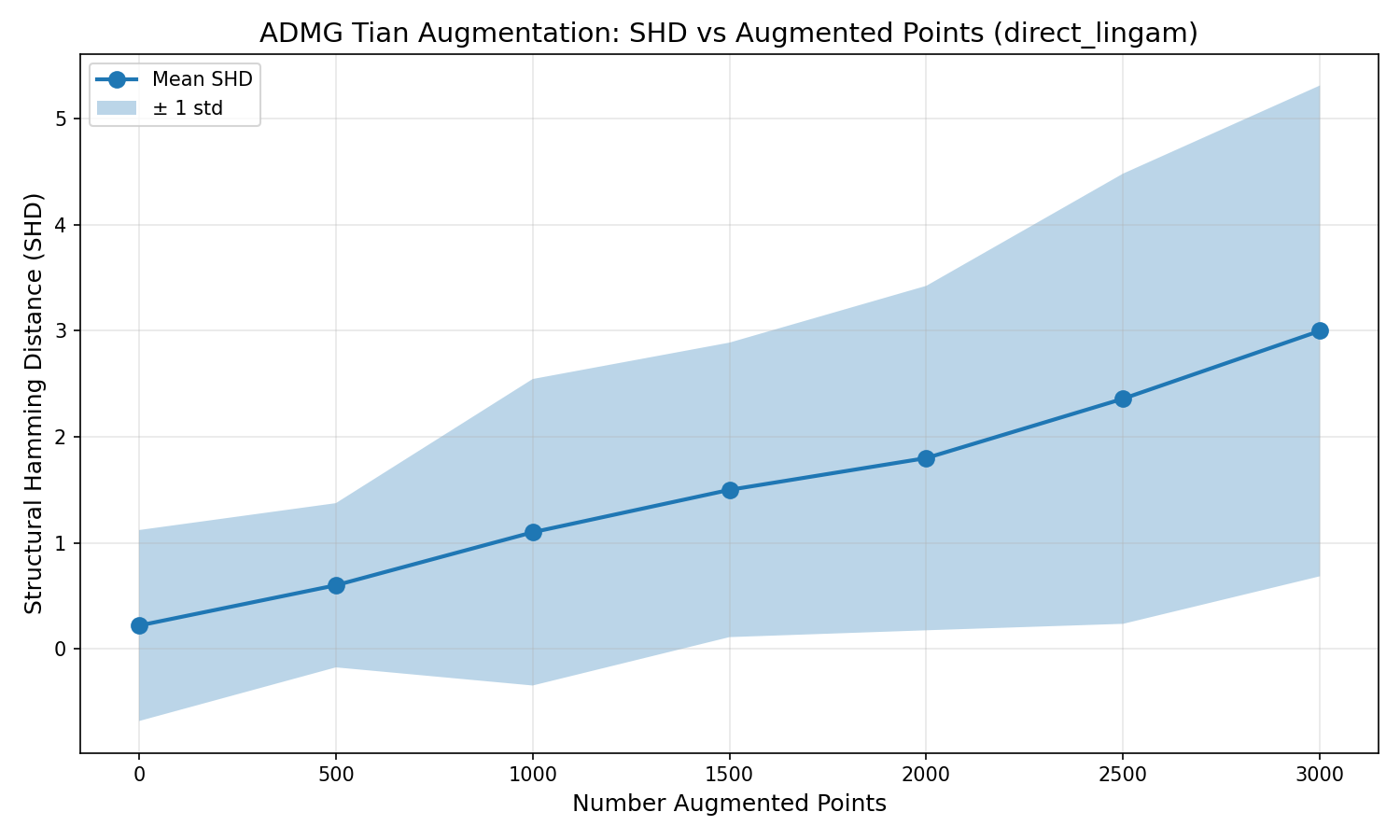}
    \caption{ADMGTian}
    \label{fig:lingam_admg_tian}
\end{subfigure}
\hfill
\begin{subfigure}[t]{0.48\columnwidth}
    \centering
    \includegraphics[width=\textwidth]{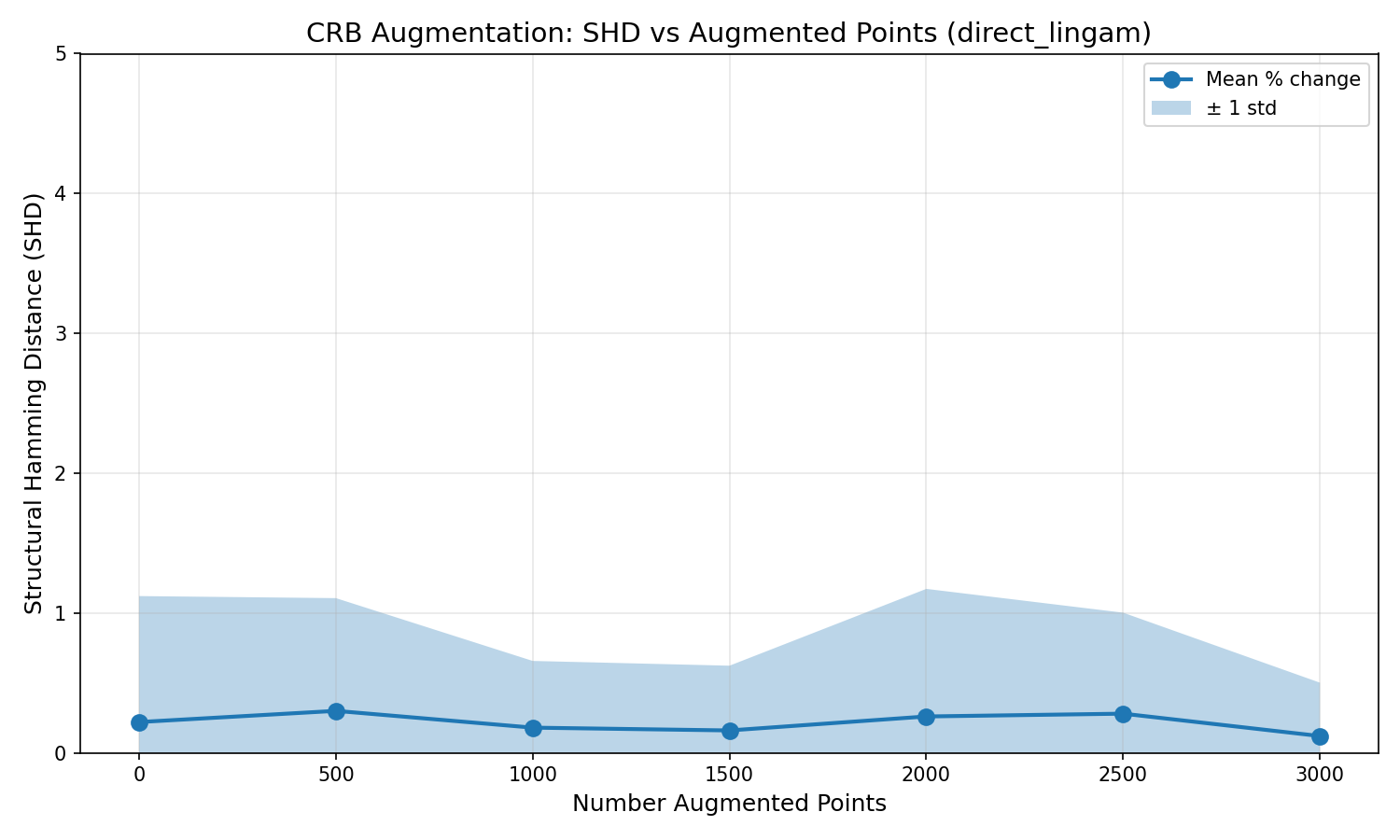}
    \caption{CRB}
    \label{fig:lingam_carb}
\end{subfigure}
\caption{Performance of DirectLINGAM algorithm.}
\label{fig:directLinGAM}
\end{figure}

\section{Empirical Validation} \label{sec: real data experiments}

In this section, we empirically evaluate the CRB method. Our experiments fall into two categories: (1) settings where the true causal graph is known and we seek to incorporate this structural information into the dataset, and (2) settings where we assess CRB as a principled approach to data augmentation, comparing it against existing methods where the graph is unknown, but learnable from causal discovery.
 Additional synthetic data experiments are in Ap.~\ref{sec: synthetic data exp}.

\subsection{Experimental Setup}

\paragraph{Prediction Model.}
For all prediction tasks, we use XGBoost~\citep{chen2016xgboost}, a gradient boosted tree model well-suited for tabular data. Hyperparameters are tuned using Optuna~\citep{akiba2019optuna} to ensure good quality of prediction and flexibility across diverse datasets~\citep{shwartz2022tabular, grinsztajn2022tree}. See Ap.~\ref{app:hyperparams} for details. Ap.~\ref{app:nn_results} includes additional experiments with neural networks.
\paragraph{Augmentation Methods.}
We compare CRB against several baseline augmentation methods: CTGAN~\citep{bowles2018gan}, a GAN-based approach for tabular data; TVAE~\citep{xu1907modeling}, a variational autoencoder adapted for tables; TabDDPM~\citep{kotelnikov2023tabddpm}, a diffusion-based generative model; ARF~\citep{watson2023adversarial}, which uses adversarial random forests for density estimation; NFLOW~\citep{durkan2019neural}, a normalizing flow approach. We also compare against ADMGTian~\citep{teshima2021incorporating}, the prior method that incorporates causal structure into augmentation. Additionally, we include a no-augmentation baseline where the predictor is trained solely on the original data.
All baseline were grid seached for every sample size and dataset; the details of which can be found in Ap.~\ref {app:hyperparams}

\subsection{Augmentation with Known Causal Graphs}

Our first experiment investigates whether incorporating known causal structure into data augmentation can improve downstream prediction accuracy. In this setting, the direct competitor to our method is ADMGTian~\citep{teshima2021incorporating}, which also leverages causal graph information. We additionally include non-causal augmentation methods (CTGAN, TVAE, DDPM, ARF, NFLOW) to verify that any observed improvements stem from the principled use of causal structure rather than simply from increasing the number of training points.

\textbf{Evaluation Protocol.} 
For evaluation, we focus on the prediction accuracy of downstream regression tasks as measured by mean squared error (MSE).
Apart from practical use to improve prediction, \citet{jiang2025tabstruct} utilized the average MSE across all possible label choices as a measure of augmentation quality. Hence, we train many models on each dataset to predict each variable using the remaining variables as covariates. Predictors are trained on the augmented data, for which we evaluate mean squared error (MSE) on held-out test data.

\textbf{Dataset: Causal Chambers.} A fundamental challenge in evaluating causal data augmentation methods is obtaining real-world datasets where the true causal graph is known. This is a surprisingly difficult requirement~\citep{brouillard2024landscape}, as most datasets lack absolute knowledge of a ground-truth graph and rely on expert opinions that may contain errors. To address this challenge, we utilize data from the Causal Chamber~\citep{gamella2025causal}, a physical experimental platform specifically designed for causal research. The Causal Chamber consists of modular hardware components---including light sources, sensors, polarizers, and other optical elements---whose physical interactions define a known causal structure. Because the data-generating process is governed by well-understood physical laws, the causal graph is determined by the experimental apparatus itself rather than by statistical estimation or expert judgment. This provides a unique opportunity to evaluate causal methods against ground truth, but it is worth noting that similar settings emerge in many engineering applications.

\textbf{Results.} Figure~\ref{fig:known_graph_mean_mse} presents the mean MSE across all variables for each augmentation method using 100 training samples. CRB achieves the lowest mean MSE, outperforming both the no-augmentation baseline and all competing methods. Notably, ADMG-Tian and the other non-causal methods (ARF, CTGAN, TVAE) perform worse than no augmentation, suggesting that naive augmentation without causal constraints can degrade predictive performance.

\begin{figure}[htbp]
    \centering
    \includegraphics[width=\columnwidth]{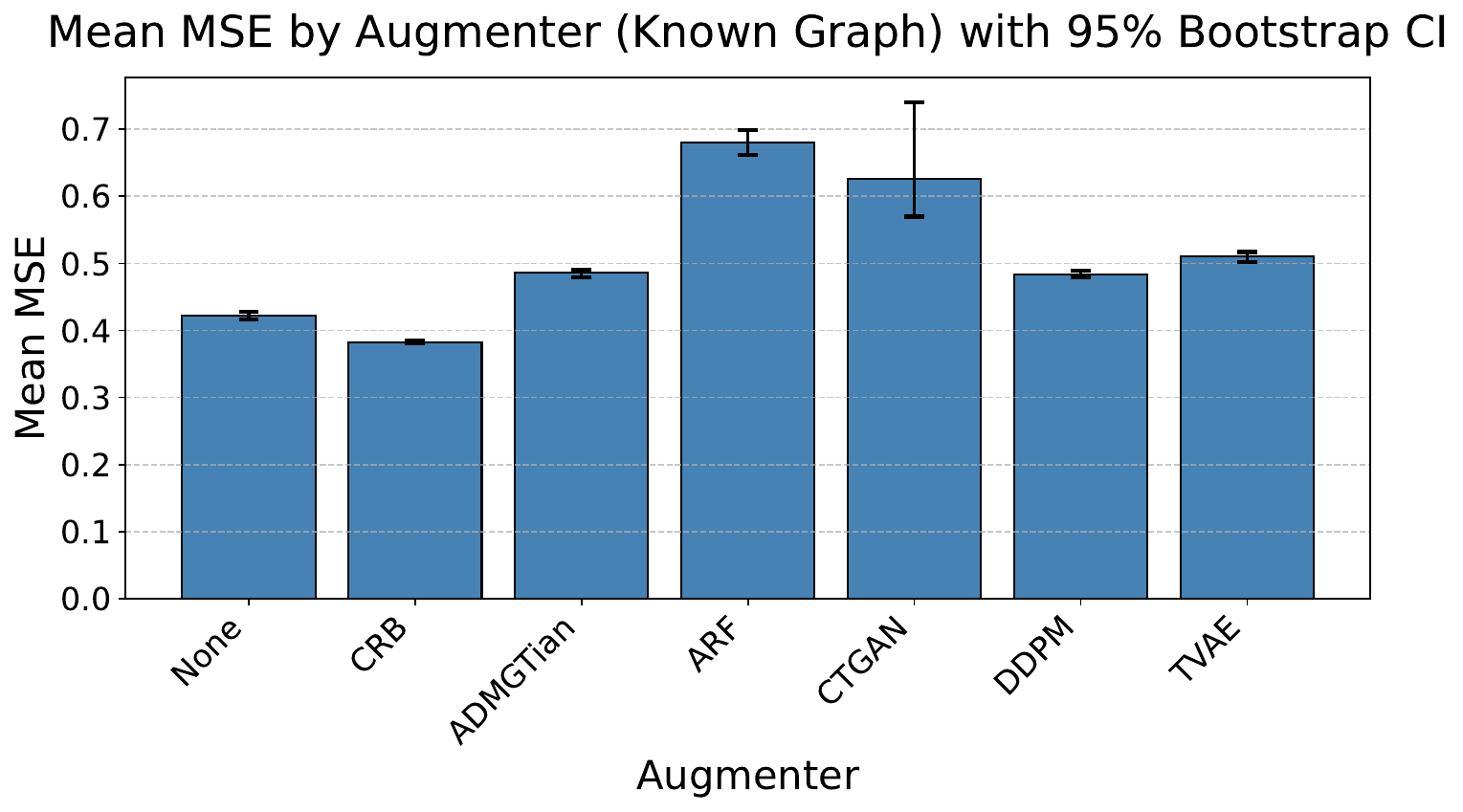}
    \caption{Mean MSE across all variables by augmentation method (Known Graph, $n=100$). Lower is better. CRB achieves the best performance, while non-causal augmenters often increase error relative to no augmentation.}
    \label{fig:known_graph_mean_mse}
\end{figure}

The improvement from CRB is not uniform across variables. Figure~\ref{fig:known_graph_variables} illustrates two contrasting cases: the \texttt{blue} variable, where CRB substantially reduces MSE compared to all other methods, and the \texttt{ir\_1} variable, where improvements are more modest. Table~\ref{tab:known_graph_improvement} summarizes the best and worst per-variable performance for each augmenter relative to no augmentation. CRB shows the largest improvement on its best variable (\texttt{green}, MSE reduced by 0.23) while exhibiting minimal degradation on its worst variable (\texttt{ir\_3}, MSE increased by only 0.006). In contrast, other methods show substantial degradation on some variables---for example, ARF increases MSE by 0.89 on \texttt{angle\_1}. Full per-variable results are provided in Ap.~\ref{app:per_variable}. 

\begin{figure}[htbp]
    \centering
    \includegraphics[width=\columnwidth]{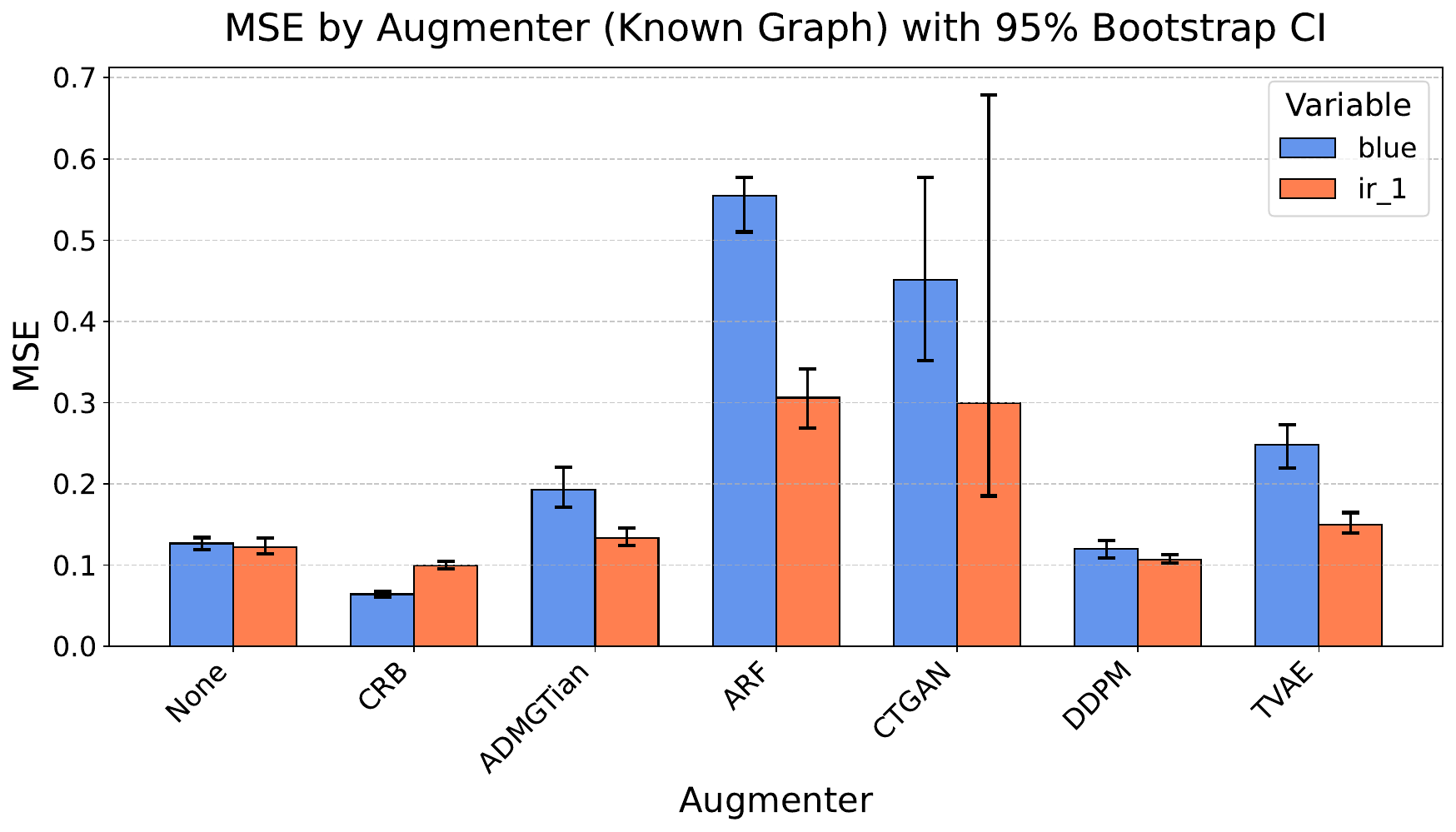}
    \caption{Per-variable MSE comparison (Known Graph, $n=100$).}
    \label{fig:known_graph_variables}
\end{figure}


\begin{table}[htbp]
\centering
\caption{Best and worst per-variable performance relative to no augmentation. ``Improv.'' corresponds to MSE reduction (positive = better) on the best performing variable; ``Worsen.'' corresponds to MSE increase on the worst-performing variable.}
\label{tab:known_graph_improvement}
\scalebox{.8}{
\begin{tabular}{lcccc}
\toprule
Method & Best Var. & Improv. & Worst Var. & Worsen. \\
\midrule
CRB & green & 0.227 & ir\_3 & 0.006 \\
ADMG-Tian & angle\_1 & 0.000 & l\_21 & 0.169 \\
ARF & l\_12 & 0.118 & angle\_1 & 0.889 \\
CTGAN & l\_12 & 0.001 & green & 0.451 \\
DDPM & red & 0.016 & l\_11 & 0.244 \\
TVAE & current & $-$0.024 & green & 0.221 \\
\bottomrule
\end{tabular}}
\end{table}

\subsection{Scaling}
Our theoretical results for linear models suggest that the improvement from incorporating causal structure diminishes as the number of training samples grows. To empirically validate this prediction, we conducted experiments on the Causal Chamber dataset with varying sample sizes.
As shown in Fig.~\ref{fig:known_graph_samples_mean_mse}, this intuition from simple models aligns with observations from the real-world dataset. CRB outperforms no augmentation for small data sizes, and the difference gradually decreases for larger datasets. 
On the other hand, the ADMGTian method fails to outperform no augmentation for all sample sizes.
\begin{table*}[htbp]
\centering
\caption{Mean MSE ($\downarrow$) across datasets with learned causal graph (100 samples). 95\% CI in parentheses. Full results in Ap.~\ref{app:full_tables}.}
\label{tab:mse_summary}
\small
\begin{tabular}{lccccc}
\toprule
Method & Boston & Sachs & Wine (red) & Wine (white) & Causal Chamber \\
\midrule
ADMGTian & 1.886 {\tiny (1.45, 2.13)} & \textbf{0.596} {\tiny (0.54, 0.67)} & 0.831 {\tiny (0.79, 0.86)} & 0.874 {\tiny (0.83, 0.93)} & 0.474 {\tiny (0.47, 0.48)} \\
ARF & 0.625 {\tiny (0.59, 0.70)} & 0.830 {\tiny (0.81, 0.85)} & 0.981 {\tiny (0.92, 1.02)} & 1.034 {\tiny (1.03, 1.04)} & 0.738 {\tiny (0.71, 0.76)} \\
CTGAN & 0.612 {\tiny (0.58, 0.66)} & 0.876 {\tiny (0.81, 0.95)} & 0.769 {\tiny (0.73, 0.80)} & 0.912 {\tiny (0.86, 1.02)} & 0.579 {\tiny (0.55, 0.61)} \\
CRB & 0.457 {\tiny (0.45, 0.47)} & \textbf{0.766} {\tiny (0.55, 1.50)} & 0.635 {\tiny (0.62, 0.65)} & \textbf{0.689} {\tiny (0.67, 0.71)} & \textbf{0.383} {\tiny (0.38, 0.38)} \\
DDPM & \textbf{0.371} {\tiny (0.36, 0.39)} & \textbf{0.545} {\tiny (0.48, 0.71)} & \textbf{0.586} {\tiny (0.58, 0.60)} & 0.757 {\tiny (0.73, 0.78)} & 0.483 {\tiny (0.48, 0.49)} \\
NFLOW & 0.489 {\tiny (0.48, 0.51)} & 0.783 {\tiny (0.72, 0.85)} & 0.668 {\tiny (0.63, 0.72)} & 0.799 {\tiny (0.77, 0.84)} & 0.537 {\tiny (0.50, 0.60)} \\
None & 0.364 {\tiny (0.35, 0.38)} & 0.998 {\tiny (0.99, 1.00)} & 0.587 {\tiny (0.58, 0.60)} & 0.679 {\tiny (0.67, 0.68)} & 0.422 {\tiny (0.42, 0.43)} \\
TVAE & 0.531 {\tiny (0.50, 0.57)} & 0.795 {\tiny (0.72, 0.86)} & 0.764 {\tiny (0.73, 0.80)} & 0.837 {\tiny (0.80, 0.90)} & 0.543 {\tiny (0.53, 0.55)} \\
\bottomrule
\end{tabular}
\end{table*}

\begin{figure}[htbp]
    \centering
    \includegraphics[width=\columnwidth]{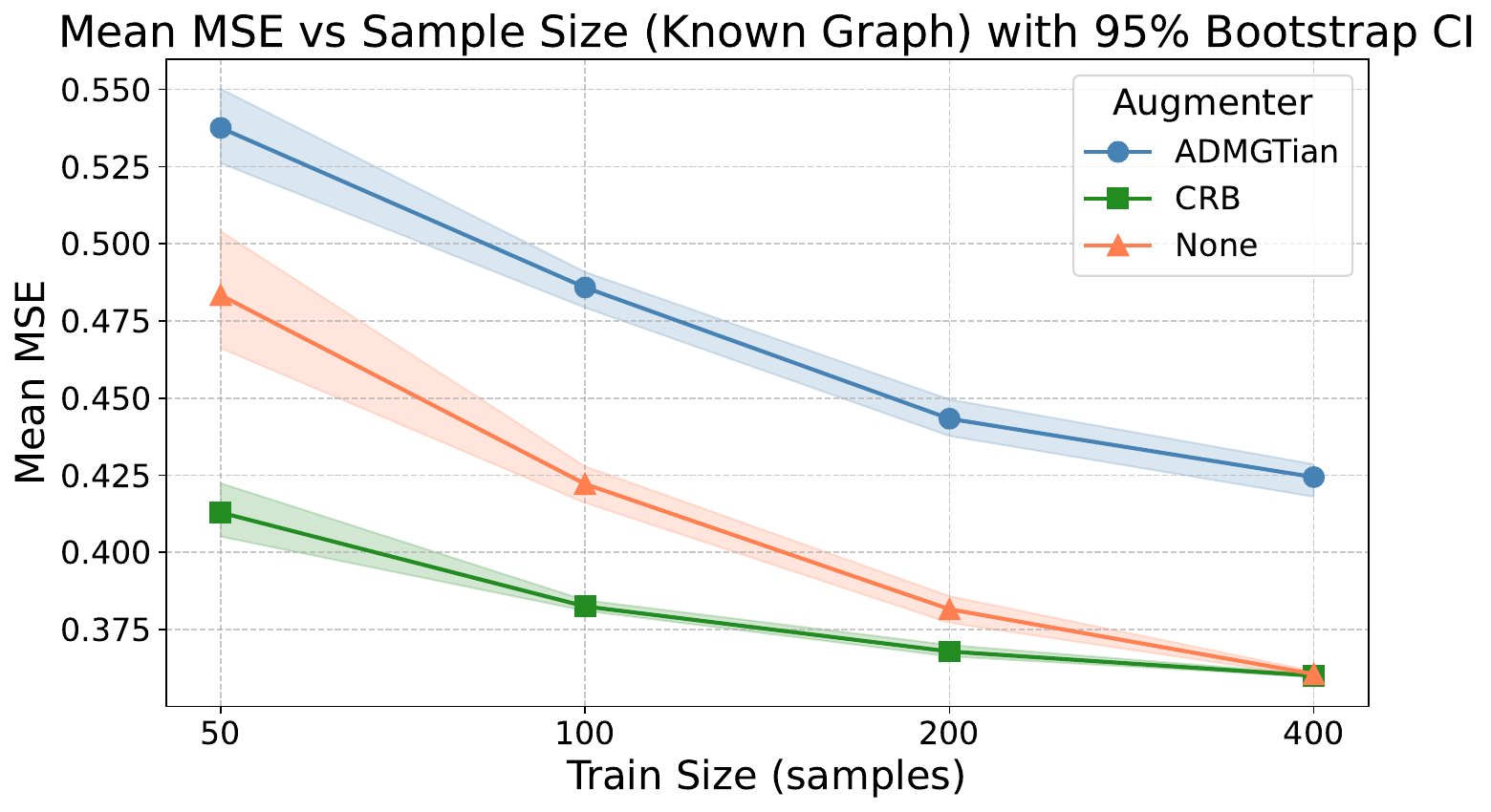}
    \caption{Mean MSE across all variables as sample size increases (Known Graph). Lower is better. CRB maintains strong performance across sample sizes.}
    \label{fig:known_graph_samples_mean_mse}
\end{figure}
\subsection{Unknown graph}
In this section, we evaluate the setting where the true causal graph is unknown and must be learned from data. Our goal is to demonstrate that CRB can serve as an effective synthetic data generation method, comparable to the best data augmentation approaches, even when relying on discovered rather than known causal structure. For these tests, we evaluate our metrics on new data points generated by the augmentation method, excluding the original dataset. The ``no augmentation'' case includes only the original data. 

For both CRB and ADMG-Tian, we employ a causal discovery step prior to augmentation. Following \citet{teshima2021incorporating}, we use the DirectLiNGAM~\citep{shimizu2011directlingam} algorithm to learn the causal graph from the observed data (details in Ap.~\ref{app:cd_details}). We continue to use the Causal Chamber dataset, as we know the underlying system satisfies the assumptions required by CRB: there exists a true causal graph, and the data contain no unobserved confounders or self-loops. Additionally, to facilitate comparison with prior work, we evaluate on several widely recognized datasets used by \citet{teshima2021incorporating}: Sachs~\citep{sachs2005causal}, Boston Housing~\citep{harrison1978hedonic}, Red Wine Quality, and White Wine Quality~\citep{wine_quality_186}.

Table~\ref{tab:mse_summary} presents the mean MSE across all variables for $100$ samples. The results for bigger samples sizes and distribution metrics are in Ap.~\ref{app:full_tables}. Overall, CRB and TabDDPM are the two best-performing methods, aligning with previous observations~\citep{jiang2025tabstruct}. On Causal Chambers, CRB achieves the best performance (0.383), outperforming all methods, including TabDDPM (0.483) and no augmentation (0.422). The same holds for White Wine, where CRB outperforms all augmentation methods and performs comparably to no augmentation. For Red Wine, TabDDPM is the best augmenter, performing on par with no augmentation, with CRB as the second best. On Sachs, all augmentation methods help, with TabDDPM, CRB, and ADMG-Tian performing best. On Boston, TabDDPM performs best with CRB again in second place.

\section{Conclusion}

We provide a simple approach to augmenting data using a known causal structure. This approach helps integrate known causal knowledge into a dataset, either from previous experiments or a learned causal structure (thereby integrating the assumption of an existing causal structure).

Our extensive empirical results show promising improvement in prediction tasks. Compared to the current state of the art, our approach is a more-consistent method for incorporating causal knowledge into data augmentation that does not require density estimation. The downside of our approach is its reliance on trained models, whose performance declines with weaker signals.

The results from experiments with learned causal graphs highlight important considerations when combining causal discovery with data augmentation. Errors or inaccuracies in the learned graph structure can significantly limit the benefits of structure-based augmentation, as incorrect edge orientations or missing edges may violate the independence assumptions on which our method relies. As methods for causal discovery improve, we expect causal data augmentation to become significantly more useful. Furthermore, different causal discovery methods operate under varying assumptions. These underlying assumptions can interact with the augmentation procedure in complex and sometimes unpredictable ways, because the learned graph may be valid only under the specific conditions assumed by the discovery algorithm.
It is not immediately obvious which causal discovery algorithms are optimal for data augmentation. Understanding these interactions and identifying which properties of learned causal graphs are most critical for effective augmentation are interesting and important areas for future work.

Data augmentation that preserves and reinforces causal structure has potentially vast applications for personal privacy because of its ability to generate synthetic points with no direct personal connections while retaining both inter-variable relationships and causal structure.

\section*{Acknowledgements}
Support for Bijan Mazaheri and Sophia Xiao was provided by the Advanced Research Concepts (ARC) COMPASS program, sponsored by the Defense Advanced Research Projects Agency (DARPA) under agreement
number HR001-25-3-0212.

\section*{Impact Statement}

This paper presents work whose goal is to advance the field of machine learning. There are many potential societal consequences of our work, none of which we feel must be specifically highlighted here. Use of the Boston Housing Dataset \citep{harrison1978hedonic} has been criticized. However, since the dataset is used \citet{teshima2021incorporating}, we elected to also include it for the sake of comparison.

\bibliography{refs}


\clearpage
\appendix
\thispagestyle{empty}

\onecolumn
\section*{Supplementary Materials}

\section{Algorithm Pseudocode} \label{sec: pseudocode}

Alg.~\ref{alg:pseudocode} provides the pseudocode for the CRB algorithm.

\begin{algorithm}
\caption{Causal Data Augmentation}
\begin{algorithmic}[1]
\REQUIRE DAG $\mathcal{G}$, dataset $\mathcal{V}$, regression model class $\mathcal{F}$, number of synthetic samples $M$
\ENSURE Synthetic dataset $\tilde{\mathcal{X}} = \{\tilde{\mathbf{x}}^{(1)}, \ldots, \tilde{\mathbf{x}}^{(M)}\}$

\STATE \textbf{// Learning Phase}
\STATE $\pi \gets \text{TopologicalSort}(\mathcal{G})$ \COMMENT{Get causal ordering}
\FOR{each variable $X_j \in \bvec{V}$ in order $\pi$}
    \STATE $\PA_j \gets \text{Parents}(X_j, \mathcal{G})$
\IF{$|\PA_j| = 0$} 
        \STATE $\text{Dist}_j \gets \{x_j^{(1)}, \ldots, x_j^{(N)}\}$ \COMMENT{Store empirical distribution}
    \ELSE 
        \STATE $\mathbf{X}_{\PA_j} \gets [\PA_j \text{ values from } \mathcal{D}]$ \COMMENT{Parent features}
        \STATE $\mathbf{y}_j \gets [x_j^{(1)}, \ldots, x_j^{(N)}]$ \COMMENT{Target values}
        \STATE $\hat{f}_j \gets \text{Train}(\mathcal{F}, \mathbf{X}_{\PA_j}, \mathbf{y}_j)$ \COMMENT{Learn regression model}
        \STATE $\hat{\mathbf{y}}_j \gets \hat{f}_j(\mathbf{X}_{\PA_j})$ \COMMENT{Predictions}
        \STATE $\hat{\boldsymbol{\varepsilon}}_j \gets \mathbf{y}_j - \hat{\mathbf{y}}_j$ \COMMENT{Compute residuals}
    \ENDIF
\ENDFOR

\STATE
\STATE \textbf{// Generation Phase}
\STATE $\tilde{\mathcal{D}} \gets \emptyset$
\FOR{$m = 1$ to $M$}
    \STATE $\tilde{\mathbf{x}}^{(m)} \gets$ empty vector
    \FOR{each variable $X_j \in \bvec{V}$ in order $\pi$}
\IF{$|\PA_j| = 0$} 
            \STATE $\tilde{x}_j^{(m)} \gets \text{Sample}(\text{Dist}_j)$ \COMMENT{Sample from marginal}
        \ELSE 
            \STATE $\tilde{\mathbf{x}}_{\PA_j}^{(m)} \gets [\text{parent values from } \tilde{\mathbf{x}}^{(m)}]$
            \STATE $\tilde{\varepsilon}_j \gets \text{Sample}(\hat{\boldsymbol{\varepsilon}}_j)$ \COMMENT{Sample residual}
            \STATE $\tilde{x}_j^{(m)} \gets \hat{f}_j(\tilde{\mathbf{x}}_{\PA_j}^{(m)}) + \tilde{\varepsilon}_j$ \COMMENT{Generate value}
        \ENDIF
    \ENDFOR
    \STATE $\tilde{\mathcal{D}} \gets \tilde{\mathcal{D}} \cup \{\tilde{\mathbf{x}}^{(m)}\}$
\ENDFOR

\RETURN $\tilde{\mathcal{D}}$
\end{algorithmic}
\label{alg:pseudocode}
\end{algorithm}

\section{Theoretical Analysis of Causal Structure in Regression} \label{sec: theory}

\subsection{Introduction and Setup}

In this section, we investigate how incorporating additional information in the form of graph knowledge can improve the quality of prediction for regression tasks. We present a theoretical analysis of this setting, establishing conditions under which knowledge of the causal DAG provably reduces estimation variance and, consequently, prediction error.

Throughout this section, we operate under the following assumptions:

\begin{assumption}[Linear Gaussian Structural Causal Model]
\label{assump:linear_gaussian_new}
The data-generating process follows a linear Gaussian structural causal model (SCM). That is, for each variable $V_j \in \bvec{V}$, for $j = 1, 2, \dots, n $, the corresponding structural equation has the form
\begin{equation}
V_j = \sum_{V_i \in \PA(V_j)} \beta_{ij} V_i + \varepsilon_j,
\end{equation}
where $\beta_{ij} \in \mathbb{R}$ are the structural coefficients and $\varepsilon_j \sim \mathcal{N}(0, \sigma_j^2)$ are mutually independent Gaussian noise terms. \\
This system can be written in matrix form as 
\begin{equation}
\mathbf{V} = \mathbf{B}\mathbf{V} + \boldsymbol{\varepsilon},
\end{equation}
where $\mathbf{B} \in \mathbb{R}^{d \times d}$ is a strictly lower-triangular matrix (under a topological ordering of the variables) of structural coefficients such that $B_{ji} = \beta_{ij}$ if $V_i \in \text{PA}(V_j)$ and $0$ otherwise, and $\boldsymbol{\varepsilon} \sim \mathcal{N}(\mathbf{0}, \boldsymbol{\Sigma_\varepsilon})$ is a multivariate Gaussian noise vector with covariance matrix $\boldsymbol{\Sigma_\varepsilon} = \text{diag}(\sigma_1^2, \dots, \sigma_d^2)$.
\end{assumption}

\begin{assumption}[Non-Degenerate Noise]
\label{assump:nonzero_noise_new}
All noise variances are strictly positive, i.e., $\sigma_j^2 > 0$ for all $j \in \{1, \ldots, n\}$. This ensures that each variable exhibits genuine stochastic variation beyond what is explained by its parents.
\end{assumption}

\begin{assumption}[Known Causal Structure]
\label{assump:known_dag_new}
We have access to the true causal DAG $\mathcal{G} = (\bvec{V}, \bvec{E})$ that generated the data. The graph correctly specifies all direct causal relationships: $(V_i, V_j) \in \bvec{E}$ if and only if $V_i$ is a direct cause of $V_j$ in the underlying SCM.
\end{assumption}

Under these assumptions, we establish three main results: (1) Causal-Residual Bootstrapping is asymptotically equivalent to constrained maximum likelihood estimation, (2) enforcing correct DAG constraints provably reduces parameter variance, and (3) this variance reduction translates to improved prediction accuracy as measured by mean squared error.

\subsection{Equivalence to Constrained Maximum Likelihood Estimation} \label{subsec: eq to constrained MLE}

We now show that Causal-Residual Bootstrapping with linear regression is equivalent to maximum likelihood estimation under the DAG-constrained model.

\subsubsection{Likelihood Factorization over the DAG}

Under the linear Gaussian SCM (Assumption~\ref{assump:linear_gaussian_new}), each conditional distribution takes the form:
\begin{equation}
V_j \mid \PA(V_j) \sim \mathcal{N}\left(\sum_{V_i \in \PA(V_j)} \beta_{ij} V_i, \sigma_j^2\right).
\end{equation}

By the Markov property induced by the DAG $\mathcal{G}$, the joint distribution factorizes according to the graph structure:
\begin{equation}
p(V_1, V_2, \ldots, V_n) = \prod_{j=1}^{n} p(V_j \mid \PA(V_j)).
\end{equation}

Taking logarithms, the log-likelihood decomposes as a sum over local conditional log-likelihoods:
\begin{equation}
\log p(V_1, V_2, \ldots, V_n) = \sum_{j=1}^{n} \log p(V_j \mid \PA(V_j)).
\end{equation}

Given a dataset $\mathcal{D} = \{(v_1^{(i)}, \ldots, v_n^{(i)})\}_{i=1}^{N}$ of $N$ i.i.d.\ observations, the total log-likelihood is:
\begin{equation}
\ell(\boldsymbol{\theta}; \mathcal{D}) = \sum_{i=1}^{N} \sum_{j=1}^{n} \log p(v_j^{(i)} \mid \mathbf{v}_{\PA(j)}^{(i)}; \boldsymbol{\theta}_j),
\end{equation}
where $\boldsymbol{\theta}_j = (\boldsymbol{\beta}_j, \sigma_j^2)$ denotes the parameters of the $j$-th conditional distribution, with $\boldsymbol{\beta}_j = (\beta_{ij})_{V_i \in \PA(V_j)}$.

\subsubsection{Decomposition into Independent Local Problems}

A key observation is that the parameters $\boldsymbol{\theta}_j$ for different variables $j$ appear in disjoint terms of the log-likelihood. Specifically, rearranging the sums:
\begin{equation}
\ell(\boldsymbol{\theta}; \mathcal{D}) = \sum_{j=1}^{n} \underbrace{\sum_{i=1}^{N} \log p(v_j^{(i)} \mid \mathbf{v}_{\PA(j)}^{(i)}; \boldsymbol{\theta}_j)}_{\ell_j(\boldsymbol{\theta}_j)}.
\end{equation}

Since $\boldsymbol{\theta}_j$ only appears in $\ell_j(\boldsymbol{\theta}_j)$, the global maximization problem separates into $n$ independent local problems:
\begin{equation}
\argmax_{\boldsymbol{\theta}} \ell(\boldsymbol{\theta}; \mathcal{D}) = \left\{ \argmax_{\boldsymbol{\theta}_j} \ell_j(\boldsymbol{\theta}_j) \right\}_{j=1}^{n}.
\end{equation}

\subsubsection{Local MLE Equals Least Squares Regression}

For the Gaussian conditional $V_j \mid \PA(V_j) \sim \mathcal{N}(\mathbf{v}_{\PA(j)}^\top \boldsymbol{\beta}_j, \sigma_j^2)$, the local log-likelihood is:
\begin{equation}
\ell_j(\boldsymbol{\beta}_j, \sigma_j^2) = -\frac{N}{2}\log(2\pi\sigma_j^2) - \frac{1}{2\sigma_j^2} \sum_{i=1}^{N} \left(v_j^{(i)} - \mathbf{v}_{\PA(j)}^{(i)\top} \boldsymbol{\beta}_j\right)^2.
\end{equation}

Maximizing with respect to $\boldsymbol{\beta}_j$ is equivalent to minimizing the sum of squared residuals:
\begin{equation}
\hat{\boldsymbol{\beta}}_j = \argmin_{\boldsymbol{\beta}_j} \sum_{i=1}^{N} \left(v_j^{(i)} - \mathbf{v}_{\PA(j)}^{(i)\top} \boldsymbol{\beta}_j\right)^2,
\end{equation}
which is precisely ordinary least squares (OLS) regression of $V_j$ on its parents $\PA(V_j)$.

The MLE for the noise variance is the empirical variance of the residuals:
\begin{equation}
\hat{\sigma}_j^2 = \frac{1}{N} \sum_{i=1}^{N} \left(v_j^{(i)} - \mathbf{v}_{\PA(j)}^{(i)\top} \hat{\boldsymbol{\beta}}_j\right)^2.
\end{equation}

\begin{proposition}[CRB Learning Phase Equals Constrained MLE]
\label{prop:crb_mle}
The learning phase of Causal-Residual Bootstrapping, which performs linear regression of each variable $V_j$ on its parents $\PA(V_j)$, computes the maximum likelihood estimates under the DAG-constrained linear Gaussian model.
\end{proposition}

\begin{proof}
The CRB learning phase fits $\hat{f}_j$ by regressing $V_j$ on $\PA(V_j)$. With linear regression, this solves:
\[
\hat{\boldsymbol{\beta}}_j = \argmin_{\boldsymbol{\beta}_j} \sum_{i=1}^{N} \left(v_j^{(i)} - \mathbf{v}_{\PA(j)}^{(i)\top} \boldsymbol{\beta}_j\right)^2.
\]
As shown above, this is exactly the local MLE for $\boldsymbol{\theta}_j$. Since the global MLE decomposes into independent local MLEs, and CRB solves each local problem, it achieves the global MLE.
\end{proof}

\begin{remark}
The DAG structure constrains which parameters are estimated: only coefficients $\beta_{ij}$ for edges $(V_i, V_j) \in \bvec{E}$ are included. Coefficients for non-edges are implicitly constrained to zero. This is in contrast to an unconstrained model that would estimate a full covariance matrix without imposing the DAG structure.
\end{remark}

\subsection{Correspondence Between Precision Factorization and DAG Constraints}
Here we show that DAG constraints on a linear Gaussian SCM correspond to constraints on the $\boldsymbol{U}$ factor of the UDU decomposition of the precision matrix of $\bvec{V}$. 

\subsubsection{Covariance Matrix Representation}

The linear Gaussian SCM from Assumption~\ref{assump:linear_gaussian_new} corresponds to a multivariate Gaussian distribution over $n$ variables $\bvec{V} = (V_1, \ldots, V_n)$. It is fully characterized by its mean vector $\boldsymbol{\mu} \in \mathbb{R}^n$ and covariance matrix $\boldsymbol{\Sigma} \in \mathbb{R}^{n \times n}$. Without loss of generality, we assume zero mean ($\boldsymbol{\mu} = \mathbf{0}$), as centering does not affect the causal structure.

By definition, the covariance matrix $\boldsymbol{\Sigma}$ is symmetric positive semi-definite.
Under Assumption~\ref{assump:nonzero_noise_new} (non-degenerate noise), the covariance matrix of $\bvec{V}$ is strictly positive definite ($\boldsymbol{\Sigma} \succ 0$); therefore, its inverse, known as the precision matrix and denoted as $\boldsymbol{\Sigma}^{-1}:=\boldsymbol{\Omega}$ exists and is also symmetric positive definite. 

\subsubsection{LDL and UDU Decomposition}
Any symmetric positive definite matrix $\boldsymbol{A} \in \mathbb{R}^n$ admits a unique LDL decomposition
\begin{equation}
\boldsymbol{A} = \mathbf{L} \mathbf{D} \mathbf{L}^\top,
\end{equation}
where:
\begin{itemize}
    \item $\mathbf{L} \in \mathbb{R}^{n \times n}$ is a lower triangular matrix with ones on the diagonal ($L_{ii} = 1$ for all $i$),
    \item $\mathbf{D} = \diag(d_1, \ldots, d_n) \in \mathbb{R}^{n \times n}$ is a diagonal matrix with strictly positive entries ($d_i > 0$).
\end{itemize}

The LDL decomposition is closely related to the standard Cholesky decomposition $\boldsymbol{A} = \mathbf{C}\mathbf{C}^\top$, where $\mathbf{C} = \mathbf{L}\mathbf{D}^{1/2}$.

Similarly, any symmetric positive definite matrix $\boldsymbol{A} \in \mathbb{R}^n$ admits a unique UDU decomposition

\begin{equation}
\boldsymbol{A} = \mathbf{U} \mathbf{D} \mathbf{U}^\top,
\end{equation}
where:
\begin{itemize}
    \item $\mathbf{U} \in \mathbb{R}^{n \times n}$ is an upper triangular matrix with ones on the diagonal ($U_{ii} = 1$ for all $i$),
    \item $\mathbf{D} = \diag(d_1', \ldots, d_n')$ is a diagonal matrix with strictly positive entries ($d_i' > 0$).
\end{itemize} 

\subsubsection{Causal Interpretation under Topological Ordering}

Let the variables $V_1, \dots, V_d$ be ordered according to a topological ordering $\pi$ of the DAG $\mathcal{G}$ (i.e., if $(V_i, V_j) \in \mathbf{E}$, then $i < j$ in the ordering). Under Assumption~\ref{assump:linear_gaussian_new}, the structural equations can be written in matrix form as $(\mathbf{I} - \mathbf{B})\bvec{V} = \boldsymbol{\varepsilon}$. 

Rearranging, we obtain the expression
\begin{equation}
\bvec{V} = (\mathbf{I}- \mathbf{B})^{-1}\boldsymbol{\varepsilon}
\end{equation}
Taking the covariance of $\bvec{V}$, we have:

\begin{equation}\label{equation:cov_v}
\boldsymbol{\Sigma} = \Cov(\bvec{V}) = (\mathbf{I} - \mathbf{B})^{-1} \boldsymbol{\Sigma_\varepsilon } (\mathbf{I} - \mathbf{B})^{-\top}
\end{equation}

where $\boldsymbol{\Sigma_\varepsilon }= \text{diag}(\sigma_1^2, \dots, \sigma_d^2)$. Note that $\mathbf{B}$ is strictly lower triangular in this ordering, so $\mathbf{I} - \mathbf{B}$ and its inverse are unit lower triangular. Note that \eqref{equation:cov_v} corresponds to the LDL decomposition of the covariance matrix $\boldsymbol{\Sigma}$.

Inverting both sides of Equation~\ref{equation:cov_v}, we obtain the \textit{UDU} decomposition for the precision matrix $\bvec{\Omega}$:
\begin{equation}
\boldsymbol{\Omega} = \boldsymbol{\Sigma^{-1}} = (I - \mathbf{B})^\top \Sigma_\varepsilon^{-1} (I - \mathbf{B})
\end{equation}
Note that $\mathbf{U} := (\mathbf{I} - \mathbf{B})^\top$ is unit upper triangular and $\mathbf{D}_\Omega := \Sigma_\varepsilon^{-1} = \text{diag}(1/\sigma_1^2, \dots, 1/\sigma_d^2)$. 

\begin{proposition}[Causal Meaning of UDU Entries]
\label{prop:udu_causal}
Let $\bvec{V}$ be ordered according to a topological ordering of $\mathcal{G}$. The unique UDU decomposition of the precision matrix $\Omega = \mathbf{U} \mathbf{D} \mathbf{U}^\top$ encodes the direct causal structure as follows:
\begin{enumerate}
    \item \textbf{Direct Effects:} For $i < j$, the super-diagonal entry $U_{ij}$ is the negative of the direct causal effect (structural coefficient) of $V_i$ on $V_j$:
    \begin{equation}
    U_{ij} = -\beta_{ji},
    \end{equation}
    where $\beta_{ji}$ is the coefficient from the SCM $V_j = \sum_{i < j} \beta_{ji} V_i + \varepsilon_j$. Consequently, $U_{ij} \neq 0$ if and only if $V_i \to V_j$ exists in $\mathcal{G}$.
    
    \item \textbf{Precision Scale:} The jth diagonal entry $d_{jj}$ of $\mathbf{D}$ is the inverse of the conditional variance of $V_j$ given its predecessors:
    \begin{equation}
    d_{jj} = \frac{1}{\Var(V_j \mid V_1, \dots, V_{j-1})} = \frac{1}{\sigma_j^2}.
    \end{equation}
\end{enumerate}
\end{proposition}

This decomposition corresponds to the recursive factorization of the joint density implied by the chain rule.

\subsubsection{DAG Constraints as Zero Entries}
The absence of an edge in the DAG corresponds to zero entries in $\mathbf{U}$.

\begin{theorem}[DAG Constraints in UDU Form]
\label{thm:dag_udu}
Let $\Omega = \mathbf{U} \mathbf{D}_\Omega \mathbf{U}^\top$ be the unique UDU decomposition of the precision matrix under a topological ordering. Then, for $i < j$:
\begin{equation}
V_i \notin \text{PA}(V_j) \Rightarrow U_{ij} = 0.
\end{equation}
\end{theorem}

\begin{proof}
As established in Proposition~\ref{prop:udu_causal}, $\mathbf{U} = (\boldsymbol{I} - \mathbf{B})^\top$. The entry $U_{ij}$ for $i < j$ corresponds to the entry $(\boldsymbol{I} - \mathbf{B})_{ji} = -B_{ji}$. In a linear SCM, the structural coefficient $B_{ji}$ is non-zero if and only if $V_i$ is a direct parent of $V_j$. Therefore, if $V_i \notin \text{PA}(V_j)$, then $B_{ji} = 0$, which implies $U_{ij} = 0$. 
\end{proof}
\begin{remark}[Parameterization Perspective]
This result provides an alternative view of DAG-constrained estimation. Rather than thinking of the DAG as constraining which regressions to run, we can equivalently view it as constraining the UDU parameterization of the precision matrix: certain off-diagonal entries of $\mathbf{U}$ are fixed to zero. The unconstrained model estimates all $\frac{n(n-1)}{2}$ sub-diagonal entries, while the DAG-constrained model only estimates the $|\bvec{E}|$ entries corresponding to edges.
\end{remark}

\subsection{Variance Reduction via Parameter Constraints}

We now establish that knowing certain parameters are zero---as implied by the DAG structure---leads to reduced estimation variance.

\subsubsection{Asymptotic Distribution of MLE}

Let $\boldsymbol{\theta} \in \mathbb{R}^p$ denote the vector of parameters in the UDU decomposition of $\Omega$ (specifically, the $d(d-1)/2$ superdiagonal entries of $\mathbf{U}$ and the $d$ diagonal entries of $\mathbf{D}$). Under standard regularity conditions, the maximum likelihood estimator $\hat{\boldsymbol{\theta}}_{\text{MLE}}$ is asymptotically normal:
\begin{equation}
\sqrt{N}(\hat{\boldsymbol{\theta}}_{\text{MLE}} - \boldsymbol{\theta}_0) \xrightarrow{d} \mathcal{N}(\mathbf{0}, \mathcal{I}(\boldsymbol{\theta}_0)^{-1}),
\end{equation}
where $\boldsymbol{\theta}_0$ is the true parameter value and $\mathcal{I}(\boldsymbol{\theta}_0)$ is the Fisher information matrix:
\begin{equation}
\mathcal{I}(\boldsymbol{\theta}_0) = -\mathbb{E}\left[\frac{\partial^2 \log p(\bvec{V}; \boldsymbol{\theta})}{\partial \boldsymbol{\theta} \partial \boldsymbol{\theta}^\top}\right]\bigg|_{\boldsymbol{\theta} = \boldsymbol{\theta}_0}.
\end{equation}

The asymptotic covariance of the MLE is therefore:
\begin{equation}
\Cov(\hat{\boldsymbol{\theta}}_{\text{MLE}}) \approx \frac{1}{N} \mathcal{I}(\boldsymbol{\theta}_0)^{-1}.
\end{equation}

\subsubsection{Constrained vs.\ Unconstrained Estimation}

Now consider two estimation scenarios:
\begin{itemize}
    \item \textbf{Unconstrained (full model):} We need to estimate all parameters $\boldsymbol{\theta}_{\text{full}} \in \mathbb{R}^p$.
    \item \textbf{DAG-constrained model:} If we know that certain parameters are zero, we can partition the parameters as $\boldsymbol{\theta}_{\text{full}} = (\boldsymbol{\theta}_G, \boldsymbol{\theta}_0)$, where $\boldsymbol{\theta}_G$ are the free parameters (corresponding to edges in the DAG) and $\boldsymbol{\theta}_0$ denote the parameters constrained to zero (corresponding to non-edges in the DAG).
\end{itemize}

The Fisher information matrix of the full model has a corresponding block structure:
\begin{equation}\label{equation:fim}
\mathcal{I}_{\text{full}} = \begin{pmatrix}
\mathcal{I}_{GG} & \mathcal{I}_{G0} \\
\mathcal{I}_{0G} & \mathcal{I}_{00}
\end{pmatrix},
\end{equation}
where $\mathcal{I}_{GG}$ corresponds to the free parameters, $\mathcal{I}_{00}$ to the constrained parameters, and $\mathcal{I}_{G0} = \mathcal{I}_{0G}^\top$ captures cross-information.

\subsubsection{Covariance Matrix Inequality}

The key result is that constraining parameters to their known values reduces the estimation variance of the remaining parameters.

\begin{theorem}[Variance Reduction from Known Parameters]
\label{thm:variance_reduction}
Let $\hat{\boldsymbol{\theta}}_G^{\text{(full)}}$ denote the MLE of $\boldsymbol{\theta}_G$ in the unconstrained model, and let $\hat{\boldsymbol{\theta}}_G^{\text{(DAG)}}$ denote the MLE in the DAG-constrained model (where $\boldsymbol{\theta}_0 = \mathbf{0}$ is known). Then:
\begin{equation}
\Cov(\hat{\boldsymbol{\theta}}_G^{\text{(DAG)}}) \preceq \Cov(\hat{\boldsymbol{\theta}}_G^{\text{(full)}}),
\end{equation}
where $\preceq$ denotes the Loewner (positive semi-definite) ordering: $\mathbf{A} \preceq \mathbf{B}$ means $\mathbf{B} - \mathbf{A}$ is positive semi-definite.
\end{theorem}

\begin{proof}
In the unconstrained model, the asymptotic covariance of $\hat{\boldsymbol{\theta}}_G^{\text{(full)}}$ is the upper-left block of $\mathcal{I}_{\text{full}}^{-1}$. Taking the Schur complement of Equation~\ref{equation:fim}, we get

\begin{equation}
(\mathcal{I}_{\text{full}}^{-1})_{GG} = \left(\mathcal{I}_{GG} - \mathcal{I}_{G0}\mathcal{I}_{00}^{-1}\mathcal{I}_{0G}\right)^{-1}.
\end{equation}

In the DAG-constrained model, since $\boldsymbol{\theta}_0$ is fixes at zero, we only need to estimate $\boldsymbol{\theta}_G$. The Fisher information for $\boldsymbol{\theta}_G$ alone is $\mathcal{I}_{GG}$, so:
\begin{equation}
\Cov(\hat{\boldsymbol{\theta}}_G^{\text{(DAG)}}) \approx \frac{1}{N}\mathcal{I}_{GG}^{-1}.
\end{equation}

The matrix $\mathcal{I}_{G0}\mathcal{I}_{00}^{-1}\mathcal{I}_{0G}$ is positive semi-definite (it is a quadratic form in $\mathcal{I}_{G0}$ with positive definite weight $\mathcal{I}_{00}^{-1}$). Therefore:
\begin{equation}
\mathcal{I}_{GG} - \mathcal{I}_{G0}\mathcal{I}_{00}^{-1}\mathcal{I}_{0G} \preceq \mathcal{I}_{GG}.
\end{equation}

Taking inverses (which reverses the ordering for positive definite matrices):
\begin{equation}
\mathcal{I}_{GG}^{-1} \preceq \left(\mathcal{I}_{GG} - \mathcal{I}_{G0}\mathcal{I}_{00}^{-1}\mathcal{I}_{0G}\right)^{-1}.
\end{equation}

This gives us:
\begin{equation}
\Cov(\hat{\boldsymbol{\theta}}_G^{\text{(DAG)}}) \preceq \Cov(\hat{\boldsymbol{\theta}}_G^{\text{(full)}}),
\end{equation}
as claimed.
\end{proof}

\begin{remark}[Intuition]
The variance reduction can be understood intuitively: in the unconstrained model, uncertainty about $\boldsymbol{\theta}_0$ ``spills over'' into uncertainty about $\boldsymbol{\theta}_G$ through the cross-information term $\mathcal{I}_{G0}$. When we know $\boldsymbol{\theta}_0 = \mathbf{0}$, this source of uncertainty is eliminated, leading to more precise estimates of $\boldsymbol{\theta}_G$.
\end{remark}

\subsubsection{Extension to Full Parameter Space}

We can extend the variance reduction result to the full parameter space by embedding the constrained estimator appropriately.

\begin{corollary}[Variance Reduction in Full Parameter Space]
\label{cor:full_variance_reduction}
Define the covariance matrices over the full parameter space $\boldsymbol{\theta}_{\text{full}} = (\boldsymbol{\theta}_G, \boldsymbol{\theta}_0)$:
\begin{itemize}
    \item For the unconstrained model: $\boldsymbol{\Sigma}_{\text{full}} = \mathcal{I}_{\text{full}}^{-1}$.
    \item For the DAG-constrained model: $\boldsymbol{\Sigma}_{\text{DAG}} = \begin{pmatrix} \mathcal{I}_{GG}^{-1} & \mathbf{0} \\ \mathbf{0} & \mathbf{0} \end{pmatrix}$,

    where the zero blocks reflect that $\boldsymbol{\theta}_0$ is known exactly (zero variance).
\end{itemize}
Then:
\begin{equation}
\boldsymbol{\Sigma}_{\text{DAG}} \preceq \boldsymbol{\Sigma}_{\text{full}}.
\end{equation}
\end{corollary}
\begin{proof}
Partition the Fisher information for the full parameter vector
$\boldsymbol{\theta}_{\text{full}}=(\boldsymbol{\theta}_G,\boldsymbol{\theta}_0)$ as
\[
\mathcal I_{\text{full}}
=
\begin{pmatrix}
A & B\\[4pt]
B^\top & C
\end{pmatrix},
\qquad A=\mathcal I_{GG},\; B=\mathcal I_{G0},\; C=\mathcal I_{00}.
\]
Assume $A$ is invertible and the Schur complement $S:=C-B^\top A^{-1}B$ is positive definite (these hold under the usual identifiability / regularity assumptions so that $\mathcal I_{\text{full}}$ is invertible).

By the block inverse (Schur complement) formula,
\[
\mathcal I_{\text{full}}^{-1}
=
\begin{pmatrix}
A^{-1}+A^{-1} B S^{-1} B^\top A^{-1} & -A^{-1} B S^{-1}\\[4pt]
- S^{-1} B^\top A^{-1} & S^{-1}
\end{pmatrix}.
\]
The DAG-constrained estimator (which corresponds to treating $\boldsymbol{\theta}_0$ as known) has covariance
\[
\Sigma_{\mathrm{DAG}}=\begin{pmatrix} A^{-1} & 0\\[4pt] 0 & 0\end{pmatrix}.
\]
Subtracting,
\[
\mathcal I_{\text{full}}^{-1}-\Sigma_{\mathrm{DAG}}
=
\begin{pmatrix}
A^{-1} B S^{-1} B^\top A^{-1} & -A^{-1} B S^{-1}\\[4pt]
- S^{-1} B^\top A^{-1} & S^{-1}
\end{pmatrix}.
\]
Set $K:=A^{-1}B$ and $M:=\begin{pmatrix} K \\ -I \end{pmatrix}$. Then the difference factors as
\[
\mathcal I_{\text{full}}^{-1}-\Sigma_{\mathrm{DAG}} \;=\; M\,S^{-1}\,M^\top,
\]
and since $S^{-1}\succeq 0$ this is positive semi-definite. Hence
\[
\Sigma_{\mathrm{DAG}}\preceq \mathcal I_{\text{full}}^{-1}=\Sigma_{\text{full}},
\]
as required.
\end{proof}

\begin{remark}
This corollary shows that the DAG-constrained estimator dominates the unconstrained estimator in the matrix sense over the \emph{entire} parameter space. The constrained parameters $\boldsymbol{\theta}_0$ contribute zero variance (since they are known), while the free parameters $\boldsymbol{\theta}_G$ have reduced variance compared to the unconstrained case. This provides a complete picture of the efficiency gains from incorporating DAG knowledge.
\end{remark}

\subsection{Preservation of Variance Reduction under Smooth Maps}

We have established variance reduction in the UDU parameter space. We now show that this reduction is preserved when mapping back to the original covariance matrix or to regression coefficients.

\subsubsection{The Delta Method and Loewner Order}

Let $g: \mathbb{R}^p \to \mathbb{R}^q$ be a differentiable function, and let $\hat{\boldsymbol{\theta}}$ be an estimator of $\boldsymbol{\theta}_0$ with asymptotic covariance $\boldsymbol{\Sigma}_{\boldsymbol{\theta}}$. By the delta method:
\begin{equation}
\sqrt{N}(g(\hat{\boldsymbol{\theta}}) - g(\boldsymbol{\theta}_0)) \xrightarrow{d} \mathcal{N}(\mathbf{0}, \mathbf{J}_g \boldsymbol{\Sigma}_{\boldsymbol{\theta}} \mathbf{J}_g^\top),
\end{equation}
where $\mathbf{J}_g = \frac{\partial g}{\partial \boldsymbol{\theta}^\top}\big|_{\boldsymbol{\theta}_0}$ is the Jacobian of $g$ evaluated at the true parameter.

A key property of the Loewner order is that it is preserved under congruence transformations:

\begin{lemma}[Loewner Order Preserved under Congruence]
\label{lem:loewner_congruence}
Let $\mathbf{A}, \mathbf{B} \in \mathbb{R}^{p \times p}$ be symmetric matrices with $\mathbf{A} \preceq \mathbf{B}$. Then for any matrix $\mathbf{M} \in \mathbb{R}^{q \times p}$:
\begin{equation}
\mathbf{M}\mathbf{A}\mathbf{M}^\top \preceq \mathbf{M}\mathbf{B}\mathbf{M}^\top.
\end{equation}
\end{lemma}

\begin{proof}
By definition, $\mathbf{A} \preceq \mathbf{B}$ means $\mathbf{B} - \mathbf{A} \succeq 0$. For any $\mathbf{x} \in \mathbb{R}^q$:
\[
\mathbf{x}^\top (\mathbf{M}\mathbf{B}\mathbf{M}^\top - \mathbf{M}\mathbf{A}\mathbf{M}^\top) \mathbf{x} = \mathbf{x}^\top \mathbf{M}(\mathbf{B} - \mathbf{A})\mathbf{M}^\top \mathbf{x} = (\mathbf{M}^\top\mathbf{x})^\top (\mathbf{B} - \mathbf{A}) (\mathbf{M}^\top\mathbf{x}) \geq 0,
\]
since $\mathbf{B} - \mathbf{A} \succeq 0$.
\end{proof}

\subsubsection{Smoothness of the UDU Decomposition}

We first establish that both the map from UDU parameters to covariance and its inverse are smooth.

The map from UDU parameters $\boldsymbol{\theta} = (\text{vec}(\mathbf{U}), \text{diag}(\mathbf{D}))$ to the covariance matrix is:
\begin{equation}
g_{\Sigma}: \boldsymbol{\theta} \mapsto \boldsymbol{\Sigma} =(\mathbf{U}\mathbf{D}\mathbf{U}^\top)^{-1}.
\end{equation}
This map is the composition of two separate maps:

\begin{enumerate}
    \item The map from the parameters to the UDU factorization of the precision matrix $\boldsymbol{\theta} \mapsto \boldsymbol{\Omega} = \boldsymbol{\mathbf{U}\mathbf{D}\mathbf{U}^\top} $. The entries of $\boldsymbol{\Omega}$ are finite sums and products of the entries of U and D; therefore this map is polynomial.
    \item Matrix inversion: $\boldsymbol{\Omega} \mapsto \boldsymbol{\Sigma}$. Note that matrix inversion is a smooth map on the set of non-singular matrices.
\end{enumerate}

$g_\Sigma$ is the composition of two smooth maps; therefore, it is smooth.

The inverse map $g_{\Sigma}^{-1}: \boldsymbol{\Sigma} \mapsto \boldsymbol{\theta}$ is also smooth on the cone of positive definite matrices. To see this, note that $g_{\Sigma}^{-1}$ is the composition of the smooth inversion map $\boldsymbol{\Sigma} \mapsto \boldsymbol{\Omega} = \boldsymbol{\Sigma}^{-1}$ and the UDU decomposition of $\boldsymbol{\Omega}$. The UDU decomposition $\boldsymbol{\Omega} = \mathbf{U}\mathbf{D}\mathbf{U}^\top$ can be computed recursively (starting from $i=p$ down to $1$):
\begin{align}
D_{ii} &= \Omega_{ii} - \sum_{k=i+1}^{p} U_{ik}^2 D_{kk}, \\
U_{ji} &= \frac{1}{D_{ii}}\left(\Omega_{ji} - \sum_{k=i+1}^{p} U_{jk} U_{ik} D_{kk}\right), \quad \text{for } j < i.
\end{align}
Each entry of $\mathbf{U}$ and $\mathbf{D}$ is a rational function of the entries of $\boldsymbol{\Omega}$ (and thus of $\boldsymbol{\Sigma}$). The denominators involve the diagonal entries $D_{pp}, D_{p-1,p-1}, \ldots$, which are the conditional precisions of the variables. More generally, these denominators are ratios of consecutive trailing principal minors of $\boldsymbol{\Omega}$. For any $\boldsymbol{\Sigma} \succ 0$, the precision matrix $\boldsymbol{\Omega}$ is also positive definite; therefore, these minors (and the resulting $D_{ii}$ values) are strictly positive. 

Since rational functions with non-vanishing denominators are smooth, the map $g_{\Sigma}^{-1}$ is smooth on the positive definite cone.

\subsubsection{Application to Covariance Matrix Estimation}

\begin{corollary}[Variance Reduction for Covariance Estimator]
\label{cor:cov_variance_reduction}
Let $\hat{\boldsymbol{\Sigma}}_{\text{DAG}}$ and $\hat{\boldsymbol{\Sigma}}_{\text{full}}$ be the covariance matrix estimators obtained by applying $g_{\Sigma}$ to the DAG-constrained and unconstrained parameter estimates, respectively. Then:
\begin{equation}
\Cov(\mathrm{vec}(\hat{\boldsymbol{\Sigma}}_{\text{DAG}})) \preceq \Cov(\mathrm{vec}(\hat{\boldsymbol{\Sigma}}_{\text{full}})).
\end{equation}
\end{corollary}

\begin{proof}
The key observation is that both estimators target the same true parameter. In the full parameter space, the true parameter is $\boldsymbol{\theta}^* = (\boldsymbol{\theta}_G^*, \mathbf{0})$, where $\boldsymbol{\theta}_0 = \mathbf{0}$ due to the DAG structure.

Both estimators map to the covariance matrix via the same function $g_\Sigma$:
\begin{itemize}
    \item DAG estimator: $\hat{\boldsymbol{\Sigma}}_{\text{DAG}} = g_\Sigma(\hat{\boldsymbol{\theta}}_G^{\text{DAG}}, \mathbf{0})$
    \item Full estimator: $\hat{\boldsymbol{\Sigma}}_{\text{full}} = g_\Sigma(\hat{\boldsymbol{\theta}}_G^{\text{full}}, \hat{\boldsymbol{\theta}}_0^{\text{full}})$
\end{itemize}

The delta method states that for an estimator $\hat{\boldsymbol{\theta}}$ of $\boldsymbol{\theta}^*$:
\[
\Cov(g(\hat{\boldsymbol{\theta}})) \approx \mathbf{J}_g(\boldsymbol{\theta}^*) \, \Cov(\hat{\boldsymbol{\theta}}) \, \mathbf{J}_g(\boldsymbol{\theta}^*)^\top,
\]
where the Jacobian $\mathbf{J}_g(\boldsymbol{\theta}^*)$ is evaluated at the \emph{true} parameter value.

Since both estimators have the same true parameter $\boldsymbol{\theta}^*$, the same Jacobian $\mathbf{J}_{g_\Sigma} = \mathbf{J}_{g_\Sigma}(\boldsymbol{\theta}^*)$ applies to both:
\begin{align}
\Cov(\mathrm{vec}(\hat{\boldsymbol{\Sigma}}_{\text{DAG}})) &\approx \mathbf{J}_{g_\Sigma} \boldsymbol{\Sigma}_{\text{DAG}} \mathbf{J}_{g_\Sigma}^\top, \\
\Cov(\mathrm{vec}(\hat{\boldsymbol{\Sigma}}_{\text{full}})) &\approx \mathbf{J}_{g_\Sigma} \boldsymbol{\Sigma}_{\text{full}} \mathbf{J}_{g_\Sigma}^\top,
\end{align}
where $\boldsymbol{\Sigma}_{\text{DAG}}$ and $\boldsymbol{\Sigma}_{\text{full}}$ are the parameter-space covariances from Corollary~\ref{cor:full_variance_reduction}.

Since $\boldsymbol{\Sigma}_{\text{DAG}} \preceq \boldsymbol{\Sigma}_{\text{full}}$, applying Lemma~\ref{lem:loewner_congruence} with $\mathbf{M} = \mathbf{J}_{g_\Sigma}$ yields:
\[
\mathbf{J}_{g_\Sigma} \boldsymbol{\Sigma}_{\text{DAG}} \mathbf{J}_{g_\Sigma}^\top \preceq \mathbf{J}_{g_\Sigma} \boldsymbol{\Sigma}_{\text{full}} \mathbf{J}_{g_\Sigma}^\top,
\]
which is the desired result.
\end{proof}

\subsubsection{Application to Regression Coefficients}

For downstream prediction, we typically care about the regression coefficients $\boldsymbol{\beta} = \boldsymbol{\Sigma}_{\mathbf{XX}}^{-1}\boldsymbol{\Sigma}_{\mathbf{X}Y}$. The map from covariance to regression coefficients:
\begin{equation}
h: \boldsymbol{\Sigma} \mapsto \boldsymbol{\beta} = \boldsymbol{\Sigma}_{\mathbf{XX}}^{-1}\boldsymbol{\Sigma}_{\mathbf{X}Y}
\end{equation}
is also smooth (rational function of covariance entries, well-defined when $\boldsymbol{\Sigma}_{\mathbf{XX}} \succ 0$).

\begin{corollary}[Variance Reduction for Regression Coefficients]
\label{cor:beta_variance_reduction}
Let $\hat{\boldsymbol{\beta}}_{\text{DAG}}$ and $\hat{\boldsymbol{\beta}}_{\text{full}}$ be the regression coefficient estimators derived from the DAG-constrained and unconstrained covariance estimates. Then:
\begin{equation}
\Cov(\hat{\boldsymbol{\beta}}_{\text{DAG}}) \preceq \Cov(\hat{\boldsymbol{\beta}}_{\text{full}}).
\end{equation}
\end{corollary}

\begin{proof}
The composition $h \circ g_\Sigma: \boldsymbol{\theta} \mapsto \boldsymbol{\beta}$ is smooth. By the chain rule, its Jacobian is $\mathbf{J}_{h \circ g_\Sigma} = \mathbf{J}_h \mathbf{J}_{g_\Sigma}$. Applying Lemma~\ref{lem:loewner_congruence} with $\mathbf{M} = \mathbf{J}_{h \circ g_\Sigma}$:
\[
\mathbf{J}_{h \circ g_\Sigma} \boldsymbol{\Sigma}_{\text{DAG}} \mathbf{J}_{h \circ g_\Sigma}^\top \preceq \mathbf{J}_{h \circ g_\Sigma} \boldsymbol{\Sigma}_{\text{full}} \mathbf{J}_{h \circ g_\Sigma}^\top.
\]
\end{proof}

\begin{remark}[Chain of Variance Reduction]
The key insight is that the Loewner ordering ``propagates'' through any smooth transformation. Starting from variance reduction in the LDL parameters, we obtain variance reduction in the covariance matrix, which in turn yields variance reduction in regression coefficients. This chain justifies why incorporating DAG knowledge improves downstream prediction tasks.
\end{remark}

\subsection{From Covariance Ordering to Prediction MSE Reduction}

We now complete the argument by showing that Loewner ordering of coefficient covariances implies reduced prediction error when using $\hat{\boldsymbol{\beta}}$ to predict $Y$ from $\mathbf{X}$.

\subsubsection{Prediction MSE Decomposition}

Consider predicting a new outcome $Y_{\text{new}}$ given features $\mathbf{X}_{\text{new}}$ using the estimated coefficients $\hat{\boldsymbol{\beta}}$. The true model is:
\begin{equation}
Y_{\text{new}} = \mathbf{X}_{\text{new}}^\top \boldsymbol{\beta}^* + \varepsilon_{\text{new}}, \quad \varepsilon_{\text{new}} \sim \mathcal{N}(0, \sigma^2),
\end{equation}
and the prediction is $\hat{Y}_{\text{new}} = \mathbf{X}_{\text{new}}^\top \hat{\boldsymbol{\beta}}$.

The prediction error is:
\begin{equation}
Y_{\text{new}} - \hat{Y}_{\text{new}} = \mathbf{X}_{\text{new}}^\top (\boldsymbol{\beta}^* - \hat{\boldsymbol{\beta}}) + +\varepsilon_{\text{new}}.
\end{equation}

Taking expectations (over the noise $\varepsilon_{\text{new}}$ and the randomness in $\hat{\boldsymbol{\beta}}$), conditioned on $\mathbf{X}_{\text{new}}$:
\begin{align}
\mathrm{MSE}_{\text{pred}}(\mathbf{X}_{\text{new}}) &= \mathbb{E}\left[(Y_{\text{new}} - \hat{Y}_{\text{new}})^2 \mid \mathbf{X}_{\text{new}}\right] \\
&= \mathbb{E}\left[(\mathbf{X}_{\text{new}}^\top (\boldsymbol{\beta}^* - \hat{\boldsymbol{\beta}}))^2\right] + \sigma^2 \\
&= \mathbf{X}_{\text{new}}^\top \mathbb{E}\left[(\hat{\boldsymbol{\beta}} - \boldsymbol{\beta}^*)(\hat{\boldsymbol{\beta}} - \boldsymbol{\beta}^*)^\top\right] \mathbf{X}_{\text{new}} + \sigma^2.
\end{align}

For an asymptotically unbiased estimator, $\mathbb{E}[(\hat{\boldsymbol{\beta}} - \boldsymbol{\beta}^*)(\hat{\boldsymbol{\beta}} - \boldsymbol{\beta}^*)^\top] \approx \Cov(\hat{\boldsymbol{\beta}})$, so:
\begin{equation}
\mathrm{MSE}_{\text{pred}}(\mathbf{X}_{\text{new}}) \approx \mathbf{X}_{\text{new}}^\top \Cov(\hat{\boldsymbol{\beta}}) \mathbf{X}_{\text{new}} + \sigma^2.
\end{equation}

\subsubsection{Prediction MSE for Fixed Test Point}

For a fixed test point $\mathbf{X}_{\text{new}}$, the prediction MSE directly inherits the Loewner ordering:

\begin{lemma}[Quadratic Form Respects Loewner Order]
\label{lem:quadratic_loewner}
If $\mathbf{A} \preceq \mathbf{B}$, then for any vector $\mathbf{x}$:
\begin{equation}
\mathbf{x}^\top \mathbf{A} \mathbf{x} \leq \mathbf{x}^\top \mathbf{B} \mathbf{x}.
\end{equation}
\end{lemma}

\begin{proof}
By definition, $\mathbf{A} \preceq \mathbf{B}$ means $\mathbf{x}^\top (\mathbf{B} - \mathbf{A}) \mathbf{x} \geq 0$ for all $\mathbf{x}$.
\end{proof}

\begin{corollary}[Prediction MSE Reduction for Fixed Test Point]
For any fixed test point $\mathbf{X}_{\text{new}}$:
\begin{equation}
\mathrm{MSE}_{\text{pred}}^{\text{DAG}}(\mathbf{X}_{\text{new}}) \leq \mathrm{MSE}_{\text{pred}}^{\text{full}}(\mathbf{X}_{\text{new}}).
\end{equation}
\end{corollary}

\begin{proof}
From Corollary~\ref{cor:beta_variance_reduction}, $\Cov(\hat{\boldsymbol{\beta}}_{\text{DAG}}) \preceq \Cov(\hat{\boldsymbol{\beta}}_{\text{full}})$. Applying Lemma~\ref{lem:quadratic_loewner} with $\mathbf{x} = \mathbf{X}_{\text{new}}$:
\[
\mathbf{X}_{\text{new}}^\top \Cov(\hat{\boldsymbol{\beta}}_{\text{DAG}}) \mathbf{X}_{\text{new}} \leq \mathbf{X}_{\text{new}}^\top \Cov(\hat{\boldsymbol{\beta}}_{\text{full}}) \mathbf{X}_{\text{new}}.
\]
Adding $\sigma^2$ to both sides gives the result.
\end{proof}

\subsubsection{Expected Prediction MSE over Test Distribution}

When $\mathbf{X}_{\text{new}}$ is drawn from a distribution with $\mathbb{E}[\mathbf{X}_{\text{new}} \mathbf{X}_{\text{new}}^\top] = \boldsymbol{\Sigma}_{\mathbf{X}}$, the expected prediction MSE is:
\begin{equation}
\mathbb{E}[\mathrm{MSE}_{\text{pred}}] = \mathbb{E}\left[\mathbf{X}_{\text{new}}^\top \Cov(\hat{\boldsymbol{\beta}}) \mathbf{X}_{\text{new}}\right] + \sigma^2 = \mathrm{tr}\left(\Cov(\hat{\boldsymbol{\beta}}) \boldsymbol{\Sigma}_{\mathbf{X}}\right) + \sigma^2.
\end{equation}

\begin{lemma}[Weighted Trace Respects Loewner Order]
\label{lem:weighted_trace_loewner}
If $\mathbf{A} \preceq \mathbf{B}$ and $\mathbf{W} \succeq 0$, then $\mathrm{tr}(\mathbf{W}\mathbf{A}) \leq \mathrm{tr}(\mathbf{W}\mathbf{B})$.
\end{lemma}

\begin{proof}
Let $\mathbf{W} = \sum_i \lambda_i \mathbf{w}_i \mathbf{w}_i^\top$ be the eigendecomposition with $\lambda_i \geq 0$. Then:
\[
\mathrm{tr}(\mathbf{W}(\mathbf{B} - \mathbf{A})) = \sum_i \lambda_i \mathbf{w}_i^\top (\mathbf{B} - \mathbf{A}) \mathbf{w}_i \geq 0,
\]
since each term is non-negative by $\mathbf{B} - \mathbf{A} \succeq 0$ and $\lambda_i \geq 0$.
\end{proof}

\subsubsection{Main Result: Prediction MSE Reduction}

\begin{theorem}[DAG Knowledge Reduces Prediction MSE]
\label{thm:prediction_mse_reduction}
Under Assumptions~\ref{assump:linear_gaussian_new}--\ref{assump:known_dag_new}, let $\hat{\boldsymbol{\beta}}_{\text{DAG}}$ and $\hat{\boldsymbol{\beta}}_{\text{full}}$ be the regression coefficient estimators from the DAG-constrained and unconstrained models. Then:
\begin{enumerate}
    \item For any fixed test point $\mathbf{X}_{\text{new}}$:
    \begin{equation}
    \mathrm{MSE}_{\text{pred}}^{\text{DAG}}(\mathbf{X}_{\text{new}}) \leq \mathrm{MSE}_{\text{pred}}^{\text{full}}(\mathbf{X}_{\text{new}}).
    \end{equation}
    \item For test points drawn from any distribution with second moment matrix $\boldsymbol{\Sigma}_{\mathbf{X}} \succeq 0$:
    \begin{equation}
    \mathbb{E}[\mathrm{MSE}_{\text{pred}}^{\text{DAG}}] \leq \mathbb{E}[\mathrm{MSE}_{\text{pred}}^{\text{full}}].
    \end{equation}
\end{enumerate}
\end{theorem}

\begin{proof}
Part 1 follows from Lemma~\ref{lem:quadratic_loewner} applied to $\Cov(\hat{\boldsymbol{\beta}}_{\text{DAG}}) \preceq \Cov(\hat{\boldsymbol{\beta}}_{\text{full}})$.

For Part 2:
\[
\mathbb{E}[\mathrm{MSE}_{\text{pred}}^{\text{DAG}}] = \mathrm{tr}(\Cov(\hat{\boldsymbol{\beta}}_{\text{DAG}}) \boldsymbol{\Sigma}_{\mathbf{X}}) + \sigma^2 \leq \mathrm{tr}(\Cov(\hat{\boldsymbol{\beta}}_{\text{full}}) \boldsymbol{\Sigma}_{\mathbf{X}}) + \sigma^2 = \mathbb{E}[\mathrm{MSE}_{\text{pred}}^{\text{full}}],
\]
where the inequality follows from Lemma~\ref{lem:weighted_trace_loewner} with $\mathbf{W} = \boldsymbol{\Sigma}_{\mathbf{X}}$.
\end{proof}

\begin{remark}[Interpretation]
This theorem provides the complete theoretical justification for Causal-Residual Bootstrapping: incorporating correct DAG knowledge into the estimation procedure provably reduces prediction error for \emph{any} test point and \emph{any} test distribution. The improvement stems from reduced uncertainty in the coefficient estimates, which directly translates to more accurate predictions.
\end{remark}

\subsection{Quantifying the MSE Improvement}

We now characterize the magnitude of the MSE improvement as a function of the sample size $N$.

\subsubsection{Scaling of Asymptotic Covariance}

Recall that the MLE has asymptotic covariance that scales as $1/N$:
\begin{equation}
\Cov(\hat{\boldsymbol{\theta}}) \approx \frac{1}{N} \mathcal{I}(\boldsymbol{\theta}^*)^{-1}.
\end{equation}

In the full parameter space, writing the Fisher information inverse as before:
\begin{align}
\Cov(\hat{\boldsymbol{\theta}}_{\text{full}}) &= \frac{1}{N} \mathcal{I}_{\text{full}}^{-1}, \\
\Cov(\hat{\boldsymbol{\theta}}_{\text{DAG}}) &= \frac{1}{N} \begin{pmatrix} \mathcal{I}_{GG}^{-1} & \mathbf{0} \\ \mathbf{0} & \mathbf{0} \end{pmatrix}.
\end{align}

Define the \emph{information gain matrix} as the difference (scaled by $N$):
\begin{equation}
\boldsymbol{\Delta}_{\mathcal{I}} := \mathcal{I}_{\text{full}}^{-1} - \begin{pmatrix} \mathcal{I}_{GG}^{-1} & \mathbf{0} \\ \mathbf{0} & \mathbf{0} \end{pmatrix} \succeq 0.
\end{equation}

This matrix captures the ``information lost'' by not exploiting the DAG constraints, and is independent of $N$.

\subsubsection{MSE Difference Scaling}

The covariance difference in the parameter space is:
\begin{equation}
\Cov(\hat{\boldsymbol{\theta}}_{\text{full}}) - \Cov(\hat{\boldsymbol{\theta}}_{\text{DAG}}) = \frac{1}{N} \boldsymbol{\Delta}_{\mathcal{I}}.
\end{equation}

This $1/N$ scaling propagates through the smooth transformations to the regression coefficients. Let $\mathbf{J}$ denote the Jacobian of the map from parameters to regression coefficients. Then:
\begin{equation}
\Cov(\hat{\boldsymbol{\beta}}_{\text{full}}) - \Cov(\hat{\boldsymbol{\beta}}_{\text{DAG}}) \approx \frac{1}{N} \mathbf{J} \boldsymbol{\Delta}_{\mathcal{I}} \mathbf{J}^\top =: \frac{1}{N} \boldsymbol{\Delta}_{\boldsymbol{\beta}},
\end{equation}
where $\boldsymbol{\Delta}_{\boldsymbol{\beta}} = \mathbf{J} \boldsymbol{\Delta}_{\mathcal{I}} \mathbf{J}^\top \succeq 0$ is independent of $N$.

\subsubsection{Explicit MSE Improvement Formula}

\begin{theorem}[MSE Improvement Rate]
\label{thm:mse_improvement_rate}
The expected prediction MSE improvement from using DAG constraints is:
\begin{equation}
\mathbb{E}[\mathrm{MSE}_{\text{pred}}^{\text{full}}] - \mathbb{E}[\mathrm{MSE}_{\text{pred}}^{\text{DAG}}] = \frac{C}{N},
\end{equation}
where the constant $C > 0$ (when the DAG provides non-trivial constraints) is given by:
\begin{equation}
C = \mathrm{tr}(\boldsymbol{\Delta}_{\boldsymbol{\beta}} \boldsymbol{\Sigma}_{\mathbf{X}}) = \mathrm{tr}(\mathbf{J} \boldsymbol{\Delta}_{\mathcal{I}} \mathbf{J}^\top \boldsymbol{\Sigma}_{\mathbf{X}}).
\end{equation}
\end{theorem}

\begin{proof}
From the prediction MSE formula:
\begin{align}
\mathbb{E}[\mathrm{MSE}_{\text{pred}}^{\text{full}}] - \mathbb{E}[\mathrm{MSE}_{\text{pred}}^{\text{DAG}}] &= \mathrm{tr}\left((\Cov(\hat{\boldsymbol{\beta}}_{\text{full}}) - \Cov(\hat{\boldsymbol{\beta}}_{\text{DAG}})) \boldsymbol{\Sigma}_{\mathbf{X}}\right) \\
&= \mathrm{tr}\left(\frac{1}{N} \boldsymbol{\Delta}_{\boldsymbol{\beta}} \boldsymbol{\Sigma}_{\mathbf{X}}\right) \\
&= \frac{1}{N} \mathrm{tr}(\boldsymbol{\Delta}_{\boldsymbol{\beta}} \boldsymbol{\Sigma}_{\mathbf{X}}).
\end{align}
\end{proof}

\begin{remark}[Interpretation of the Constant $C$]
The constant $C$ depends on:
\begin{itemize}
    \item $\boldsymbol{\Delta}_{\mathcal{I}}$: The information gain from the DAG constraints. This is larger when the DAG excludes more edges (more zeros in $\mathbf{L}$) and when the cross-information $\mathcal{I}_{G0}$ between free and constrained parameters is larger.
    \item $\mathbf{J}$: The sensitivity of regression coefficients to parameter changes. This depends on the true covariance structure.
    \item $\boldsymbol{\Sigma}_{\mathbf{X}}$: The test distribution. The improvement is larger when test points lie in directions where the coefficient uncertainty is most reduced.
\end{itemize}
\end{remark}

\begin{remark}[Practical Implications]
The $1/N$ scaling has important practical implications:
\begin{itemize}
    \item \textbf{Small samples}: The relative improvement $\frac{\mathbb{E}[\mathrm{MSE}^{\text{full}}] - \mathbb{E}[\mathrm{MSE}^{\text{DAG}}]}{\mathbb{E}[\mathrm{MSE}^{\text{full}}]}$ is largest when $N$ is small, precisely when incorporating prior knowledge is most valuable.
    \item \textbf{Large samples}: As $N \to \infty$, both estimators converge to the true coefficients, and the absolute improvement vanishes. However, the DAG-constrained estimator is never worse.
    \item \textbf{Sparse DAGs}: When the true DAG is sparse (few edges), the constant $C$ is larger because more parameters are constrained to zero, leading to greater MSE reduction.
\end{itemize}
\end{remark}
\subsubsection{Role of the Markov Boundary}

We now show that only DAG constraints involving the Markov boundary of $Y$ contribute to prediction improvement. Constraints involving variables outside the Markov boundary provide no additional benefit.

\begin{definition}[Markov Boundary]
\label{def:markov_boundary}
The \emph{Markov boundary} of $Y$ with respect to $\mathbf{X}$, denoted $\mathrm{MB}(Y)$, is the minimal subset $\mathbf{X}_{\mathrm{MB}} \subseteq \mathbf{X}$ such that:
\begin{equation}
Y \perp\!\!\!\perp \mathbf{X}_{-\mathrm{MB}} \mid \mathbf{X}_{\mathrm{MB}},
\end{equation}
where $\mathbf{X}_{-\mathrm{MB}} = \mathbf{X} \setminus \mathbf{X}_{\mathrm{MB}}$ denotes the variables outside the Markov boundary.
\end{definition}

\begin{remark}[Markov Boundary in DAGs]
In a DAG, the Markov boundary of $Y$ consists of its parents, children, and parents of children: $\mathrm{MB}(Y) = \mathrm{Pa}(Y) \cup \mathrm{Ch}(Y) \cup \mathrm{Pa}(\mathrm{Ch}(Y))$. When $Y$ is a sink node (no children), we have $\mathrm{MB}(Y) = \mathrm{Pa}(Y)$.
\end{remark}

The following lemma establishes that variables outside the Markov boundary have zero contribution to the optimal linear predictor.

\begin{lemma}[Zero Coefficients Outside Markov Boundary]
\label{lem:zero_coeff_outside_mb}
Under Assumption~\ref{assump:linear_gaussian_new}, let $\boldsymbol{\beta}^* = \boldsymbol{\Sigma}_{\mathbf{XX}}^{-1}\boldsymbol{\Sigma}_{\mathbf{X}Y}$ be the population regression coefficients. Partition $\mathbf{X} = (\mathbf{X}_{\mathrm{MB}}, \mathbf{X}_{-\mathrm{MB}})$ and correspondingly $\boldsymbol{\beta}^* = (\boldsymbol{\beta}^*_{\mathrm{MB}}, \boldsymbol{\beta}^*_{-\mathrm{MB}})$. Then:
\begin{equation}
\boldsymbol{\beta}^*_{-\mathrm{MB}} = \mathbf{0}.
\end{equation}
\end{lemma}

\begin{proof}
By the definition of Markov boundary, $Y \perp\!\!\!\perp \mathbf{X}_{-\mathrm{MB}} \mid \mathbf{X}_{\mathrm{MB}}$. In a Gaussian model, conditional independence implies zero partial covariance:
\begin{equation}
\Cov(Y, \mathbf{X}_{-\mathrm{MB}} \mid \mathbf{X}_{\mathrm{MB}}) = \boldsymbol{\Sigma}_{Y, -\mathrm{MB}} - \boldsymbol{\Sigma}_{Y, \mathrm{MB}} \boldsymbol{\Sigma}_{\mathrm{MB}, \mathrm{MB}}^{-1} \boldsymbol{\Sigma}_{\mathrm{MB}, -\mathrm{MB}} = \mathbf{0}.
\end{equation}
The regression coefficient for $\mathbf{X}_{-\mathrm{MB}}$ in the full regression $Y \sim \mathbf{X}$ equals the coefficient in the partial regression of $(Y \mid \mathbf{X}_{\mathrm{MB}})$ on $(\mathbf{X}_{-\mathrm{MB}} \mid \mathbf{X}_{\mathrm{MB}})$. Since the partial covariance is zero, these coefficients vanish: $\boldsymbol{\beta}^*_{-\mathrm{MB}} = \mathbf{0}$.
\end{proof}

We now establish that constraints on variables outside the Markov boundary do not improve prediction.

\begin{theorem}[Irrelevance of Non-Markov-Boundary Constraints]
\label{thm:mb_irrelevance}
Partition the DAG constraints into two sets:
\begin{itemize}
    \item $\mathcal{C}_{\mathrm{MB}}$: constraints involving at least one variable in $\mathrm{MB}(Y) \cup \{Y\}$,
    \item $\mathcal{C}_{-\mathrm{MB}}$: constraints involving only variables in $\mathbf{X}_{-\mathrm{MB}}$.
\end{itemize}
Let $\hat{\boldsymbol{\beta}}_{\mathcal{C}_{\mathrm{MB}}}$ denote the estimator using only constraints $\mathcal{C}_{\mathrm{MB}}$, and let $\hat{\boldsymbol{\beta}}_{\mathcal{C}_{\mathrm{MB}} \cup \mathcal{C}_{-\mathrm{MB}}}$ denote the estimator using all constraints. Then the expected prediction MSE is identical:
\begin{equation}
\mathbb{E}[\mathrm{MSE}_{\mathrm{pred}}^{\mathcal{C}_{\mathrm{MB}} \cup \mathcal{C}_{-\mathrm{MB}}}] = \mathbb{E}[\mathrm{MSE}_{\mathrm{pred}}^{\mathcal{C}_{\mathrm{MB}}}].
\end{equation}
\end{theorem}

\begin{proof}
The proof proceeds in two steps: first we show that the optimal DAG-aware estimator only involves $\mathbf{X}_{\mathrm{MB}}$, then we show that constraints among non-MB variables cannot affect this estimator.

\textbf{Step 1: Reduction to Markov Boundary Regression.}
By Lemma~\ref{lem:zero_coeff_outside_mb}, $\boldsymbol{\beta}^*_{-\mathrm{MB}} = \mathbf{0}$. An efficient DAG-constrained estimator exploits this by:
\begin{itemize}
    \item Setting $\hat{\boldsymbol{\beta}}_{-\mathrm{MB}} = \mathbf{0}$ (the known true value), and
    \item Estimating $\hat{\boldsymbol{\beta}}_{\mathrm{MB}}$ by regressing $Y$ on $\mathbf{X}_{\mathrm{MB}}$ alone.
\end{itemize}

Since $\hat{\boldsymbol{\beta}}_{-\mathrm{MB}}$ is fixed (not estimated), we have:
\begin{equation}
\Cov(\hat{\boldsymbol{\beta}}_{-\mathrm{MB}}) = \mathbf{0}, \qquad \Cov(\hat{\boldsymbol{\beta}}_{\mathrm{MB}}, \hat{\boldsymbol{\beta}}_{-\mathrm{MB}}) = \mathbf{0}.
\end{equation}

\textbf{Step 2: Simplification of the Trace Formula.}
Partition the true covariance matrix as:
\begin{equation}
\boldsymbol{\Sigma}_{\mathbf{X}} = \begin{pmatrix} \boldsymbol{\Sigma}_{\mathrm{MB}} & \boldsymbol{\Sigma}_{\mathrm{MB}, -\mathrm{MB}} \\ \boldsymbol{\Sigma}_{-\mathrm{MB}, \mathrm{MB}} & \boldsymbol{\Sigma}_{-\mathrm{MB}} \end{pmatrix}.
\end{equation}

The expected prediction MSE becomes:
\begin{align}
\mathbb{E}[\mathrm{MSE}_{\mathrm{pred}}] &= \mathrm{tr}\left(\Cov(\hat{\boldsymbol{\beta}}) \boldsymbol{\Sigma}_{\mathbf{X}}\right) + \sigma^2 \nonumber \\
&= \mathrm{tr}\left(\begin{pmatrix} \Cov(\hat{\boldsymbol{\beta}}_{\mathrm{MB}}) & \mathbf{0} \\ \mathbf{0} & \mathbf{0} \end{pmatrix} \begin{pmatrix} \boldsymbol{\Sigma}_{\mathrm{MB}} & \boldsymbol{\Sigma}_{\mathrm{MB}, -\mathrm{MB}} \\ \boldsymbol{\Sigma}_{-\mathrm{MB}, \mathrm{MB}} & \boldsymbol{\Sigma}_{-\mathrm{MB}} \end{pmatrix}\right) + \sigma^2 \nonumber \\
&= \mathrm{tr}\left(\Cov(\hat{\boldsymbol{\beta}}_{\mathrm{MB}}) \boldsymbol{\Sigma}_{\mathrm{MB}}\right) + \sigma^2.
\end{align}

\textbf{Step 3: Irrelevance of Non-MB Constraints.}
The estimator $\hat{\boldsymbol{\beta}}_{\mathrm{MB}}$ is obtained by regressing $Y$ on $\mathbf{X}_{\mathrm{MB}}$ only. Its covariance is:
\begin{equation}
\Cov(\hat{\boldsymbol{\beta}}_{\mathrm{MB}}) = \sigma^2 \left(\mathbf{X}_{\mathrm{MB}}^\top \mathbf{X}_{\mathrm{MB}}\right)^{-1} \xrightarrow{p} \frac{\sigma^2}{N} \boldsymbol{\Sigma}_{\mathrm{MB}}^{-1}.
\end{equation}

This covariance depends only on the distribution of $\mathbf{X}_{\mathrm{MB}}$ and the noise variance $\sigma^2$. The constraints $\mathcal{C}_{-\mathrm{MB}}$, which involve only relationships among variables in $\mathbf{X}_{-\mathrm{MB}}$, do not affect the estimation of $\hat{\boldsymbol{\beta}}_{\mathrm{MB}}$ nor its covariance.

Therefore, adding constraints $\mathcal{C}_{-\mathrm{MB}}$ to $\mathcal{C}_{\mathrm{MB}}$ leaves the expected prediction MSE unchanged.
\end{proof}

\begin{remark}[Practical Implication]
This result has important practical implications: when the goal is to predict $Y$, one need only incorporate DAG constraints that involve the Markov boundary of $Y$. Constraints among variables that are conditionally independent of $Y$ given its Markov boundary---while valid structural knowledge---provide no benefit for prediction. This justifies focusing computational and statistical effort on learning and enforcing constraints relevant to the target variable.
\end{remark}

\section{Per variable results}
\label{app:per_variable}

\begin{table}[htbp]
\centering
\caption{Per-variable MSE by Augmenter (Known Graph) - Bold indicates best or statistically tied}
\label{tab:known_graph_per_var_mse}
\footnotesize

\begin{tabular}{lcccccccccc}
\toprule
Augmenter & \rotatebox{90}{angle\_1} & \rotatebox{90}{angle\_2} & \rotatebox{90}{blue} & \rotatebox{90}{current} & \rotatebox{90}{green} & \rotatebox{90}{ir\_1} & \rotatebox{90}{ir\_2} & \rotatebox{90}{ir\_3} & \rotatebox{90}{l\_11} & \rotatebox{90}{l\_12} \\
\midrule
ADMGTian & 0.007 & \textbf{0.008} & 0.193 & 0.563 & 0.503 & 0.134 & 0.134 & 0.151 & 1.209 & 1.227 \\
ARF      & 0.895 & 0.613 & 0.555 & 0.530 & 0.794 & 0.306 & 0.291 & 0.300 & \textbf{1.026} & \textbf{1.021} \\
CTGAN    & 0.420 & 0.415 & 0.451 & 0.596 & 0.862 & 0.299 & 0.290 & 0.282 & 1.134 & 1.139 \\
CRB     & \textbf{0.003} & \textbf{0.007} & \textbf{0.064} & \textbf{0.489} & \textbf{0.184} & \textbf{0.100} & \textbf{0.106} & \textbf{0.129} & \textbf{1.023} & \textbf{1.020} \\
DDPM     & 0.019 & \textbf{0.011} & 0.120 & 0.559 & 0.412 & \textbf{0.106} & \textbf{0.101} & \textbf{0.116} & 1.295 & 1.279 \\
None     & 0.006 & \textbf{0.007} & 0.127 & \textbf{0.501} & 0.410 & 0.122 & \textbf{0.108} & \textbf{0.123} & \textbf{1.051} & 1.139 \\
TVAE     & 0.138 & 0.158 & 0.249 & 0.525 & 0.631 & 0.150 & 0.142 & 0.154 & 1.171 & 1.198 \\
\bottomrule
\end{tabular}

\vspace{1em}

\begin{tabular}{lcccccccccc}
\toprule
Augmenter & \rotatebox{90}{l\_21} & \rotatebox{90}{l\_22} & \rotatebox{90}{l\_31} & \rotatebox{90}{l\_32} & \rotatebox{90}{pol\_1} & \rotatebox{90}{pol\_2} & \rotatebox{90}{red} & \rotatebox{90}{vis\_1} & \rotatebox{90}{vis\_2} & \rotatebox{90}{vis\_3} \\
\midrule
ADMGTian & 1.212 & 1.256 & 1.209 & 1.245 & 0.007 & 0.014 & 0.245 & 0.122 & 0.129 & 0.148 \\
ARF      & \textbf{1.032} & \textbf{1.039} & \textbf{1.053} & \textbf{1.033} & 0.856 & 0.694 & 0.663 & 0.261 & 0.316 & 0.317 \\
CTGAN    & 1.092 & 1.134 & 1.109 & 1.122 & 0.428 & 0.426 & 0.511 & 0.257 & 0.232 & 0.324 \\
CRB     & \textbf{1.019} & \textbf{1.041} & \textbf{1.036} & \textbf{1.021} & \textbf{0.003} & \textbf{0.005} & \textbf{0.077} & \textbf{0.096} & \textbf{0.102} & \textbf{0.125} \\
DDPM     & 1.275 & 1.319 & 1.285 & 1.289 & 0.016 & 0.012 & 0.145 & \textbf{0.096} & \textbf{0.103} & \textbf{0.112} \\
None     & \textbf{1.043} & \textbf{1.130} & \textbf{1.082} & 1.090 & 0.005 & \textbf{0.007} & 0.161 & \textbf{0.100} & \textbf{0.108} & \textbf{0.121} \\
TVAE     & 1.181 & 1.209 & 1.172 & 1.180 & 0.129 & 0.147 & 0.255 & 0.133 & 0.146 & 0.150 \\
\bottomrule
\end{tabular}

\end{table}

\section{Full Performance Tables for Learned Graph Experiments}
\label{app:full_tables}

Tables~\ref{tab:boston_100_samples}--\ref{tab:uniform_reference_2000_samples} present the complete performance metrics for all augmentation methods across the benchmark datasets and density metrics used for tabular data generators: $\alpha$-Precision and $\beta$-Recall~\citep{alaa2022faithful}. $\alpha$-Precision assesses the similarity of the distribution to the reference one, and $\beta$ -Recall quantifies the diversity of the points. We also used privacy metrics: DCR~\citep{zhao2021ctab}, which assesses how likely the data is to be copied from the training set, and $\delta$-Presence~\citep{qian2023synthcity}.

In the case of the Causal Chambers dataset, i.e., the dataset with an underlying causal structure, CRB methods consistently outperform all other methods on MSE and population metrics. On the Sachs dataset, CRB and DDPM perform on par with respect to the MSE metric; as the size grows, DDPM becomes the best-performing method on distributional metrics. For white wine, initially CRB performs better than DDPM (100 and 500 samples), but for the larger samples, DDPM is better. For red wine, DDPM is again the best-performing model on MSE and density metrics, similar to the Boston dataset.

\begin{table}[htbp]
\centering
\caption{Performance metrics for Boston dataset (100 samples) - Bold indicates best or statistically tied}
\label{tab:boston_100_samples}
\footnotesize
\begin{tabular}{lccccc}
\toprule
Augmenter & Mean MSE $\downarrow$ & $\alpha$-Precision $\uparrow$ & $\beta$-Recall $\uparrow$ & DCR $\uparrow$ & $\delta$-Presence $\uparrow$ \\
\midrule
ADMGTian & 1.886 {\scriptsize (1.43, 2.12)} & 0.327 {\scriptsize (0.24, 0.48)} & 0.014 {\scriptsize (0.01, 0.03)} & 0.126 {\scriptsize (0.10, 0.15)} & 0.999 {\scriptsize (1.00, 1.00)} \\
ARF & 0.625 {\scriptsize (0.59, 0.70)} & 0.855 {\scriptsize (0.84, 0.87)} & 0.045 {\scriptsize (0.03, 0.06)} & \textbf{0.394} {\scriptsize (0.37, 0.42)} & 1.000 {\scriptsize } \\
CTGAN & 0.612 {\scriptsize (0.58, 0.66)} & 0.873 {\scriptsize (0.84, 0.90)} & 0.065 {\scriptsize (0.04, 0.08)} & 0.287 {\scriptsize (0.26, 0.33)} & 1.000 {\scriptsize } \\
CRB & 0.457 {\scriptsize (0.45, 0.47)} & 0.724 {\scriptsize (0.70, 0.78)} & 0.114 {\scriptsize (0.10, 0.13)} & 0.249 {\scriptsize (0.24, 0.26)} & 1.000 {\scriptsize } \\
DDPM & \textbf{0.371} {\scriptsize (0.36, 0.39)} & 0.886 {\scriptsize (0.86, 0.91)} & \textbf{0.304} {\scriptsize (0.29, 0.32)} & 0.119 {\scriptsize (0.11, 0.13)} & 0.996 {\scriptsize (0.99, 1.00)} \\
NFLOW & 0.489 {\scriptsize (0.48, 0.51)} & 0.833 {\scriptsize (0.77, 0.87)} & 0.097 {\scriptsize (0.08, 0.12)} & 0.285 {\scriptsize (0.25, 0.33)} & 1.000 {\scriptsize } \\
None & \textbf{0.364} {\scriptsize (0.35, 0.38)} & \textbf{0.943} {\scriptsize (0.93, 0.95)} & 0.194 {\scriptsize (0.19, 0.20)} & 0.133 {\scriptsize (0.13, 0.14)} & 0.979 {\scriptsize (0.97, 0.98)} \\
TVAE & 0.531 {\scriptsize (0.50, 0.57)} & 0.757 {\scriptsize (0.70, 0.81)} & 0.137 {\scriptsize (0.12, 0.16)} & 0.201 {\scriptsize (0.18, 0.22)} & 1.000 {\scriptsize } \\
\bottomrule
\end{tabular}
\end{table}

\begin{table}[htbp]
\centering
\caption{Performance metrics for Sachs dataset (100 samples) - Bold indicates best or statistically tied}
\label{tab:sachs_100_samples}
\footnotesize
\begin{tabular}{lccccc}
\toprule
Augmenter & Mean MSE $\downarrow$ & $\alpha$-Precision $\uparrow$ & $\beta$-Recall $\uparrow$ & DCR $\uparrow$ & $\delta$-Presence $\uparrow$ \\
\midrule
ADMGTian & \textbf{0.596} {\scriptsize (0.54, 0.67)} & \textbf{0.925} {\scriptsize (0.90, 0.94)} & 0.010 {\scriptsize (0.01, 0.01)} & 0.020 {\scriptsize (0.02, 0.03)} & 0.081 {\scriptsize (0.06, 0.10)} \\
ARF & 0.830 {\scriptsize (0.81, 0.85)} & 0.577 {\scriptsize (0.52, 0.62)} & 0.010 {\scriptsize (0.01, 0.01)} & \textbf{0.050} {\scriptsize (0.04, 0.06)} & \textbf{0.515} {\scriptsize (0.46, 0.57)} \\
CTGAN & 0.876 {\scriptsize (0.81, 0.95)} & \textbf{0.870} {\scriptsize (0.81, 0.91)} & 0.038 {\scriptsize (0.03, 0.05)} & 0.016 {\scriptsize (0.01, 0.02)} & 0.153 {\scriptsize (0.10, 0.19)} \\
CRB & \textbf{0.766} {\scriptsize (0.55, 1.58)} & 0.863 {\scriptsize (0.80, 0.90)} & 0.027 {\scriptsize (0.02, 0.03)} & 0.020 {\scriptsize (0.02, 0.02)} & 0.153 {\scriptsize (0.12, 0.20)} \\
DDPM & \textbf{0.545} {\scriptsize (0.48, 0.72)} & 0.863 {\scriptsize (0.81, 0.90)} & \textbf{0.057} {\scriptsize (0.05, 0.06)} & 0.013 {\scriptsize (0.01, 0.01)} & 0.244 {\scriptsize (0.21, 0.30)} \\
NFLOW & 0.783 {\scriptsize (0.72, 0.85)} & 0.816 {\scriptsize (0.76, 0.86)} & 0.028 {\scriptsize (0.02, 0.04)} & 0.022 {\scriptsize (0.02, 0.02)} & 0.260 {\scriptsize (0.22, 0.30)} \\
None & 0.998 {\scriptsize (0.99, 1.00)} & --- & --- & --- & --- \\
TVAE & \textbf{0.795} {\scriptsize (0.71, 0.86)} & 0.861 {\scriptsize (0.79, 0.89)} & \textbf{0.055} {\scriptsize (0.05, 0.06)} & 0.014 {\scriptsize (0.01, 0.02)} & 0.146 {\scriptsize (0.12, 0.17)} \\
\bottomrule
\end{tabular}
\end{table}

\begin{table}[htbp]
\centering
\caption{Performance metrics for Wine (red) dataset (100 samples) - Bold indicates best or statistically tied}
\label{tab:wine_red_100_samples}
\footnotesize
\begin{tabular}{lccccc}
\toprule
Augmenter & Mean MSE $\downarrow$ & $\alpha$-Precision $\uparrow$ & $\beta$-Recall $\uparrow$ & DCR $\uparrow$ & $\delta$-Presence $\uparrow$ \\
\midrule
ADMGTian & 0.831 {\scriptsize (0.79, 0.86)} & 0.780 {\scriptsize (0.72, 0.82)} & 0.020 {\scriptsize (0.02, 0.03)} & 0.211 {\scriptsize (0.20, 0.23)} & 0.643 {\scriptsize (0.58, 0.72)} \\
ARF & 0.981 {\scriptsize (0.92, 1.02)} & 0.814 {\scriptsize (0.80, 0.83)} & 0.027 {\scriptsize (0.02, 0.03)} & \textbf{0.303} {\scriptsize (0.29, 0.31)} & 1.000 {\scriptsize } \\
CTGAN & 0.769 {\scriptsize (0.73, 0.80)} & 0.826 {\scriptsize (0.70, 0.88)} & 0.104 {\scriptsize (0.08, 0.12)} & 0.217 {\scriptsize (0.20, 0.26)} & 1.000 {\scriptsize (1.00, 1.00)} \\
CRB & 0.635 {\scriptsize (0.62, 0.65)} & 0.876 {\scriptsize (0.85, 0.91)} & \textbf{0.141} {\scriptsize (0.13, 0.15)} & 0.205 {\scriptsize (0.20, 0.21)} & 1.000 {\scriptsize } \\
DDPM & \textbf{0.586} {\scriptsize (0.58, 0.60)} & 0.846 {\scriptsize (0.82, 0.87)} & \textbf{0.136} {\scriptsize (0.13, 0.15)} & 0.160 {\scriptsize (0.16, 0.16)} & 0.907 {\scriptsize (0.89, 0.92)} \\
NFLOW & 0.668 {\scriptsize (0.63, 0.72)} & 0.789 {\scriptsize (0.68, 0.87)} & 0.096 {\scriptsize (0.07, 0.11)} & 0.244 {\scriptsize (0.23, 0.28)} & 1.000 {\scriptsize (1.00, 1.00)} \\
None & \textbf{0.587} {\scriptsize (0.58, 0.59)} & \textbf{0.940} {\scriptsize (0.93, 0.95)} & 0.068 {\scriptsize (0.06, 0.07)} & 0.208 {\scriptsize (0.20, 0.21)} & 0.667 {\scriptsize (0.64, 0.69)} \\
TVAE & 0.764 {\scriptsize (0.73, 0.80)} & 0.620 {\scriptsize (0.56, 0.67)} & 0.123 {\scriptsize (0.11, 0.13)} & 0.190 {\scriptsize (0.17, 0.21)} & 1.000 {\scriptsize (1.00, 1.00)} \\
\bottomrule
\end{tabular}
\end{table}

\begin{table}[htbp]
\centering
\caption{Performance metrics for Wine (white) dataset (100 samples) - Bold indicates best or statistically tied}
\label{tab:wine_white_100_samples}
\footnotesize
\begin{tabular}{lccccc}
\toprule
Augmenter & Mean MSE $\downarrow$ & $\alpha$-Precision $\uparrow$ & $\beta$-Recall $\uparrow$ & DCR $\uparrow$ & $\delta$-Presence $\uparrow$ \\
\midrule
ADMGTian & 0.874 {\scriptsize (0.83, 0.93)} & 0.831 {\scriptsize (0.76, 0.88)} & 0.009 {\scriptsize (0.01, 0.01)} & \textbf{0.175} {\scriptsize (0.16, 0.21)} & 0.627 {\scriptsize (0.57, 0.69)} \\
ARF & 1.034 {\scriptsize (1.03, 1.04)} & 0.880 {\scriptsize (0.75, 0.93)} & 0.014 {\scriptsize (0.01, 0.02)} & \textbf{0.211} {\scriptsize (0.20, 0.24)} & 1.000 {\scriptsize } \\
CTGAN & 0.912 {\scriptsize (0.86, 1.03)} & 0.754 {\scriptsize (0.67, 0.84)} & 0.038 {\scriptsize (0.03, 0.05)} & \textbf{0.169} {\scriptsize (0.15, 0.20)} & 1.000 {\scriptsize (1.00, 1.00)} \\
CRB & \textbf{0.689} {\scriptsize (0.67, 0.71)} & 0.878 {\scriptsize (0.86, 0.90)} & \textbf{0.064} {\scriptsize (0.06, 0.07)} & 0.147 {\scriptsize (0.14, 0.16)} & 1.000 {\scriptsize } \\
DDPM & 0.757 {\scriptsize (0.73, 0.78)} & 0.827 {\scriptsize (0.79, 0.86)} & 0.047 {\scriptsize (0.04, 0.05)} & 0.130 {\scriptsize (0.12, 0.14)} & 0.884 {\scriptsize (0.84, 0.93)} \\
NFLOW & 0.799 {\scriptsize (0.77, 0.84)} & 0.848 {\scriptsize (0.73, 0.92)} & 0.037 {\scriptsize (0.03, 0.04)} & \textbf{0.180} {\scriptsize (0.17, 0.20)} & 1.000 {\scriptsize } \\
None & \textbf{0.679} {\scriptsize (0.67, 0.68)} & \textbf{0.946} {\scriptsize (0.94, 0.95)} & 0.023 {\scriptsize (0.02, 0.02)} & 0.185 {\scriptsize (0.18, 0.19)} & 0.585 {\scriptsize (0.56, 0.60)} \\
TVAE & 0.837 {\scriptsize (0.80, 0.90)} & 0.610 {\scriptsize (0.53, 0.71)} & 0.049 {\scriptsize (0.03, 0.06)} & 0.147 {\scriptsize (0.13, 0.18)} & 0.998 {\scriptsize (0.99, 1.00)} \\
\bottomrule
\end{tabular}
\end{table}

\begin{table}[htbp]
\centering
\caption{Performance metrics for Causal Chambers dataset (100 samples) - Bold indicates best or statistically tied}
\label{tab:uniform_reference_100_samples}
\footnotesize
\begin{tabular}{lccccc}
\toprule
Augmenter & Mean MSE $\downarrow$ & $\alpha$-Precision $\uparrow$ & $\beta$-Recall $\uparrow$ & DCR $\uparrow$ & $\delta$-Presence $\uparrow$ \\
\midrule
ADMGTian & 0.474 {\scriptsize (0.47, 0.48)} & \textbf{0.925} {\scriptsize (0.90, 0.94)} & 0.009 {\scriptsize (0.01, 0.01)} & 0.606 {\scriptsize (0.60, 0.62)} & 1.000 {\scriptsize } \\
ARF & 0.738 {\scriptsize (0.72, 0.76)} & \textbf{0.937} {\scriptsize (0.90, 0.96)} & 0.002 {\scriptsize (0.00, 0.00)} & \textbf{0.779} {\scriptsize (0.77, 0.79)} & 1.000 {\scriptsize } \\
CTGAN & 0.579 {\scriptsize (0.55, 0.61)} & \textbf{0.914} {\scriptsize (0.89, 0.93)} & 0.018 {\scriptsize (0.01, 0.02)} & 0.670 {\scriptsize (0.65, 0.69)} & 1.000 {\scriptsize } \\
CRB & \textbf{0.383} {\scriptsize (0.38, 0.38)} & \textbf{0.944} {\scriptsize (0.91, 0.96)} & \textbf{0.090} {\scriptsize (0.09, 0.09)} & 0.573 {\scriptsize (0.57, 0.58)} & 1.000 {\scriptsize } \\
DDPM & 0.483 {\scriptsize (0.48, 0.49)} & 0.758 {\scriptsize (0.74, 0.78)} & 0.037 {\scriptsize (0.03, 0.04)} & 0.606 {\scriptsize (0.60, 0.61)} & 1.000 {\scriptsize } \\
NFLOW & 0.537 {\scriptsize (0.50, 0.60)} & 0.615 {\scriptsize (0.46, 0.75)} & 0.011 {\scriptsize (0.01, 0.02)} & \textbf{0.751} {\scriptsize (0.71, 0.82)} & 1.000 {\scriptsize } \\
None & 0.422 {\scriptsize (0.42, 0.43)} & \textbf{0.938} {\scriptsize (0.92, 0.95)} & 0.010 {\scriptsize (0.01, 0.01)} & 0.634 {\scriptsize (0.63, 0.64)} & 1.000 {\scriptsize } \\
TVAE & 0.543 {\scriptsize (0.53, 0.55)} & 0.575 {\scriptsize (0.50, 0.63)} & 0.058 {\scriptsize (0.05, 0.06)} & 0.585 {\scriptsize (0.58, 0.59)} & 1.000 {\scriptsize } \\
\bottomrule
\end{tabular}
\end{table}

\begin{table}[htbp]
\centering
\caption{Performance metrics for Sachs dataset (500 samples) - Bold indicates best or statistically tied}
\label{tab:sachs_100_samples}
\footnotesize
\begin{tabular}{lccccc}
\toprule
Augmenter & Mean MSE $\downarrow$ & $\alpha$-Precision $\uparrow$ & $\beta$-Recall $\uparrow$ & DCR $\uparrow$ & $\delta$-Presence $\uparrow$ \\
\midrule
ADMGTian & 0.436 {\scriptsize (0.39, 0.54)} & 0.869 {\scriptsize (0.83, 0.90)} & 0.029 {\scriptsize (0.02, 0.04)} & 0.016 {\scriptsize (0.01, 0.02)} & 0.126 {\scriptsize (0.11, 0.15)} \\
ARF & 0.628 {\scriptsize (0.60, 0.69)} & 0.692 {\scriptsize (0.56, 0.74)} & 0.102 {\scriptsize (0.07, 0.11)} & \textbf{0.039} {\scriptsize (0.03, 0.07)} & \textbf{0.371} {\scriptsize (0.31, 0.55)} \\
CTGAN & 0.841 {\scriptsize (0.71, 1.30)} & 0.812 {\scriptsize (0.60, 0.89)} & 0.148 {\scriptsize (0.10, 0.18)} & \textbf{0.025} {\scriptsize (0.01, 0.07)} & \textbf{0.231} {\scriptsize (0.15, 0.48)} \\
CRB & \textbf{0.372} {\scriptsize (0.36, 0.38)} & 0.863 {\scriptsize (0.84, 0.88)} & 0.141 {\scriptsize (0.13, 0.15)} & 0.021 {\scriptsize (0.02, 0.02)} & 0.166 {\scriptsize (0.14, 0.20)} \\
DDPM & \textbf{0.400} {\scriptsize (0.37, 0.43)} & \textbf{0.938} {\scriptsize (0.93, 0.95)} & \textbf{0.341} {\scriptsize (0.33, 0.35)} & 0.011 {\scriptsize (0.01, 0.01)} & 0.121 {\scriptsize (0.10, 0.14)} \\
NFLOW & 0.697 {\scriptsize (0.65, 0.78)} & \textbf{0.910} {\scriptsize (0.83, 0.94)} & 0.200 {\scriptsize (0.17, 0.21)} & 0.018 {\scriptsize (0.01, 0.03)} & 0.164 {\scriptsize (0.12, 0.26)} \\
TVAE & 0.520 {\scriptsize (0.49, 0.55)} & \textbf{0.922} {\scriptsize (0.87, 0.95)} & 0.246 {\scriptsize (0.22, 0.26)} & 0.012 {\scriptsize (0.01, 0.01)} & 0.106 {\scriptsize (0.09, 0.12)} \\
\bottomrule
\end{tabular}
\end{table}

\begin{table}[htbp]
\centering
\caption{Performance metrics for Wine (red) dataset (500 samples) - Bold indicates best or statistically tied}
\label{tab:wine_red_100_samples}
\footnotesize
\begin{tabular}{lccccc}
\toprule
Augmenter & Mean MSE $\downarrow$ & $\alpha$-Precision $\uparrow$ & $\beta$-Recall $\uparrow$ & DCR $\uparrow$ & $\delta$-Presence $\uparrow$ \\
\midrule
ADMGTian & 1.210 {\scriptsize (1.14, 1.32)} & 0.283 {\scriptsize (0.24, 0.32)} & 0.016 {\scriptsize (0.01, 0.02)} & 0.205 {\scriptsize (0.19, 0.23)} & 0.734 {\scriptsize (0.69, 0.80)} \\
ARF & 0.597 {\scriptsize (0.58, 0.62)} & 0.872 {\scriptsize (0.86, 0.89)} & 0.278 {\scriptsize (0.26, 0.29)} & \textbf{0.260} {\scriptsize (0.25, 0.27)} & 1.000 {\scriptsize } \\
CTGAN & 0.624 {\scriptsize (0.60, 0.64)} & \textbf{0.919} {\scriptsize (0.89, 0.95)} & 0.375 {\scriptsize (0.35, 0.40)} & 0.207 {\scriptsize (0.20, 0.22)} & 1.000 {\scriptsize (1.00, 1.00)} \\
CRB & 0.494 {\scriptsize (0.47, 0.53)} & 0.835 {\scriptsize (0.81, 0.86)} & 0.438 {\scriptsize (0.43, 0.45)} & 0.203 {\scriptsize (0.20, 0.21)} & 1.000 {\scriptsize } \\
DDPM & \textbf{0.411} {\scriptsize (0.40, 0.42)} & \textbf{0.912} {\scriptsize (0.88, 0.94)} & \textbf{0.509} {\scriptsize (0.50, 0.52)} & 0.168 {\scriptsize (0.16, 0.18)} & 0.997 {\scriptsize (0.99, 1.00)} \\
NFLOW & 0.578 {\scriptsize (0.55, 0.64)} & \textbf{0.921} {\scriptsize (0.89, 0.94)} & 0.391 {\scriptsize (0.37, 0.41)} & 0.212 {\scriptsize (0.20, 0.22)} & 1.000 {\scriptsize } \\
TVAE & 0.553 {\scriptsize (0.53, 0.57)} & 0.818 {\scriptsize (0.77, 0.87)} & 0.431 {\scriptsize (0.42, 0.44)} & 0.178 {\scriptsize (0.17, 0.18)} & 1.000 {\scriptsize (1.00, 1.00)} \\
\bottomrule
\end{tabular}
\end{table}

\begin{table}[htbp]
\centering
\caption{Performance metrics for Wine (white) dataset (500 samples) - Bold indicates best or statistically tied}
\label{tab:wine_white_100_samples}
\footnotesize
\begin{tabular}{lccccc}
\toprule
Augmenter & Mean MSE $\downarrow$ & $\alpha$-Precision $\uparrow$ & $\beta$-Recall $\uparrow$ & DCR $\uparrow$ & $\delta$-Presence $\uparrow$ \\
\midrule
ADMGTian & \textbf{0.630} {\scriptsize (0.55, 0.92)} & \textbf{0.908} {\scriptsize (0.65, 0.98)} & 0.098 {\scriptsize (0.06, 0.11)} & 0.126 {\scriptsize (0.12, 0.13)} & 0.611 {\scriptsize (0.60, 0.62)} \\
ARF & 0.688 {\scriptsize (0.68, 0.69)} & \textbf{0.915} {\scriptsize (0.89, 0.93)} & 0.128 {\scriptsize (0.12, 0.14)} & \textbf{0.182} {\scriptsize (0.17, 0.19)} & \textbf{1.000} {\scriptsize (1.00, 1.00)} \\
CTGAN & 0.736 {\scriptsize (0.72, 0.76)} & \textbf{0.874} {\scriptsize (0.81, 0.92)} & 0.153 {\scriptsize (0.13, 0.17)} & \textbf{0.161} {\scriptsize (0.15, 0.18)} & \textbf{1.000} {\scriptsize (1.00, 1.00)} \\
CRB & \textbf{0.557} {\scriptsize (0.55, 0.56)} & \textbf{0.877} {\scriptsize (0.85, 0.90)} & 0.214 {\scriptsize (0.21, 0.22)} & 0.151 {\scriptsize (0.15, 0.16)} & \textbf{1.000} {\scriptsize (1.00, 1.00)} \\
DDPM & \textbf{0.552} {\scriptsize (0.55, 0.56)} & \textbf{0.869} {\scriptsize (0.83, 0.90)} & \textbf{0.238} {\scriptsize (0.22, 0.25)} & 0.128 {\scriptsize (0.12, 0.13)} & 0.986 {\scriptsize (0.97, 0.99)} \\
NFLOW & 0.652 {\scriptsize (0.64, 0.68)} & \textbf{0.865} {\scriptsize (0.77, 0.91)} & 0.159 {\scriptsize (0.13, 0.18)} & \textbf{0.167} {\scriptsize (0.16, 0.18)} & \textbf{1.000} {\scriptsize (1.00, 1.00)} \\
TVAE & 0.694 {\scriptsize (0.68, 0.71)} & 0.751 {\scriptsize (0.72, 0.78)} & 0.176 {\scriptsize (0.16, 0.19)} & 0.138 {\scriptsize (0.13, 0.14)} & 0.999 {\scriptsize (1.00, 1.00)} \\
\bottomrule
\end{tabular}
\end{table}

\begin{table}[htbp]
\centering
\caption{Performance metrics for Causal chambers dataset (500 samples) - Bold indicates best or statistically tied}
\label{tab:uniform_reference_100_samples}
\footnotesize
\begin{tabular}{lccccc}
\toprule
Augmenter & Mean MSE $\downarrow$ & $\alpha$-Precision $\uparrow$ & $\beta$-Recall $\uparrow$ & DCR $\uparrow$ & $\delta$-Presence $\uparrow$ \\
\midrule
ADMGTian & 0.412 {\scriptsize (0.41, 0.41)} & \textbf{0.967} {\scriptsize (0.96, 0.97)} & 0.042 {\scriptsize (0.04, 0.05)} & 0.584 {\scriptsize (0.58, 0.59)} & 1.000 {\scriptsize } \\
ARF & 0.556 {\scriptsize (0.55, 0.59)} & 0.961 {\scriptsize (0.95, 0.97)} & 0.031 {\scriptsize (0.03, 0.03)} & \textbf{0.727} {\scriptsize (0.72, 0.73)} & 1.000 {\scriptsize } \\
CTGAN & 0.468 {\scriptsize (0.46, 0.49)} & 0.897 {\scriptsize (0.83, 0.94)} & 0.124 {\scriptsize (0.10, 0.14)} & 0.647 {\scriptsize (0.63, 0.67)} & 1.000 {\scriptsize } \\
CRB & \textbf{0.362} {\scriptsize (0.36, 0.36)} & \textbf{0.976} {\scriptsize (0.97, 0.98)} & \textbf{0.331} {\scriptsize (0.33, 0.33)} & 0.583 {\scriptsize (0.58, 0.59)} & 1.000 {\scriptsize } \\
DDPM & 0.443 {\scriptsize (0.44, 0.45)} & 0.781 {\scriptsize (0.74, 0.82)} & 0.249 {\scriptsize (0.24, 0.26)} & 0.561 {\scriptsize (0.55, 0.57)} & 1.000 {\scriptsize } \\
NFLOW & 0.421 {\scriptsize (0.41, 0.43)} & 0.879 {\scriptsize (0.80, 0.92)} & 0.137 {\scriptsize (0.11, 0.17)} & 0.647 {\scriptsize (0.63, 0.68)} & 1.000 {\scriptsize } \\
TVAE & 0.469 {\scriptsize (0.46, 0.48)} & 0.812 {\scriptsize (0.78, 0.84)} & 0.189 {\scriptsize (0.18, 0.20)} & 0.607 {\scriptsize (0.60, 0.61)} & 1.000 {\scriptsize } \\
\bottomrule
\end{tabular}
\end{table}

\begin{table}[htbp]
\centering
\caption{Performance metrics for Sachs dataset (1000 samples) - Bold indicates best or statistically tied}
\label{tab:sachs_1000_samples}
\footnotesize
\begin{tabular}{lccccc}
\toprule
Augmenter & Mean MSE $\downarrow$ & $\alpha$-Precision $\uparrow$ & $\beta$-Recall $\uparrow$ & DCR $\uparrow$ & $\delta$-Presence $\uparrow$ \\
\midrule
ADMGTian & 0.416 {\scriptsize (0.38, 0.47)} & 0.753 {\scriptsize (0.73, 0.77)} & 0.034 {\scriptsize (0.03, 0.05)} & 0.018 {\scriptsize (0.02, 0.02)} & 0.202 {\scriptsize (0.19, 0.22)} \\
ARF & 0.570 {\scriptsize (0.51, 0.70)} & 0.691 {\scriptsize (0.56, 0.76)} & 0.180 {\scriptsize (0.11, 0.22)} & \textbf{0.036} {\scriptsize (0.03, 0.06)} & \textbf{0.392} {\scriptsize (0.30, 0.60)} \\
CTGAN & 0.659 {\scriptsize (0.62, 0.69)} & \textbf{0.941} {\scriptsize (0.91, 0.96)} & 0.351 {\scriptsize (0.32, 0.38)} & 0.013 {\scriptsize (0.01, 0.02)} & 0.147 {\scriptsize (0.14, 0.16)} \\
CRB & \textbf{0.349} {\scriptsize (0.33, 0.41)} & 0.857 {\scriptsize (0.84, 0.88)} & 0.238 {\scriptsize (0.22, 0.25)} & 0.022 {\scriptsize (0.02, 0.02)} & 0.176 {\scriptsize (0.16, 0.19)} \\
DDPM & \textbf{0.341} {\scriptsize (0.33, 0.36)} & \textbf{0.954} {\scriptsize (0.94, 0.96)} & \textbf{0.564} {\scriptsize (0.56, 0.57)} & 0.011 {\scriptsize (0.01, 0.01)} & 0.081 {\scriptsize (0.07, 0.09)} \\
NFLOW & 0.630 {\scriptsize (0.59, 0.68)} & 0.893 {\scriptsize (0.85, 0.92)} & 0.345 {\scriptsize (0.32, 0.37)} & 0.016 {\scriptsize (0.01, 0.02)} & 0.109 {\scriptsize (0.08, 0.17)} \\
TVAE & 0.500 {\scriptsize (0.47, 0.55)} & 0.911 {\scriptsize (0.87, 0.93)} & 0.364 {\scriptsize (0.27, 0.41)} & 0.013 {\scriptsize (0.01, 0.01)} & 0.099 {\scriptsize (0.09, 0.11)} \\
\bottomrule
\end{tabular}
\end{table}

\begin{table}[htbp]
\centering
\caption{Performance metrics for Wine (red) dataset (1000 samples) - Bold indicates best or statistically tied}
\label{tab:wine_red_1000_samples}
\footnotesize
\begin{tabular}{lccccc}
\toprule
Augmenter & Mean MSE $\downarrow$ & $\alpha$-Precision $\uparrow$ & $\beta$-Recall $\uparrow$ & DCR $\uparrow$ & $\delta$-Presence $\uparrow$ \\
\midrule
ADMGTian & 1.639 {\scriptsize (1.52, 1.90)} & 0.164 {\scriptsize (0.14, 0.19)} & 0.016 {\scriptsize (0.01, 0.02)} & \textbf{0.272} {\scriptsize (0.25, 0.29)} & 0.815 {\scriptsize (0.74, 0.88)} \\
ARF & 0.516 {\scriptsize (0.50, 0.52)} & \textbf{0.898} {\scriptsize (0.88, 0.91)} & 0.587 {\scriptsize (0.57, 0.60)} & \textbf{0.266} {\scriptsize (0.26, 0.27)} & 1.000 {\scriptsize } \\
CTGAN & 0.587 {\scriptsize (0.55, 0.61)} & \textbf{0.924} {\scriptsize (0.91, 0.93)} & 0.637 {\scriptsize (0.62, 0.65)} & 0.213 {\scriptsize (0.21, 0.22)} & 1.000 {\scriptsize (1.00, 1.00)} \\
CRB & 0.452 {\scriptsize (0.44, 0.46)} & 0.807 {\scriptsize (0.78, 0.83)} & 0.669 {\scriptsize (0.66, 0.68)} & 0.216 {\scriptsize (0.21, 0.22)} & 1.000 {\scriptsize } \\
DDPM & \textbf{0.344} {\scriptsize (0.33, 0.36)} & \textbf{0.920} {\scriptsize (0.88, 0.94)} & \textbf{0.736} {\scriptsize (0.72, 0.75)} & 0.185 {\scriptsize (0.18, 0.20)} & 0.999 {\scriptsize (1.00, 1.00)} \\
NFLOW & 0.522 {\scriptsize (0.49, 0.55)} & \textbf{0.923} {\scriptsize (0.89, 0.94)} & 0.640 {\scriptsize (0.62, 0.66)} & 0.232 {\scriptsize (0.22, 0.24)} & 1.000 {\scriptsize } \\
TVAE & 0.508 {\scriptsize (0.49, 0.53)} & 0.830 {\scriptsize (0.80, 0.88)} & 0.672 {\scriptsize (0.66, 0.68)} & 0.199 {\scriptsize (0.19, 0.21)} & 1.000 {\scriptsize (1.00, 1.00)} \\
\bottomrule
\end{tabular}
\end{table}

\begin{table}[htbp]
\centering
\caption{Performance metrics for Wine (white) dataset (1000 samples) - Bold indicates best or statistically tied}
\label{tab:wine_white_1000_samples}
\footnotesize
\begin{tabular}{lccccc}
\toprule
Augmenter & Mean MSE $\downarrow$ & $\alpha$-Precision $\uparrow$ & $\beta$-Recall $\uparrow$ & DCR $\uparrow$ & $\delta$-Presence $\uparrow$ \\
\midrule
ADMGTian & \textbf{0.816} {\scriptsize (0.48, 1.41)} & \textbf{0.816} {\scriptsize (0.52, 0.97)} & 0.173 {\scriptsize (0.10, 0.21)} & 0.125 {\scriptsize (0.11, 0.14)} & 0.666 {\scriptsize (0.61, 0.77)} \\
ARF & 0.641 {\scriptsize (0.63, 0.65)} & \textbf{0.934} {\scriptsize (0.92, 0.94)} & 0.229 {\scriptsize (0.21, 0.24)} & \textbf{0.182} {\scriptsize (0.17, 0.19)} & \textbf{1.000} {\scriptsize (1.00, 1.00)} \\
CTGAN & 0.683 {\scriptsize (0.67, 0.70)} & \textbf{0.913} {\scriptsize (0.87, 0.94)} & 0.257 {\scriptsize (0.23, 0.27)} & 0.155 {\scriptsize (0.15, 0.16)} & \textbf{1.000} {\scriptsize (1.00, 1.00)} \\
CRB & 0.534 {\scriptsize (0.52, 0.55)} & 0.841 {\scriptsize (0.82, 0.88)} & 0.321 {\scriptsize (0.31, 0.33)} & 0.148 {\scriptsize (0.14, 0.16)} & \textbf{1.000} {\scriptsize (1.00, 1.00)} \\
DDPM & \textbf{0.494} {\scriptsize (0.49, 0.50)} & \textbf{0.923} {\scriptsize (0.89, 0.94)} & \textbf{0.352} {\scriptsize (0.34, 0.36)} & 0.139 {\scriptsize (0.13, 0.15)} & 0.998 {\scriptsize (1.00, 1.00)} \\
NFLOW & 0.636 {\scriptsize (0.61, 0.66)} & \textbf{0.919} {\scriptsize (0.88, 0.95)} & 0.264 {\scriptsize (0.25, 0.28)} & \textbf{0.166} {\scriptsize (0.16, 0.18)} & \textbf{1.000} {\scriptsize (1.00, 1.00)} \\
TVAE & 0.655 {\scriptsize (0.64, 0.67)} & 0.803 {\scriptsize (0.78, 0.83)} & 0.269 {\scriptsize (0.25, 0.29)} & 0.148 {\scriptsize (0.14, 0.16)} & 0.999 {\scriptsize (1.00, 1.00)} \\
\bottomrule
\end{tabular}
\end{table}

\begin{table}[htbp]
\centering
\caption{Performance metrics for Causal chambers dataset (1000 samples) - Bold indicates best or statistically tied}
\label{tab:uniform_reference_1000_samples}
\footnotesize
\begin{tabular}{lccccc}
\toprule
Augmenter & Mean MSE $\downarrow$ & $\alpha$-Precision $\uparrow$ & $\beta$-Recall $\uparrow$ & DCR $\uparrow$ & $\delta$-Presence $\uparrow$ \\
\midrule
ADMGTian & 0.401 {\scriptsize (0.40, 0.41)} & \textbf{0.980} {\scriptsize (0.97, 0.98)} & 0.081 {\scriptsize (0.07, 0.09)} & 0.580 {\scriptsize (0.58, 0.58)} & 1.000 {\scriptsize } \\
ARF & 0.511 {\scriptsize (0.51, 0.52)} & 0.970 {\scriptsize (0.97, 0.97)} & 0.084 {\scriptsize (0.08, 0.09)} & \textbf{0.712} {\scriptsize (0.71, 0.72)} & 1.000 {\scriptsize } \\
CTGAN & 0.449 {\scriptsize (0.43, 0.49)} & 0.917 {\scriptsize (0.89, 0.94)} & 0.263 {\scriptsize (0.23, 0.29)} & 0.633 {\scriptsize (0.62, 0.64)} & 1.000 {\scriptsize } \\
CRB & \textbf{0.358} {\scriptsize (0.36, 0.36)} & \textbf{0.989} {\scriptsize (0.98, 0.99)} & \textbf{0.505} {\scriptsize (0.50, 0.51)} & 0.579 {\scriptsize (0.58, 0.58)} & 1.000 {\scriptsize } \\
DDPM & 0.418 {\scriptsize (0.40, 0.43)} & 0.859 {\scriptsize (0.82, 0.90)} & 0.472 {\scriptsize (0.46, 0.48)} & 0.568 {\scriptsize (0.56, 0.58)} & 1.000 {\scriptsize } \\
NFLOW & 0.421 {\scriptsize (0.41, 0.43)} & 0.844 {\scriptsize (0.78, 0.90)} & 0.202 {\scriptsize (0.17, 0.23)} & 0.666 {\scriptsize (0.65, 0.68)} & 1.000 {\scriptsize } \\
TVAE & 0.450 {\scriptsize (0.44, 0.46)} & 0.810 {\scriptsize (0.77, 0.84)} & 0.351 {\scriptsize (0.34, 0.36)} & 0.600 {\scriptsize (0.60, 0.60)} & 1.000 {\scriptsize } \\
\bottomrule
\end{tabular}
\end{table}

\begin{table}[htbp]
\centering
\caption{Performance metrics for Sachs dataset (2000 samples) - Bold indicates best or statistically tied}
\label{tab:sachs_2000_samples}
\footnotesize
\begin{tabular}{lccccc}
\toprule
Augmenter & Mean MSE $\downarrow$ & $\alpha$-Precision $\uparrow$ & $\beta$-Recall $\uparrow$ & DCR $\uparrow$ & $\delta$-Presence $\uparrow$ \\
\midrule
ADMGTian & 0.429 {\scriptsize (0.38, 0.48)} & 0.533 {\scriptsize (0.50, 0.57)} & 0.049 {\scriptsize (0.03, 0.07)} & \textbf{0.039} {\scriptsize (0.03, 0.05)} & \textbf{0.398} {\scriptsize (0.37, 0.43)} \\
ARF & 0.490 {\scriptsize (0.43, 0.62)} & 0.751 {\scriptsize (0.64, 0.80)} & 0.363 {\scriptsize (0.24, 0.42)} & \textbf{0.029} {\scriptsize (0.02, 0.04)} & \textbf{0.342} {\scriptsize (0.28, 0.49)} \\
CTGAN & 0.680 {\scriptsize (0.63, 0.75)} & 0.884 {\scriptsize (0.68, 0.94)} & 0.450 {\scriptsize (0.31, 0.53)} & 0.019 {\scriptsize (0.01, 0.03)} & 0.194 {\scriptsize (0.15, 0.33)} \\
CRB & \textbf{0.317} {\scriptsize (0.30, 0.35)} & 0.849 {\scriptsize (0.83, 0.86)} & 0.367 {\scriptsize (0.36, 0.38)} & 0.024 {\scriptsize (0.02, 0.03)} & 0.214 {\scriptsize (0.19, 0.24)} \\
DDPM & \textbf{0.316} {\scriptsize (0.30, 0.35)} & \textbf{0.965} {\scriptsize (0.96, 0.97)} & \textbf{0.757} {\scriptsize (0.75, 0.76)} & 0.012 {\scriptsize (0.01, 0.01)} & 0.071 {\scriptsize (0.06, 0.08)} \\
NFLOW & 0.628 {\scriptsize (0.55, 0.72)} & 0.901 {\scriptsize (0.84, 0.94)} & 0.555 {\scriptsize (0.53, 0.57)} & 0.015 {\scriptsize (0.01, 0.02)} & 0.099 {\scriptsize (0.07, 0.14)} \\
TVAE & 0.421 {\scriptsize (0.39, 0.44)} & \textbf{0.957} {\scriptsize (0.93, 0.97)} & 0.558 {\scriptsize (0.48, 0.60)} & 0.014 {\scriptsize (0.01, 0.01)} & 0.110 {\scriptsize (0.10, 0.13)} \\
\bottomrule
\end{tabular}
\end{table}

\begin{table}[htbp]
\centering
\caption{Performance metrics for Wine (white) dataset (2000 samples) - Bold indicates best or statistically tied}
\label{tab:wine_white_2000_samples}
\footnotesize
\begin{tabular}{lccccc}
\toprule
Augmenter & Mean MSE $\downarrow$ & $\alpha$-Precision $\uparrow$ & $\beta$-Recall $\uparrow$ & DCR $\uparrow$ & $\delta$-Presence $\uparrow$ \\
\midrule
ADMGTian & 0.952 {\scriptsize (0.52, 1.70)} & 0.734 {\scriptsize (0.47, 0.91)} & 0.311 {\scriptsize (0.17, 0.40)} & \textbf{0.178} {\scriptsize (0.15, 0.22)} & 0.787 {\scriptsize (0.73, 0.88)} \\
ARF & 0.606 {\scriptsize (0.60, 0.63)} & \textbf{0.942} {\scriptsize (0.94, 0.95)} & 0.396 {\scriptsize (0.38, 0.41)} & \textbf{0.198} {\scriptsize (0.18, 0.22)} & \textbf{1.000} {\scriptsize (1.00, 1.00)} \\
CTGAN & 0.734 {\scriptsize (0.66, 0.91)} & 0.847 {\scriptsize (0.58, 0.93)} & 0.367 {\scriptsize (0.24, 0.41)} & \textbf{0.202} {\scriptsize (0.17, 0.30)} & \textbf{1.000} {\scriptsize (1.00, 1.00)} \\
CRB & 0.504 {\scriptsize (0.50, 0.51)} & 0.814 {\scriptsize (0.80, 0.83)} & 0.464 {\scriptsize (0.46, 0.47)} & \textbf{0.166} {\scriptsize (0.15, 0.18)} & \textbf{1.000} {\scriptsize (1.00, 1.00)} \\
DDPM & \textbf{0.449} {\scriptsize (0.44, 0.46)} & \textbf{0.924} {\scriptsize (0.91, 0.94)} & \textbf{0.519} {\scriptsize (0.50, 0.53)} & \textbf{0.159} {\scriptsize (0.15, 0.17)} & 0.999 {\scriptsize (1.00, 1.00)} \\
NFLOW & 0.629 {\scriptsize (0.60, 0.66)} & \textbf{0.925} {\scriptsize (0.90, 0.94)} & 0.404 {\scriptsize (0.38, 0.43)} & \textbf{0.186} {\scriptsize (0.17, 0.21)} & \textbf{1.000} {\scriptsize (1.00, 1.00)} \\
TVAE & 0.611 {\scriptsize (0.59, 0.62)} & 0.828 {\scriptsize (0.80, 0.85)} & 0.434 {\scriptsize (0.41, 0.45)} & \textbf{0.165} {\scriptsize (0.15, 0.18)} & 1.000 {\scriptsize (1.00, 1.00)} \\
\bottomrule
\end{tabular}
\end{table}

\begin{table}[htbp]
\centering
\caption{Performance metrics for Causal chambers dataset (2000 samples) - Bold indicates best or statistically tied}
\label{tab:uniform_reference_2000_samples}
\footnotesize
\begin{tabular}{lccccc}
\toprule
Augmenter & Mean MSE $\downarrow$ & $\alpha$-Precision $\uparrow$ & $\beta$-Recall $\uparrow$ & DCR $\uparrow$ & $\delta$-Presence $\uparrow$ \\
\midrule
ADMGTian & 0.395 {\scriptsize (0.39, 0.40)} & \textbf{0.983} {\scriptsize (0.98, 0.99)} & 0.158 {\scriptsize (0.14, 0.18)} & 0.580 {\scriptsize (0.58, 0.58)} & 1.000 {\scriptsize } \\
ARF & 0.499 {\scriptsize (0.48, 0.57)} & 0.970 {\scriptsize (0.97, 0.97)} & 0.183 {\scriptsize (0.13, 0.20)} & \textbf{0.704} {\scriptsize (0.70, 0.73)} & 1.000 {\scriptsize } \\
CTGAN & 0.468 {\scriptsize (0.43, 0.52)} & 0.917 {\scriptsize (0.87, 0.95)} & 0.387 {\scriptsize (0.32, 0.45)} & 0.630 {\scriptsize (0.62, 0.64)} & 1.000 {\scriptsize } \\
CRB & \textbf{0.356} {\scriptsize (0.36, 0.36)} & \textbf{0.988} {\scriptsize (0.98, 0.99)} & \textbf{0.697} {\scriptsize (0.69, 0.70)} & 0.582 {\scriptsize (0.58, 0.59)} & 1.000 {\scriptsize } \\
DDPM & 0.379 {\scriptsize (0.36, 0.39)} & 0.910 {\scriptsize (0.87, 0.95)} & \textbf{0.696} {\scriptsize (0.69, 0.70)} & 0.571 {\scriptsize (0.56, 0.58)} & 1.000 {\scriptsize } \\
NFLOW & 0.415 {\scriptsize (0.41, 0.42)} & 0.948 {\scriptsize (0.94, 0.96)} & 0.376 {\scriptsize (0.35, 0.40)} & 0.649 {\scriptsize (0.64, 0.65)} & 1.000 {\scriptsize } \\
TVAE & 0.421 {\scriptsize (0.41, 0.43)} & 0.867 {\scriptsize (0.83, 0.89)} & 0.546 {\scriptsize (0.51, 0.57)} & 0.602 {\scriptsize (0.60, 0.61)} & 1.000 {\scriptsize } \\
\bottomrule
\end{tabular}
\end{table}

\section{Results with neural networks}
\label{app:nn_results}
\begin{figure}[htbp]
    \centering
    \includegraphics[width=\columnwidth]{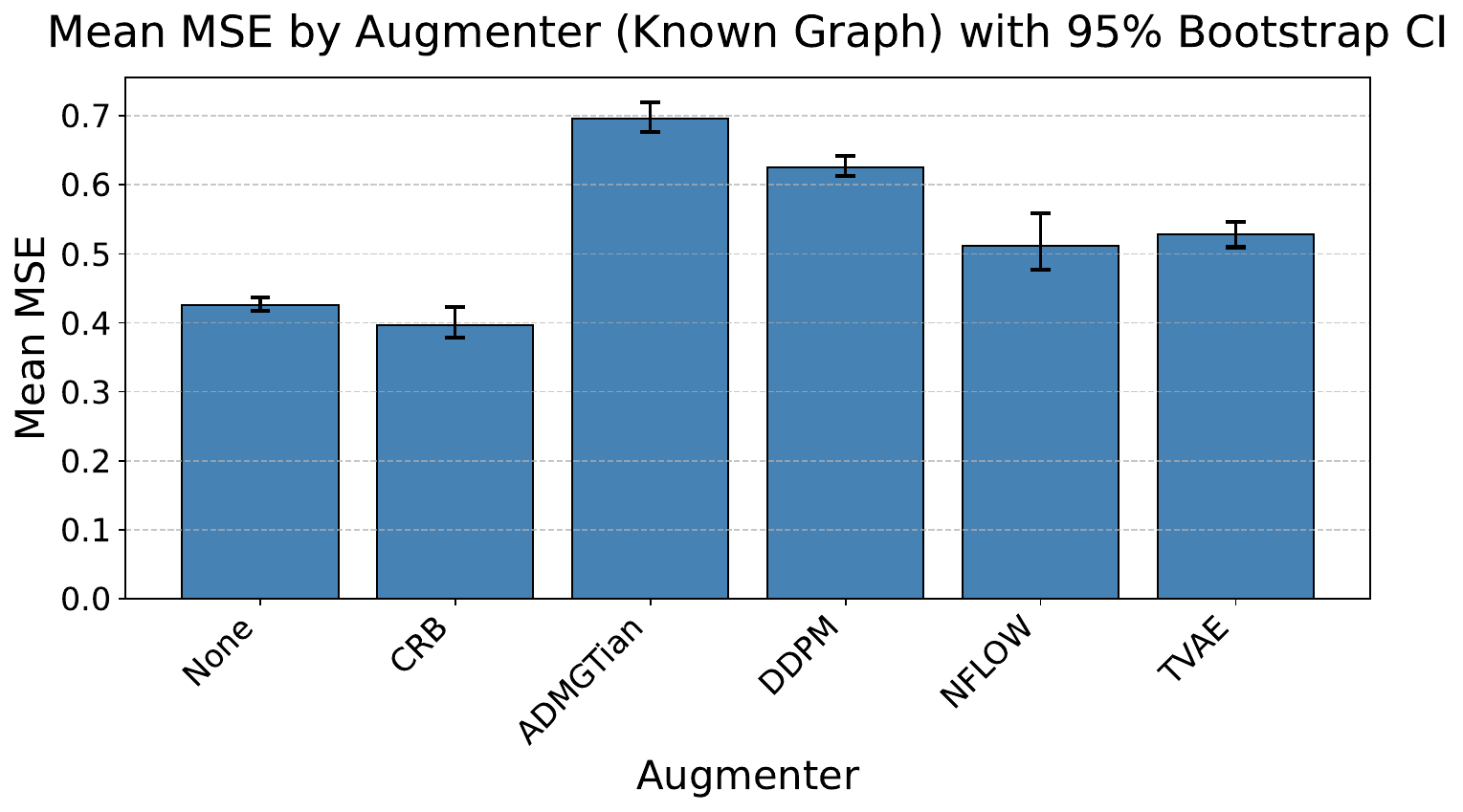}
    \caption{Mean MSE across all variables as sample size increases (Known Graph). Lower is better. CRB maintains strong performance across sample sizes. Neural networks were used as predictors.}
    \label{fig:known_graph_samples_mean_mse}
\end{figure}
\section{Experimental evaluation details}
\subsection{Hyperparameter grid search}
\label{app:hyperparams}
All grid were performed with 3-fold cross validation.

  Hyperparameter optimization for Xgboost was performed using Optuna with 500 trials. The search space for XGBoost included: number of
  estimators $n_{\text{estimators}} \in \{10, 110, \ldots, 1020\}$, learning rate $\eta \in [0.01, 0.3]$ (log-uniform), L2
  regularization $\lambda \in [0.1, 100]$, L1 regularization $\alpha \in [10^{-4}, 10]$ (log-uniform), minimum child weight $\in
  \{1, 3, \ldots, 19\}$, subsample ratio $\in [0.1, 1.0]$, column subsample ratio $\in [0.1, 1.0]$, and minimum loss reduction
  $\gamma \in [0, 5]$.

    Hyperparameter optimization for the neural network was performed using Optuna with 500 trials. The search space included: number
   of training epochs $\in \{50, 100, \ldots, 10000\}$, number of hidden layers $\in \{2, 3, 4\}$, hidden layer width $\in \{4, 8,
   \ldots, 32\}$, and learning rate $\in [10^{-5}, 10^{-2}]$ (log-uniform). The network used ReLU activations and was trained with
   Adam optimizer and early stopping (patience of 10 epochs).

  TVAE hyperparameters: encoder layers $\in \{1, \ldots, 5\}$, encoder units $\in [50, 500]$, encoder activation $\in
  \{\text{relu}, \text{leaky\_relu}, \text{tanh}, \text{elu}\}$, embedding units $\in [50, 500]$, decoder layers $\in \{1, \ldots,
   5\}$, decoder units $\in [50, 500]$, decoder activation $\in \{\text{relu}, \text{leaky\_relu}, \text{tanh}, \text{elu}\}$,
  iterations $\in [100, 1000]$, learning rate $\in [10^{-4}, 10^{-3}]$ (log), weight decay $\in [10^{-4}, 10^{-3}]$ (log).

  DDPM hyperparameters: iterations $\in [1000, 10000]$, learning rate $\in [10^{-5}, 10^{-1}]$ (log), weight decay $\in [10^{-4},
  10^{-3}]$ (log), diffusion timesteps $\in [10, 1000]$.

  ARF hyperparameters: number of trees $\in [1, 500]$ (step 10), minimum node size $\in [15, 500]$ (step 10).

    CTGAN hyperparameters: generator layers $\in \{1, \ldots, 5\}$, generator units $\in [1, 2000]$ (step 50), generator activation
  $\in \{\text{relu}, \text{leaky\_relu}, \text{tanh}, \text{elu}, \text{selu}\}$, discriminator layers $\in \{1, \ldots, 5\}$,
  discriminator units $\in [1, 2000]$ (step 50), discriminator activation $\in \{\text{relu}, \text{leaky\_relu}, \text{tanh},
  \text{elu}, \text{selu}\}$, iterations $\in [200, 5000]$ (step 100), discriminator iterations $\in \{1, \ldots, 5\}$, learning
  rate $\in [10^{-5}, 0.03]$ (log).

  ADMG-Tian hyperparameters: bandwidth temperature $\in [10^{-4}, 0.1]$ (log), weight threshold $\in [10^{-5}, 10^{-2}]$ (log).

  NFLOW (Normalizing Flows):
  Normalizing Flows hyperparameters: hidden layers $\in \{1, \ldots, 10\}$, hidden units $\in [10, 100]$, linear transform $\in
  \{\text{lu}, \text{permutation}, \text{svd}\}$, base transform $\in \{\text{affine-coupling}, \text{quadratic-coupling},
  \text{rq-coupling}, \text{affine-autoregressive}, \text{quadratic-autoregressive}, \text{rq-autoregressive}\}$, dropout $\in [0,
   0.2]$, batch normalization $\in \{\text{true}, \text{false}\}$, learning rate $\in [2 \times 10^{-4}, 10^{-3}]$ (log),
  iterations $\in [100, 5000]$.

\subsection{Causal discovery}
\label{app:cd_details}
We used DirectLiNGAM~\citep{shimizu2014lingam} for causal discovery with default settings: \texttt{pwling} independence measure and threshold $= 0$ (no
  edge pruning) and no prior knowledge about structure.

\section{Synthetic Data Experiments} \label{sec: synthetic data exp}
We first provide experiments on synthetic data to demonstrate how our approach is able to outperform the previous ADMG approach by \citet{teshima2021incorporating}.

To better understand the behavior of our proposed method in controlled settings, we constructed a simple synthetic experiment using a three-node causal chain $A \rightarrow B \rightarrow C$ with additive noise. In this setup, our objective is to predict the value of node $B$ given observations of both $A$ and $C$. It is important to note that both $A$ and $C$ lie in the Markov blanket of $B$---$A$ as its parent and $C$ as its child---making both variables necessary for the optimal prediction of $B$. This three-node chain represents the smallest configuration where meaningful causal data augmentation can be performed.

We evaluated four different data-generating configurations to assess how the method performs under various functional forms and noise distributions. The configurations tested were: (1) linear relationships with Gaussian noise, (2) linear relationships with non-Gaussian noise (specifically, uniform noise), (3) nonlinear relationships using quadratic functions with Gaussian noise, and (4) nonlinear relationships using ReLU neural networks with architecture $[4,4]$ and Gaussian noise.

\begin{figure*}[htbp]
\centering
\begin{subfigure}[t]{0.49\textwidth}
        \centering

        \includegraphics[width=\textwidth]{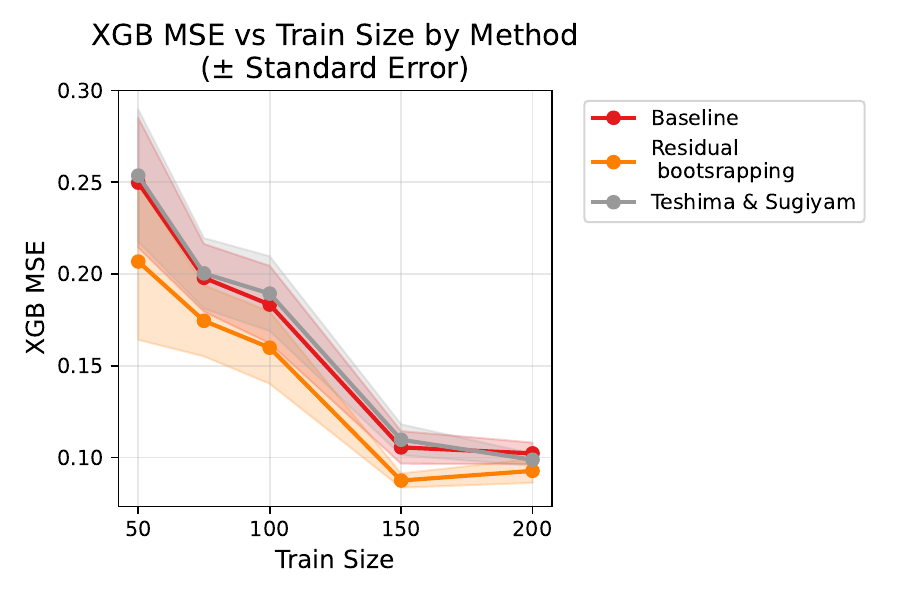}

    \caption{Results for chain with linear relations and additive Gaussian noise.}
    \label{fig:linear_gaussian}
\end{subfigure}
\hfill
\begin{subfigure}[t]{0.49\textwidth}
        \centering
        \includegraphics[width=\textwidth]{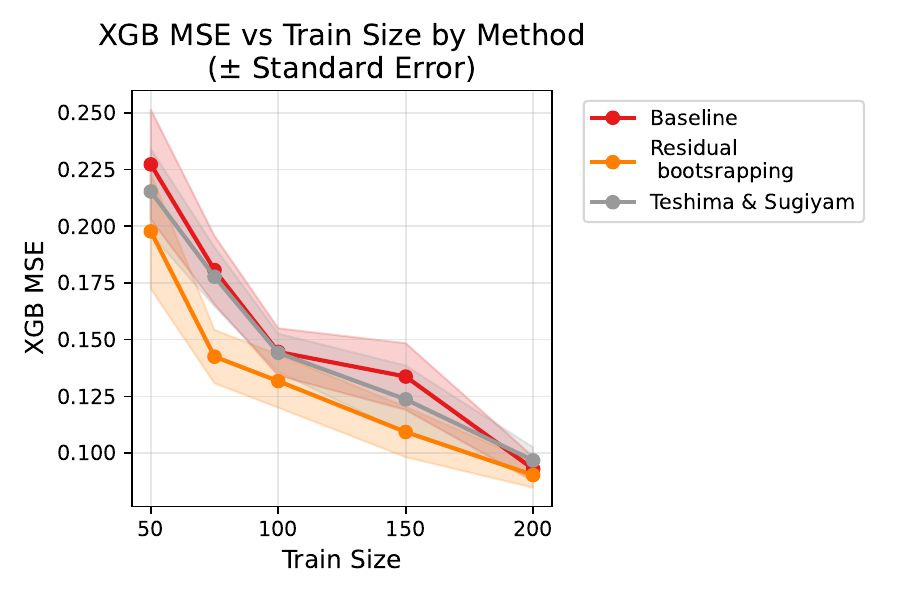}
    \caption{Results for chain with linear relations and additive uniform noise.}
\label{fig:linear_nongaussian}
\end{subfigure}
\caption{Results for chain with linear relations}
\end{figure*}

As we can see in Fig.~\ref{fig:linear_gaussian} the biggest improvment for residual bootstrapping can be seen for linear data with Gaussian additive noise, a relatively easy setup. We can also observe that residual bootstrapping yields better results for linear data with additive uniform noise in Fig.~\ref{fig:linear_nongaussian}.

For a non-linear function, we see an interesting phenomenon for both residual bootstrapping and the method from~\citet{teshima2021incorporating} -- more samples (around $75$) are needed for both methods to improve over baseline because of a more complicated functional relationship.

But once there are enough samples, both methods perform similarly, exhibiting slightly better results. This trend is visible for quadratic relationships in~\ref{fig:quadratic_gaussian} and more complicated relations sampled from random ReLu networks in~\ref{fig:relu_gaussian}.

\begin{figure*}[htbp]
\centering
\begin{subfigure}[t]{0.49\textwidth}
        \centering

        \includegraphics[width=\textwidth]{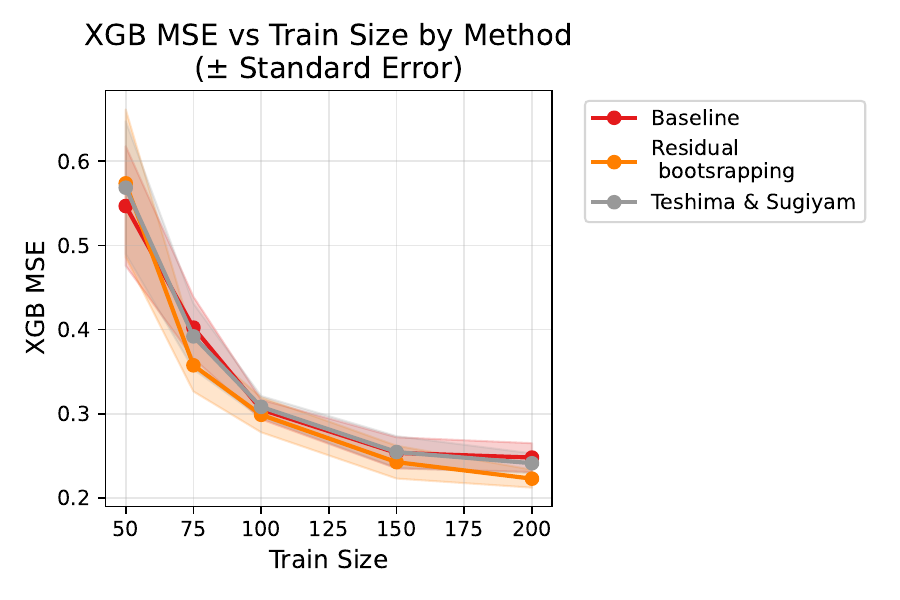}
    \caption{Results for chain with quadratic relations and additive Gaussian noise.}
\label{fig:quadratic_gaussian}
\end{subfigure}
\hfill
\begin{subfigure}[t]{0.49\textwidth}
        \centering
        \includegraphics[width=\textwidth]{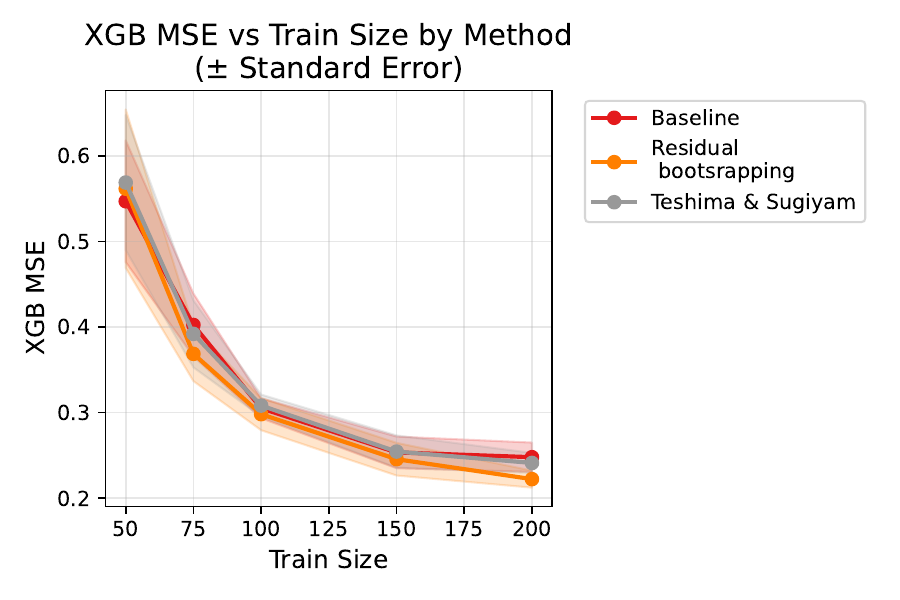}
    \caption{Results for chain with ReLu relations and additive Gaussian noise.}
\label{fig:relu_gaussian}
\end{subfigure}
\caption{Results for chain with non-linear relations}

\end{figure*}

\section{Compute resources}
\label{appendix:resources}
The computations were carried out on a FormatServer THOR E221 (Supermicro) server equipped with two AMD EPYC 7702 64-Core processors and 512 GB of RAM with operating system Ubuntu 22.04.1 LTS.

\section{Experimental Setup for Section~\ref{sec: data augmentation and causal structure}}\label{sec: implementation details for causal structure loss}

We generated $100$ random Directed Acyclic Graphs (DAGs) using the Erdos-Renyi model with $10$ nodes and $10$ expected edges. Edge weights were sampled uniformly from the range $[0.5,2.0]$. For each DAG, we simulated $2,000$ datapoints from the corresponding linear Gaussian SEM (where the noise term for each variable was drawn from a standard Normal distribution).

 The diffusion model, VAE, and GAN were trained to approximate the original data distribution. Training was conducted until the Fréchet Inception Distance (FID) plateaued. We utilized FID as the primary model selection criterion rather than loss or optimizer convergence, as the latter are often poor stopping criteria for these architectures. Given that the true data distribution is multivariate normal, FID is a more appropriate metric for generative quality in this simulation.
\end{document}